\documentclass{article} 
\usepackage{icml2026}

\usepackage{amsmath,amsfonts,bm}

\def\eqref#1{equation~\ref{#1}}

\def\1{\bm{1}}

\DeclareMathAlphabet{\mathsfit}{\encodingdefault}{\sfdefault}{m}{sl}
\SetMathAlphabet{\mathsfit}{bold}{\encodingdefault}{\sfdefault}{bx}{n}

\newcommand{\R}{\mathbb{R}}

\DeclareMathOperator*{\argmax}{arg\,max}
\DeclareMathOperator*{\argmin}{arg\,min}

\usepackage{natbib}
\usepackage{hyperref}
\usepackage{url}
\usepackage{amsmath}
\usepackage{amssymb}
\usepackage{amsfonts}
\usepackage{mathtools}
\usepackage{amsthm}
\usepackage{bm}
\usepackage{booktabs}
\usepackage[capitalize,noabbrev]{cleveref}
\usepackage{graphicx,caption}
\usepackage{algorithm, algorithmic}
\newtheorem{theorem}{Theorem}
\newtheorem{lemma}[theorem]{Lemma}
\newtheorem{proposition}[theorem]{Proposition}
\newtheorem{corollary}[theorem]{Corollary}
\theoremstyle{definition}
\newtheorem{definition}[theorem]{Definition}

\DeclareMathOperator{\ReLU}{ReLU}

\newcommand{\st}{\mathrm{s.t.}}

\title{Dual Certified White-Box Inference for Input Convex Neural Networks}

\author{
  \mbox{Kang Liu\textsuperscript{1}}
  \qquad
  \mbox{Jianchen Hu\textsuperscript{2,*}}
  \qquad
  \mbox{Wei Peng\textsuperscript{2}}
}

\begin{document}

\maketitle

\footnotetext[1]{School of Future Technology, Xi'an Jiaotong University, Xi'an, Shannxi, China, 710049.}
\footnotetext[2]{School of Automation Science and Engineering, Xi'an Jiaotong University, Xi'an, Shannxi, China, 710049.}

\begin{abstract}
Input convex neural networks (ICNNs) are used to learn convex objectives whose minimizers define decisions, making efficient and reliable optimization central to inference. At nonsmooth inputs, automatic differentiation returns a single derivative rather than the full subdifferential governing optimality and descent. Second-order cone ICNNs (SOC-ICNNs) admit an exact representation as value functions of parametric second-order cone programs, providing a white-box approach to recovering their full subdifferentials from optimal dual multipliers and deriving explicit Hessians on smooth regions. Building on this representation, we develop dual-certified inference (DCI), which combines the network and feasible set geometries to obtain exact stationarity certificates and tangent common descent directions. DCI uses local curvature for Newton acceleration and an exact proximal safeguard. We establish global convergence and, under standard regularity conditions, local quadratic convergence near structurally nondegenerate interior minimizers. Numerical experiments validate the recovered geometry and demonstrate the reliability and efficiency of DCI. Code is avaliable at \url{https://anonymous.4open.science/r/DCI-ICNN-507D/}.
\end{abstract}

\section{Introduction}
\label{sec:intro}

Input Convex Neural Networks (ICNNs) \citep{amos2017icnn} parameterize
functions that are convex in designated inputs and are therefore widely used
to learn objectives, energies, and value functions. Applications include
optimal transport \citep{makkuva2020oticnn}, system identification
\citep{bunning2021buildingmpc}, and control with learned models
\citep{lawrynczuk2022icnnmpc}. In optimization-based applications, inference
is not merely a forward evaluation. Given a context \(y\), we define the inference objective
\(F_y(x):=f_\theta(x;y)+\Phi(x;y)\), where \(x\in\mathbb R^d\)
is the decision variable, \(\theta\) denotes the network
parameters, and \(\Phi(\cdot;y)\) is an additional
term that is finite, convex, and continuously differentiable in \(x\). With \(y\) and \(\theta\)
held fixed, the final decision is obtained by solving
{\small
\begin{equation}
x^\star(y)\in\argmin_{x\in\mathcal X}F_y(x),
\label{eq:intro_inference_problem}
\end{equation}
}where \(\mathcal X\subseteq\mathbb R^d\) is a nonempty closed convex
feasible set.
The reliability and deployment efficiency of an ICNN therefore depend on both
its approximation quality and our ability to minimize the learned objective
quickly and certify the resulting decision.

The main obstacle is the derivative information available at nonsmooth
points. Consider the convex function \(F(x)=\max\{x,-2x\}\). Its global
minimizer is \(x=0\), since
\(\partial F(0)=[-2,1]\) contains the origin. Automatic differentiation (AD) may,
however, follow the branch \(F(x)=x\) and return the gradient
\(g_{\mathrm{AD}}=1\). This nonzero derivative incorrectly suggests that the
current point is nonstationary. Moreover, its negative is not a descent
direction: \(F'(0;-1)=2>0\). Thus, the derivative selected from one executed
branch can fail both as an optimality certificate and as a search direction.

The same issue appears in constrained ICNN inference. At a ReLU boundary, or
when a Euclidean norm residual vanishes, multiple supporting slopes may be
valid at the same input. AD returns a derivative
selected by the executed computational graph
\citep{bolte2020ad}, whereas
optimality depends jointly on the complete network subdifferential and the
constraint geometry \citep{rockafellar1998variational}:
{\small
\begin{equation}
\begin{aligned}
&0\in\partial F_y(x)+N_{\mathcal X}(x)
\Longleftrightarrow
x\in\argmin_{z\in\mathcal X}F_y(z),\\
&F_y'(x;d)=\max_{g\in\partial F_y(x)}g^\top d.
\end{aligned}
\label{eq:intro_first_order_conditions}
\end{equation}
}where \(N_{\mathcal X}(x)\) denotes the normal cone of the feasible
set \(\mathcal X\) at \(x\). A constrained search direction must
also belong to the tangent cone \(T_{\mathcal X}(x)\) to satisfy
first-order feasibility.

Recovering the complete geometry is nontrivial for a deep network. A single
boundary unit may admit a continuum of valid slopes, and the admissible choices
are coupled recursively across layers. Enumerating activation branches is
therefore combinatorial, while general nonsmooth methods address nonsmoothness
through gradient sampling \citep{burke2005gradient}, bundle models
\citep{kiwiel1990proximity}, or proximal constructions
\citep{rockafellar1976proximal}. The challenge
is richer for second-order cone ICNNs (SOC-ICNNs) \citep{LiuHu2026SOCICNN},
which extend conventional ICNNs with quadratic and conic modules and represent
a broader class of convex functions. Thus the central question is:
\textit{Can we recover and exploit the network's subdifferential
and curvature information for white-box inference?}

We exploit the exact representation of SOC-ICNNs as value functions
of parametric second-order cone programs (SOCPs). From the network's forward
states, we characterize the complete set of optimal dual multipliers
and prove that their associated support slopes form the
network subdifferential. The internal conic dual structure of the network therefore
provides exact directional derivatives and explicit local curvature on smooth
branches. This representation supplies the geometry needed for
certification and descent, together with the curvature required for Newton
acceleration. Moreover, the set-valued computation is restricted to multiplier
blocks associated with active structural boundaries, so its size is determined
by only the local boundary structure.

Building on this representation, we propose \emph{dual-certified inference}
(DCI) for fast and reliable SOC-ICNN deployment. DCI projects the origin onto
\(\partial F_y(x)+N_{\mathcal X}(x)\) under a positive definite metric.
The projection residual vanishes exactly at a constrained global
minimizer. Otherwise, it determines a direction in the tangent cone
that decreases every active affine support of \(F_y\). On unconstrained
smooth regions, using the explicit Hessian as the metric reduces this
direction to a Newton step. We combine the local step with an exact proximal safeguard
and establish global convergence, together with local quadratic convergence
near structurally nondegenerate interior minimizers under standard
regularity conditions. Numerical experiments verify the
recovered geometry and demonstrate the reliability and efficiency of DCI.

Our contributions are threefold:

(1) We characterize the complete set of optimal dual multipliers of an
    SOC-ICNN and prove that its readout exactly recovers the full network
    subdifferential and all directional derivatives. The same structure also
    yields explicit curvature on smooth branches.

(2) We develop DCI, which combines the network subdifferential with the
    geometry of the feasible set to provide exact stationarity certificates
    and common descent directions in the tangent cone. Its
    additional set-valued computation is restricted to the active structural
    boundaries.

(3) We establish global convergence of safeguarded DCI and local quadratic
    convergence under standard smoothness and regularity conditions. Numerical experiments validate that the recovered white-box geometry enables reliable inference, with substantial speedups on applications.

\section{Problem Statement}
\label{sec:preliminaries}

\paragraph{SOC-ICNN model.}
For a fixed context \(y\), write \(f(x):=f_\theta(x;y)\), where
\(x\in\mathbb R^d\) is the decision input and \(\theta\) collects the
trained parameters. A SOC-ICNN combines a ReLU-ICNN backbone with parallel
quadratic and norm modules~\citep{LiuHu2026SOCICNN}.
The backbone computes
\(
a_\ell(x)=W_\ell x+U_\ell z_{\ell-1}(x)+b_\ell,
z_\ell(x)=\ReLU(a_\ell(x)),
\)
for the layer index \(\ell=1,\ldots,L\),
\(z_0\equiv0\). Here \(a_\ell\) and \(z_\ell\) are the preactivation and
hidden activation vectors. The matrices \(W_\ell\) and \(U_\ell\) map the
original input and the previous hidden activation to layer \(\ell\),
respectively, and \(b_\ell\) is its bias.

The parallel modules use affine residuals
\(q_h(x):=B_hx+e_h\in\mathbb R^{m_h}\) and
\(u_g(x):=A_gx+d_g\in\mathbb R^{k_g}\), where \(B_h,A_g\) are module
matrices, \(e_h,d_g\) are offsets, and \(m_h,k_g\) are residual dimensions.
With \(H\) quadratic modules and \(G\) norm modules, the output is
{\small
\begin{equation}
f(x)
=
c^\top z_L(x)+v^\top x+b_0
+
\sum_{h=1}^{H}\frac{\alpha_h}{2}\|q_h(x)\|_2^2
+
\sum_{g=1}^{G}\lambda_g\|u_g(x)\|_2,
\label{eq:model}
\end{equation}
}where \(c\geq0\) weights the final hidden activation,
\(v^\top x+b_0\) is an affine output term, and
\(\alpha_h>0\), \(\lambda_g\geq0\) are module weights.
Together with the coordinatewise restriction \(U_\ell\geq0\), these
conditions preserve convexity in \(x\). Standard ReLU ICNNs are recovered when \(H=G=0\).

A forward pass provides \(f(x)\) and the internal states
\(\{a_\ell(x),z_\ell(x)\}_{\ell=1}^{L}\),
\(\{q_h(x)\}_{h=1}^{H}\), and \(\{u_g(x)\}_{g=1}^{G}\), which,
together with the fixed parameters, determine the subdifferential and
local curvature studied below. Architecture details and geometric proofs
are given in Appendix~\ref{sec:app_geometry_proofs}.

\paragraph{Problem setting and local geometry.}
We consider the inference problem in \eqref{eq:intro_inference_problem}.
Fixing the context \(y\), we write
\(
f(x):=f_\theta(x;y),~
\Phi(x):=\Phi(x;y),~
F(x):=f(x)+\Phi(x).
\)
For \(M\succ0\), we write
\(
\|v\|_M:=\sqrt{v^\top Mv}.
\)
We study the local first- and second-order geometry of the trained SOC-ICNN
and use it to solve
\(
\min_{x\in\mathcal X}F(x).
\)

For \(f(x)\), the complete first-order geometry at \(x\) is characterized by its subdifferential
\(
\partial f(x)
=
\{g\in\R^d:
f(x')\geq f(x)+g^\top(x'-x),~\forall x'\}.
\)
Since \(\Phi\) is finite, convex, and continuously differentiable,
standard subdifferential calculus \citep{rockafellar1998variational} gives
\(
\partial F(x)=\partial f(x)+\nabla\Phi(x),
~
F'(x;d)=\max_{g\in\partial F(x)}g^\top d,
\)
and constrained optimality satisfies \eqref{eq:intro_first_order_conditions}. Thus, \(\partial f(x)\) supplies the network side first-order information
needed for certification and descent construction.

Second-order information is available on smooth network branches.
The following condition guarantees such local smoothness.

\begin{definition}[Structural nondegeneracy]
\label{def:nondegenerate}
An input \(x\) is structurally nondegenerate if
\(
a_{\ell,i}(x)\neq0
~\forall~(\ell,i),
~
u_g(x)\neq0
~\forall~g\text{ with }\lambda_g>0.
\)
\end{definition}

At a structurally nondegenerate input, the ReLU activation pattern is locally
fixed and every positive-weight norm module is smooth; hence \(f\) is twice
continuously differentiable in a neighborhood of \(x\), with local gradient
\(\nabla f(x)\) and Hessian \(\nabla^2 f(x)\).
At structural boundaries, \(\partial f(x)\) may contain multiple subgradients.
Accordingly, we recover \(\partial f(x)\) at arbitrary inputs and combine
it with \(\nabla\Phi(x)\) and \(N_{\mathcal X}(x)\) for certification and
descent. On structurally nondegenerate regions, explicit network curvature
also supports the construction of the DCI metric.

\section{Exact First- and Second-Order Geometry of SOC-ICNNs}
\label{sec:geometry}

Given a trained SOC-ICNN and an input \(x\), we characterize the
first- and second-order geometry of its scalar output \(f\). At arbitrary
inputs, we recover the complete subdifferential \(\partial f(x)\); on
structurally nondegenerate regions, we further obtain explicit gradient and
Hessian information.

\subsection{Structured dual representation}
\label{sec:structured_dual}

By Lemma \ref{lem:app_exact_primal}, \(f(x)\) is the value of a parametric
SOCP, with an optimal primal solution supplied by the forward pass.
Its reduced dual variables are \(\nu\), \(p\), and \(r\) for the
ReLU, quadratic, and norm modules, respectively. For the ReLU backbone, define
\(
\mathcal K_\nu
:=
\{
\nu=\{\nu_\ell\}_{\ell=1}^{L}:
0\leq\nu_L\leq c,\,
0\leq\nu_\ell\leq U_{\ell+1}^{\top}\nu_{\ell+1},~\ell<L
\},
\)
where all inequalities are coordinatewise. The quadratic variables satisfy
\(
p\in\mathcal P:=\prod_{h=1}^{H}\R^{m_h},
\)
while the norm variables satisfy
\(
\mathcal K_r
:=
\prod_{g=1}^{G}
\{r_g\in\R^{k_g}:\|r_g\|_2\leq\lambda_g\}.
\)
Every included quadratic module has \(\alpha_h>0\).
For
\(
\xi=(\nu,p,r)\in\mathcal K_\nu\times\mathcal P\times\mathcal K_r,
\)
define
{\small
\[
\begin{aligned}
\Psi(x;\xi)
:=
v^\top x+b_0
+\sum_{\ell=1}^{L}\nu_\ell^\top(W_\ell x+b_\ell)\\
+\sum_{h=1}^{H}
\left[
p_h^\top q_h(x)-\frac{1}{2\alpha_h}\|p_h\|_2^2
\right]
+\sum_{g=1}^{G}r_g^\top u_g(x),
\end{aligned}
\]
}with input-space slope
\(
\mathcal G(\xi)
:=
v
+\sum_{\ell=1}^{L}W_\ell^\top\nu_\ell
+\sum_{h=1}^{H}B_h^\top p_h
+\sum_{g=1}^{G}A_g^\top r_g.
\)
For fixed \(\xi\), \(\Psi(\cdot;\xi)\) is affine in \(x\) with slope
\(\mathcal G(\xi)\). Every \(\xi\) therefore defines an affine
minorant of \(f\), and an optimal \(\xi\in\mathcal D^\star(x)\) defines an
affine support at \(x\). Map \(\mathcal G\) reads the corresponding slope
in the input space.

\begin{proposition}[Dual representation]
\label{prop:dual_representation}
For every \(x\in\R^d\),
\(
f(x)
=
\max_{\xi\in\mathcal K_\nu\times\mathcal P\times\mathcal K_r}
\Psi(x;\xi).
\)
Define the complete dual solution set
\(
\mathcal D^\star(x)
:=
\argmax_{\xi\in\mathcal K_\nu\times\mathcal P\times\mathcal K_r}
\Psi(x;\xi).
\)
For every \(\xi\in\mathcal D^\star(x)\),
\(
f(x)=\Psi(x;\xi),
~
f(x')\geq\Psi(x';\xi)
~\forall x'\in\R^d.
\)
Hence, \(\Psi(\cdot;\xi)\) is an affine support of \(f\) at \(x\) with
input-space slope \(\mathcal G(\xi)\).
\end{proposition}

The representation follows from the ReLU epigraph dual, quadratic conjugacy,
and the support-function representation of the Euclidean norm
\citep{boyd2004convex}. Consequently,
\(\mathcal D^\star(x)\) indexes the network supports active at \(x\), and
recovering first-order geometry reduces to characterizing this set and reading
its slopes through \(\mathcal G\).

\subsection{Exact first-order geometry}
\label{sec:exact_subdifferential}

The complete dual solution set can be determined directly from the forward-pass
states \(a_{\ell}(x)\), \(q_h(x)\), and \(u_g(x)\).
For each ReLU coordinate, define the recursive upper bound by
\(
\overline\nu_{L,i}(\nu):=c_i
\)
at the last layer and
\(
\overline\nu_{\ell,i}(\nu)
:=(U_{\ell+1}^{\top}\nu_{\ell+1})_i
\)
for \(\ell<L\).

\begin{proposition}[Blockwise characterization of the optimal-dual set]
\label{prop:blockwise_optimal_set} For every \(x\),
\(
\mathcal D^\star(x)
=
\mathcal D_{\mathrm R}^\star(x)
\times
\prod_{h=1}^{H}\{\alpha_hq_h(x)\}
\times
\prod_{g=1}^{G}\mathcal R_g^\star(x).
\)
For each norm module, if \(u_g(x)\neq0\), then
\(
\mathcal R_g^\star(x)
=
\left\{
\lambda_g\frac{u_g(x)}{\|u_g(x)\|_2}
\right\};
\)
if \(u_g(x)=0\), then
\(
\mathcal R_g^\star(x)
=
\left\{
r_g\in\R^{k_g}:\|r_g\|_2\leq\lambda_g
\right\}.
\)

For the ReLU backbone,
\(\mathcal D_{\mathrm R}^\star(x)\) consists of all
\(\nu\in\mathcal K_\nu\) satisfying the following coordinatewise conditions.
If \(a_{\ell,i}(x)<0\), then
\(
\nu_{\ell,i}=0.
\)
If \(a_{\ell,i}(x)>0\), then
\(
\nu_{\ell,i}
=
\overline\nu_{\ell,i}(\nu).
\)
If \(a_{\ell,i}(x)=0\), then
\(
0
\leq
\nu_{\ell,i}
\leq
\overline\nu_{\ell,i}(\nu).
\)
Moreover, \(\mathcal D^\star(x)\) is nonempty, compact, and convex.
\end{proposition}

Thus, quadratic multipliers and norm multipliers with nonzero residuals
are unique. ReLU multipliers are zero at negative preactivations
and recursively determined by later layers at positive preactivations.
The remaining free variables are associated with zero ReLU
preactivations and zero norm residuals with positive weights.

The input-space image of this complete dual set gives the exact first-order
geometry.

\begin{theorem}[Exact subdifferential readout]
\label{thm:exact_subdifferential}
For every \(x\in\R^d\),
\(
\partial f(x)
=
\mathcal G\bigl(\mathcal D^\star(x)\bigr)
=
\left\{
\mathcal G(\xi):
\xi\in\mathcal D^\star(x)
\right\}.
\)
Moreover, for every \(d\in\R^d\),
\(
f'(x;d)
=
\max_{\xi\in\mathcal D^\star(x)}
\mathcal G(\xi)^\top d.
\)
Consequently, \(f\) is differentiable at \(x\) if and only if
\(\mathcal G(\mathcal D^\star(x))\) is a singleton.
\end{theorem}

Each \(\xi\in\mathcal D^\star(x)\) defines an affine support active at \(x\),
and Danskin's theorem \citep{danskin1966} recovers the complete set of
input-space slopes.
Although \(\mathcal P\) is unbounded, its maximizer is uniquely
\(p_h^\star(x)=\alpha_hq_h(x)\) and is locally bounded in \(x\); the dual
maximum can therefore be locally restricted to a compact set. Since
\(\mathcal D^\star(x)\) is convex and \(\mathcal G\) is affine, no additional
convex hull is required.

Different dual solutions may induce the same input-space slope, so multiplier
nonuniqueness does not necessarily imply functional nonsmoothness. Moreover, minimizing the norm of the dual multipliers need not minimize
the norm of the resulting input-space slope, because the readout map
introduces scaling, coupling, and null directions. DCI therefore constructs
its certificate in the input space.

For the inference objective \(F=f+\Phi\), with \(\Phi\) finite,
convex, and continuously differentiable,
\begin{equation}
\partial F(x)
=
\mathcal G\bigl(\mathcal D^\star(x)\bigr)
+\nabla\Phi(x).
\label{eq:full_objective_subdifferential}
\end{equation}

The recovered correspondence also has a closed graph and is locally
bounded, properties used in the convergence analysis below.

\begin{proposition}[Closedness and local boundedness]
\label{prop:readout_stability}
The correspondence \(x\mapsto\mathcal D^\star(x)\) has nonempty compact
values, is locally bounded, and has a closed graph. Consequently, so does
\(
x\mapsto
\mathcal G\bigl(\mathcal D^\star(x)\bigr)
=
\partial f(x).
\)
In particular,
\(
x_n\to x,~
g_n\in\partial f(x_n),~
g_n\to g
\Longrightarrow
g\in\partial f(x).
\)
\end{proposition}

The result follows from compactness of \(\mathcal K_\nu\) and
\(\mathcal K_r\), local boundedness of
\(p_h^\star(x)=\alpha_hq_h(x)\), and continuity of \(f\), \(\Psi\), and
\(\mathcal G\).

\subsection{Exact second-order geometry on smooth branches}
\label{sec:local_curvature}

The dual readout above recovers the complete first-order geometry at arbitrary
inputs. On a structurally nondegenerate region, the ReLU pattern is fixed and
all positive-weight norm modules remain smooth, allowing the local curvature
to be recovered explicitly from the network structure.

\begin{proposition}[Smooth-branch geometry]
\label{prop:smooth_branch_geometry}
Let
\(
\mathcal I_+
:=
\{g\in\{1,\ldots,G\}:\lambda_g>0\},
\)
and let \(\bar x\) be structurally nondegenerate. Then there exists
\(\rho>0\) such that, for every \(x\in\mathbb B(\bar x,\rho)\), the ReLU
activation pattern is fixed, every \(u_g(x)\) with \(g\in\mathcal I_+\) is
nonzero, and \(\mathcal D^\star(x)\) is a singleton.

On this neighborhood, the ReLU component is affine. Thus, for some
\(\bar a\in\R^d\) and \(\bar b\in\R\),
{\small
\[
f(x)
=
\bar a^\top x+\bar b
+
\sum_{h=1}^{H}
\frac{\alpha_h}{2}\|q_h(x)\|_2^2
+
\sum_{g\in\mathcal I_+}
\lambda_g\|u_g(x)\|_2.
\]
}
Writing
\(
\widehat u_g(x):=u_g(x)/\|u_g(x)\|_2,
\)
the gradient and Hessian are
{\small
\begin{subequations}
\label{eq:local_derivatives}
\begin{align}
\nabla f(x)
&=
\bar a
+
\sum_{h=1}^{H}\alpha_hB_h^\top q_h(x)
+
\sum_{g\in\mathcal I_+}
\lambda_gA_g^\top\widehat u_g(x),
\label{eq:local_gradient}\\
\nabla^2 f(x)
&=
\sum_{h=1}^{H}\alpha_hB_h^\top B_h
+
\sum_{g\in\mathcal I_+}
\lambda_gA_g^\top
\frac{
I-\widehat u_g(x)\widehat u_g(x)^\top
}{
\|u_g(x)\|_2
}
A_g
\succeq0.
\label{eq:local_hessian}
\end{align}
\end{subequations}
}Moreover, \(\nabla^2f\) is Lipschitz continuous on every sufficiently small
closed ball in this neighborhood.
\end{proposition}

The ReLU backbone contributes no curvature on a fixed activation region.
Each quadratic module contributes the constant positive-semidefinite matrix
\(\alpha_hB_h^\top B_h\), while each norm module contributes curvature in
directions orthogonal to its current residual and zero curvature in the radial
direction. Since
\(
\inf_{x\in\mathbb B(\bar x,\rho)}
\|u_g(x)\|_2>0,
~~ g\in\mathcal I_+,
\)
after reducing \(\rho\) if necessary, the norm-module Hessians are locally
Lipschitz. If \(\Phi\) is twice continuously differentiable, then
\(
\nabla F(x)=\nabla f(x)+\nabla\Phi(x),
~
\nabla^2F(x)=\nabla^2f(x)+\nabla^2\Phi(x).
\)

We have therefore recovered the two network-side geometric objects required
for inference: the first-order geometry
\(\partial f(x)\) at arbitrary inputs and the local curvature
\(\nabla^2f(x)\) on structurally nondegenerate regions. The next section
combines the former with \(\nabla\Phi(x)\) and \(N_{\mathcal X}(x)\) to
construct certified search directions, while the latter provides the local
curvature metric used by DCI.

\section{Dual-Certified Inference}
\label{sec:dci}

We now turn the recovered SOC-ICNN subdifferential and local curvature into
an iterative method for \(\min_{x\in\mathcal X}F(x)\), where \(F=f+\Phi\).
Dual-certified inference (DCI) jointly optimizes over network multipliers
and constraint normals to obtain an optimality residual and a search
direction. A feasible line search and a proximal safeguard then determine
the next iterate.

\subsection{Search direction and optimality certificate}
\label{sec:dci_projection}

At a feasible input \(x\), the directional derivative
\(F'(x;d)=\max_{g\in\partial F(x)}g^\top d\) describes the first-order
variation of the objective. Given a positive definite metric
\(M\in\R^{d\times d}\), DCI selects the direction
\begin{equation}
\small
d_M(x)
:=
\argmin_{d\in T_{\mathcal X}(x)}
\left\{
F'(x;d)+\frac12\|d\|_M^2
\right\},
\label{eq:dci_direction_model}
\end{equation}
where \(T_{\mathcal X}(x)\) is the tangent cone of \(\mathcal X\) at \(x\).
The tangent constraint imposes first-order feasibility, while the quadratic
term controls the scale of the direction through \(M\).

The representation in \eqref{eq:full_objective_subdifferential} allows this
model to be solved through the optimal network multipliers
\(\xi\in\mathcal D^\star(x)\) and normals \(n\in N_{\mathcal X}(x)\),
where \(N_{\mathcal X}(x)\) is the normal cone of the feasible set at \(x\).
Define the constrained first-order set
\(\mathcal C_{\mathcal X}(x):=\partial F(x)+N_{\mathcal X}(x)\).
Projecting the origin onto this set in the \(M^{-1}\)-norm gives
{\small
\begin{subequations}
\label{eq:dci_projection}
\begin{equation}
\bigl(\xi_M^\dagger(x),n_M^\dagger(x)\bigr)
\in
\argmin_{
\xi\in\mathcal D^\star(x),
n\in N_{\mathcal X}(x)
}
\frac12
\left\|
\mathcal G(\xi)+\nabla\Phi(x)+n
\right\|_{M^{-1}}^2,
\label{eq:constrained_projection}
\end{equation}
\begin{equation}
g_M^\dagger(x)
:=
\mathcal G\bigl(\xi_M^\dagger(x)\bigr)
+\nabla\Phi(x)+n_M^\dagger(x),
~~
\mathcal R_M(x)
:=
\|g_M^\dagger(x)\|_{M^{-1}}.
\label{eq:projected_readout}
\end{equation}
\end{subequations}
}Here \(\mathcal G\) maps optimal network multipliers to network
subgradients, and \(\mathcal R_M(x)\) is the projection residual.
The set \(\mathcal C_{\mathcal X}(x)\) is nonempty, closed, and convex, so
\(g_M^\dagger(x)\) is unique even when its multiplier and normal
decomposition is nonunique. By
Proposition \ref{prop:app_steepest_feasible_descent}, the unique solution of
\eqref{eq:dci_direction_model} is
\(d_M(x)=-M^{-1}g_M^\dagger(x)\).

\begin{theorem}[Certificate and tangent common descent]
\label{thm:certificate}
For every \(x\in\mathcal X\) and \(M\succ0\),
\(
\mathcal R_M(x)=0
\Longleftrightarrow
x\in\argmin_{z\in\mathcal X}F(z).
\)
If \(\mathcal R_M(x)>0\), then
\(d_M(x)\in T_{\mathcal X}(x)\) and
\begin{equation}
g^\top d_M(x)
\leq
-\mathcal R_M(x)^2,
~~
\forall\,g\in\partial F(x).
\label{eq:common_descent}
\end{equation}
Consequently,
\(F'(x;d_M(x))\leq-\mathcal R_M(x)^2<0\).
\end{theorem}

Thus, the same projection either certifies global optimality or produces a
tangent direction that decreases every affine support of \(F\) active at
\(x\). If \(F\) is \(\mu_M\)-strongly convex on \(\mathcal X\) with respect
to the \(M\)-norm, where \(\mu_M>0\), the residual also bounds the objective
gap:
\(
0\leq F(x)-F^\star
\leq\mathcal R_M(x)^2/(2\mu_M),
\)
where \(F^\star:=\min_{z\in\mathcal X}F(z)\)
(see Corollary \ref{cor:app_residual_gap}).

\subsection{Boundary-local computation and metric selection}
\label{sec:boundary_local_projection}

By Proposition \ref{prop:blockwise_optimal_set}, several coordinates in
\eqref{eq:constrained_projection} can be eliminated. Quadratic multipliers
and norm multipliers with nonzero residuals have explicit values and can
be substituted directly. ReLU multipliers are eliminated in reverse layer
order: negative preactivations set their multipliers to zero, while
positive preactivations express them recursively through later layers.
The remaining network variables correspond to zero ReLU preactivations
and zero norm residuals with positive weights, with the recursive
constraints between layers retained.

Define the corresponding index sets
\(
\mathcal Z_{\mathrm R}(x)
:=\{(\ell,i):a_{\ell,i}(x)=0\}
\)
and
\(
\mathcal Z_{\mathrm N}(x)
:=\{g:u_g(x)=0,\ \lambda_g>0\}.
\)
Suppose that the normal cone admits the representation
\(
N_{\mathcal X}(x)
=
\{B_{\mathcal X}(x)\zeta:
\zeta\in\mathcal K_{\mathcal X}(x)\},
\)
where \(\zeta\) collects \(q_{\mathcal X}(x)\) scalar normal parameters,
\(B_{\mathcal X}(x)\in\R^{d\times q_{\mathcal X}(x)}\), and
\(\mathcal K_{\mathcal X}(x)\subseteq\R^{q_{\mathcal X}(x)}\) is a closed
convex cone. If \(N_{\mathcal X}(x)=\{0\}\), we take
\(q_{\mathcal X}(x)=0\).

\begin{proposition}[Boundary-local projection]
\label{prop:projection_size}
After eliminating fixed or recursively determined dual coordinates, the
projection in \eqref{eq:constrained_projection} contains at most
\(
|\mathcal Z_{\mathrm R}(x)|
+
\sum_{g\in\mathcal Z_{\mathrm N}(x)}k_g
+
q_{\mathcal X}(x)
\)
scalar decision variables, where \(k_g\) is the dimension of the norm
residual \(u_g(x)\). Its objective is convex quadratic, and its feasible
set consists of the recursive ReLU interval constraints, the Euclidean
balls associated with zero norm residuals, and the cone
\(\mathcal K_{\mathcal X}(x)\). If
\(\mathcal Z_{\mathrm R}(x)=\mathcal Z_{\mathrm N}(x)=\varnothing\)
and \(q_{\mathcal X}(x)=0\), then
\(\mathcal C_{\mathcal X}(x)\) is a singleton and no projection solve
is required.
\end{proposition}

Thus, the variable dimension of the additional projection problem is
determined by the current network boundaries and the normal-cone
representation.

Local curvature enters the direction model through the metric. At a structurally nondegenerate iterate \(x_k\), if \(\Phi\) is twice
continuously differentiable near \(x_k\), define
\(
H_k:=\nabla^2f(x_k)+\nabla^2\Phi(x_k)=\nabla^2F(x_k).
\)
The network Hessian is given in \eqref{eq:local_hessian}.
If \(x_k\in\operatorname{int}\mathcal X\) and \(H_k\succ0\), choosing
\(M_k=H_k\) recovers the Newton direction \citep{nocedal2006numerical}
\(d_{M_k}(x_k)=-H_k^{-1}\nabla F(x_k)\).
For the global analysis, the metrics satisfy
\begin{equation}
0<mI\preceq M_k\preceq\overline m I,
\label{eq:metric_bounds}
\end{equation}
where \(0<m\leq\overline m<\infty\) are fixed constants and \(I\) is the
identity matrix. A regularized curvature metric is
\(
M_k=
\operatorname{clip}_{[m,\overline m]}(H_k+\tau_kI),
\)
where \(\tau_k\geq0\) is a regularization parameter and
\(\operatorname{clip}_{[m,\overline m]}\) clips eigenvalues to
\([m,\overline m]\). When this curvature is unavailable, any positive
definite metric satisfying \eqref{eq:metric_bounds} may be used.

\subsection{Feasible steps and safeguard}
\label{sec:safeguarded_dci}

Starting from \(x_0\in\mathcal X\), at iteration \(k\) we write
\(
g_k^\dagger:=g_{M_k}^\dagger(x_k),~
\mathcal R_k:=\mathcal R_{M_k}(x_k),~
d_k:=-M_k^{-1}g_k^\dagger.
\)
DCI returns \(x_k\) when \(\mathcal R_k\leq\varepsilon\), where
\(\varepsilon\geq0\) is the stopping tolerance. Otherwise,
the tangent direction is converted into a feasible trial point using a
retraction
\(\mathsf P_{\mathcal X}:\mathcal X\times\R^d\to\mathcal X\)
satisfying
\begin{equation}
\mathsf P_{\mathcal X}(x,0)=x,
~~
\mathsf P_{\mathcal X}(x,td)=x+td+o(t)
\quad\text{as }t\downarrow0
\label{eq:feasible_retraction}
\end{equation}
for every \(x\in\mathcal X\) and \(d\in T_{\mathcal X}(x)\).
The remainder satisfies \(\|o(t)\|_2/t\to0\), so the retraction preserves
the direction to first order. One admissible choice is
\(\mathsf P_{\mathcal X}(x,v)=\Pi_{\mathcal X}(x+v)\), where
\(\Pi_{\mathcal X}\) denotes Euclidean projection onto \(\mathcal X\).

Let \(0<\sigma<1/2\) and \(0<\eta<1\) be sufficient-decrease parameters,
\(0<\beta<1\) the backtracking contraction, and \(J_{\max}\) a fixed
nonnegative integer limiting the number of reductions. DCI tests
\(t=\beta^j\) for \(j=0,\ldots,J_{\max}\), in this order. For each
trial, define the feasible candidate
\(z_k(t):=\mathsf P_{\mathcal X}(x_k,td_k)\) and its displacement
\(s_k(t):=z_k(t)-x_k\). The first step satisfying
\begin{equation}
F'\bigl(x_k;s_k(t)\bigr)
\leq-\eta t\mathcal R_k^2,~~
F\bigl(z_k(t)\bigr)
\leq F(x_k)+\sigma F'\bigl(x_k;s_k(t)\bigr)
\label{eq:armijo}
\end{equation}
is accepted, with \(x_{k+1}=z_k(t)\). The first condition preserves
sufficient directional decrease along the actual feasible displacement;
the second adapts the Armijo sufficient-decrease principle
\citep{armijo1966minimization} to the directional prediction along this
displacement. The directional derivative is computed using
Theorem \ref{thm:exact_subdifferential} and \eqref{eq:full_objective_subdifferential}.

At any fixed point with positive residual, common descent and the
first-order consistency in \eqref{eq:feasible_retraction} ensure that all
sufficiently small steps satisfy \eqref{eq:armijo}
(Lemma \ref{lem:app_armijo_exists}). The finite schedule gives the uniform
lower bound \(t\geq\beta^{J_{\max}}\) on accepted steps, relating their
objective decrease to the current residual. If none of the tested steps is
accepted, DCI applies a metric proximal safeguard based on proximal-point
regularization \citep{rockafellar1976proximal}:
\begin{equation}
x_{k+1}
=
\argmin_{z\in\mathcal X}
\left\{
F(z)+\frac{1}{2\gamma}\|z-x_k\|_{M_k}^2
\right\},
~~ \gamma>0,
\label{eq:proximal}
\end{equation}
where \(\gamma\) is a fixed proximal parameter. This strongly convex
subproblem has a unique solution. Its exact solution decreases the
objective and yields a vector in the constrained first-order set at the new
iterate, whose \(M_k^{-1}\)-norm bounds
\(\mathcal R_{M_k}(x_{k+1})\) from above
(Lemma \ref{lem:app_proximal_safeguard}). The pseudocode is summarized
in Algorithm~\ref{alg:app_dci}.

\subsection{Convergence}
\label{sec:dci_convergence}

\begin{theorem}[Global and local convergence]
\label{thm:convergence}
\emph{Global convergence.}
Assume that \(f\) is finite and convex, \(\Phi\) is finite, convex, and
continuously differentiable on \(\R^d\), and \(\mathcal X\) is nonempty,
closed, and convex. Let the feasible retraction satisfy
\eqref{eq:feasible_retraction}. Suppose that the initial level set
\(
\mathcal L_0:=\{x\in\mathcal X:F(x)\leq F(x_0)\}
\)
is compact, the metrics satisfy \eqref{eq:metric_bounds},
\(\varepsilon=0\), and every invoked proximal subproblem is solved
exactly. Then \(F(x_{k+1})\leq F(x_k)\). Finite termination returns a
global minimizer; for an infinite run, \(F(x_k)\to F^\star\) and every
accumulation point of \(\{x_k\}\) is globally optimal. If the minimizer
is unique, the entire sequence converges to it.

\emph{Local quadratic convergence.}
Suppose additionally that the minimizer \(x^\star\) is unique, lies in
\(\operatorname{int}\mathcal X\), and is structurally nondegenerate.
Assume that \(F\) is twice continuously differentiable near \(x^\star\),
\(\nabla^2F(x^\star)\succ0\), and \(\nabla^2F\) is locally Lipschitz
continuous. Assume also that the retraction is locally exact:
\(\mathsf P_{\mathcal X}(x,v)=x+v\) whenever \(x\) and \(x+v\) lie in a
sufficiently small neighborhood of \(x^\star\). If
\(M_k=\nabla^2F(x_k)\) in this neighborhood, then an infinite run
eventually accepts unit steps and satisfies
\(
\|x_{k+1}-x^\star\|_2
\leq C\|x_k-x^\star\|_2^2
\)
for some constant \(C>0\) and all sufficiently large \(k\).
\end{theorem}

The two update types satisfy
{\small
\begin{equation}
F(x_k)-F(x_{k+1})
\geq
\begin{cases}
\sigma\eta\beta^{J_{\max}}\mathcal R_k^2,
\text{for an accepted local step},\\[1mm]
\dfrac{1}{2\gamma}\|x_{k+1}-x_k\|_{M_k}^2,
\text{for a proximal step}.
\end{cases}
\label{eq:dci_sufficient_decrease}
\end{equation}
}Moreover, proximal optimality gives
\(
\gamma^{-1}M_k(x_k-x_{k+1})
\in\partial F(x_{k+1})+N_{\mathcal X}(x_{k+1}).
\)
Together with compactness and the closed graph of
\(x\mapsto\mathcal C_{\mathcal X}(x)\), these relations yield a globally
optimal accumulation point. Monotonicity and continuity then give
\(F(x_k)\to F^\star\), establishing optimality of every accumulation
point. Once DCI recovers the Newton iteration near the structurally
nondegenerate minimizer, the standard local Newton analysis
\citep{nocedal2006numerical} yields quadratic convergence.
Proofs are given in
\cref{sec:app_global_convergence,sec:app_local_quadratic}.

\section{Experiments}
\label{sec:experiments}
We organize five experiments to address four questions; complete settings and results are given in \cref{app:experiments}. Since no existing method directly provides the same
complete first- and second-order geometry of ICNNs, we compare DCI with AD-based
optimizers, Gradient Sampling \citep{burke2005gradient}, a high-accuracy primal--dual solver applied to
the full lifted conic formulation (Lifted-PD), and component ablations.

Success is independently determined from the objective gap and complete
constrained residual; an invalid declared convergence is counted as a wrong
stop. DCI-I uses the identity metric and DCI-C uses explicit curvature.
Local variants omit the proximal safeguard, whereas Full variants include it.

\subsection{Does the Readout Recover the Exact Geometry?}

At structurally nondegenerate inputs, the explicit gradient and Hessian agree
with independently computed autograd quantities from PyTorch \citep{paszke2019pytorch} to mean errors
\(2.12\times10^{-15}\) and \(4.12\times10^{-16}\), respectively. The
second-order Taylor residual has log--log slope \(3.006\). At mixed
structural boundaries, where ReLU preactivations and norm residuals vanish
simultaneously, the directional readout agrees with one-sided finite
differences to \(2.37\times10^{-8}\). These results verify the smooth and
nonsmooth geometric formulas; complete tests are given in \cref{app:exp1}.

\subsection{Does the Recovered Geometry Improve Inference?}

We evaluate 24 strongly convex objectives whose known minimizers lie
at ReLU, norm, or mixed structural boundaries, using two initializations
per objective. Across 48 matched runs, DCI-I, DCI-C, and Lifted-PD
achieve \(100\%\) success, compared with \(68.8\%\) for Gradient
Sampling and \(0\%\) for Single-Branch and L-BFGS \citep{liu1989lbfgs}. Gradient Sampling
also records a \(12.5\%\) wrong-stop rate
(Table~\ref{tab:main_mechanisms}(a)).
These results demonstrate the reliability of DCI on the constructed
boundary problems and support the use of the recovered subdifferential
for nonsmooth inference.

To isolate curvature, all trajectories remain on smooth branches,
where AD-GD and Readout-GD use the same exact gradients.
We test quadratic and quadratic--norm objectives with Hessian condition
numbers
\(\kappa\in\{1,\dots,10^4\}\).
DCI-C achieves \(100\%\) success throughout, with median iteration counts
of one and two on the respective families. Readout-GD and AD-L-BFGS
slow as \(\kappa\) increases, while AD-GD closely tracks Readout-GD
(Fig. \ref{fig:main_conditioning}).
Canonical-Newton uses AD gradients and Hessians
\citep{nocedal2006numerical} and matches DCI-C's iteration counts.
At \(\kappa=10^4\), DCI-C's median runtimes are \(7.2\times\) and
\(10.4\times\) faster, respectively, demonstrating the computational
benefit.

\begin{figure*}[ht]
\centering
\captionsetup{skip=2pt}
\includegraphics[width=0.8\linewidth]{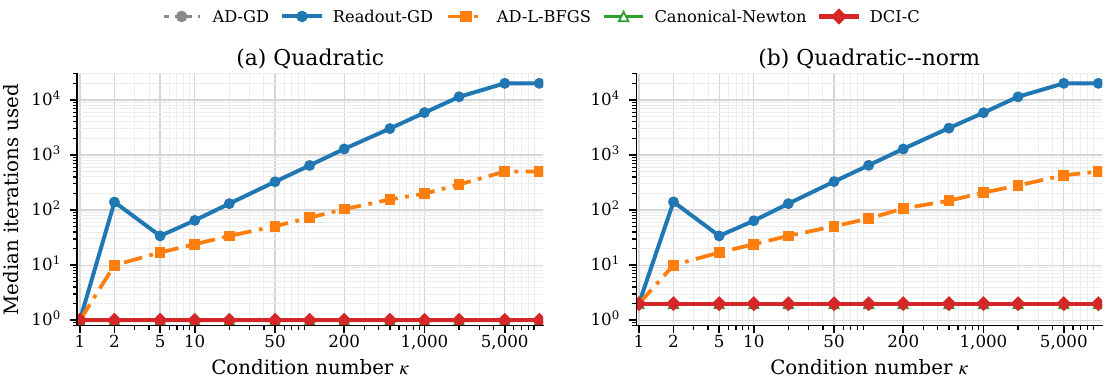}
\caption{\textbf{Inference under increasing condition number.}
Median iterations over five instances per \(\kappa\) on smooth branches,
isolating curvature effects. DCI-C takes one and two iterations on the two
families, while first-order and quasi-Newton baselines slow as conditioning
worsens.}
\label{fig:main_conditioning}
\end{figure*}
\vspace{-2.5pt}

\subsection{Why Are Constraint Geometry and the Safeguard Needed?}

We evaluate 95 objectives with known solutions in unconstrained space
and over box, ball, simplex, and second-order cone feasible sets.
The suite contains regular, structural-boundary, constraint-boundary,
joint-boundary, and stress cases.
At the structural, joint, and stress optima, the network-only residual
has mean \(1.123\), while the joint residual is at most
\(3.33\times10^{-8}\) across all optima.

AD-PGD and Readout-PGD use the projected-gradient update structure
\citep{nocedal2006numerical}, with a branch-selected AD vector and a
minimum-norm objective subgradient, respectively.
DCI-I-Local solves \(63.2\%\) of the suite, DCI-C-Local solves
\(94.7\%\), and DCI-C-Full solves all instances
(Table~\ref{tab:main_mechanisms}(b)).
With the default \(J_{\max}=10\), the safeguard is invoked in only
\(5/95\) runs. Prox-Only, our metric variant of the proximal-point method
\citep{rockafellar1976proximal}, takes a proximal step at every iteration
and also solves all instances but requires
\(4.83\times\) the median runtime of DCI-C-Full.
In a separate budget sweep over \(J_{\max}=0,\ldots,10\), larger
budgets improve Local success and reduce safeguard use, while Full
remains \(100\%\) successful at every tested budget
(\cref{fig:main_safeguard}).

\begin{table*}[ht]
\centering
\caption{\footnotesize
\textbf{Boundary inference and safeguard comparison.}
(a) Success (Succ.) and wrong-stop rates over 48 matched runs per method.
(b) Success, certified termination (Cert.), safeguard use (SG), and
median runtime over 95 objectives per method.
All rates are percentages; a dash denotes an unassessed wrong-stop rate.}
\label{tab:main_mechanisms}
\scriptsize
\setlength{\tabcolsep}{3.5pt}
\begin{tabular}{@{}c@{\hspace{12pt}}c@{}}
\textbf{(a) Boundary inference}
&
\textbf{(b) Constraint and safeguard}
\\[1mm]
\begin{tabular}[t]{@{}lrr@{}}
\toprule
Method & Succ. & Wrong\\
\midrule
Single-Branch  &   0.0 &  0.0\\
L-BFGS         &   0.0 &   --\\
Grad. Sampling &  68.8 & 12.5\\
DCI-I          & 100.0 &  0.0\\
DCI-C          & 100.0 &  0.0\\
Lifted-PD      & 100.0 &  0.0\\
\bottomrule
\end{tabular}
&
\begin{tabular}[t]{@{}lrrrr@{}}
\toprule
Method & Succ. & Cert. & SG & Time (ms)\\
\midrule
AD-PGD      &  49.5 &  36.8 &   0.0 &  86.9\\
Readout-PGD &  60.0 &  47.4 &   0.0 &  71.5\\
DCI-I-Local &  63.2 &  63.2 &   0.0 & 168.1\\
DCI-C-Local &  94.7 &  94.7 &   0.0 &  93.4\\
DCI-C-Full  & 100.0 & 100.0 &   5.3 & 135.6\\
Prox-Only   & 100.0 & 100.0 & 100.0 & 655.4\\
\bottomrule
\end{tabular}
\end{tabular}
\end{table*}


\subsection{Can DCI Reliably Optimize Trained Application Objectives?}

We consider day-ahead battery scheduling using historical data from the
Swedish SE3 bidding zone in the Open Power System Data Time series package
\citep{opsd2020timeseries}. Each pre-dispatch context
contains day-ahead prices, official load forecasts, causal wind forecasts
constructed from past observations, and calendar information. The decision
\(z\in[-1,1]^{12}\) parameterizes a 24-hour battery trajectory whose energy,
power, and terminal constraints are satisfied by construction. An SOC-ICNN is
trained to represent the resulting context-dependent convex operating cost.

We evaluate three independently trained SOC-ICNNs on 30 held-out contexts,
giving 90 frozen box-constrained objectives. The comparisons include
AD-L-BFGS-B, which uses the bound-constrained limited-memory BFGS method
\citep{byrd1995lbfgsb}. Every method optimizes the same
learned objective from the same initial point, thereby isolating inference
error from surrogate approximation error. DCI-C is the only non-reference
method with \(100\%\) success, reaching a median complete residual of
\(1.18\times10^{-6}\) in two iterations and \(3.09\) ms. Lifted-PD also
solves every objective but requires a median of \(537.5\) iterations and
\(241\) ms, approximately \(78\times\) the runtime of DCI-C
(\cref{tab:main_learned}(a)).

For the operational comparison, Point Forecast substitutes pre-dispatch
forecasts into the engineering objective, following point-prediction-driven
optimization \citep{bertsimas2020predictive}. Historical SAA uses
nearest-neighbor sample average approximation \citep{bertsimas2020predictive},
averaging the realized objectives of the seven nearest standardized training
contexts. No Storage uses \(z=0\), and the perfect-information (PI) oracle
minimizes the engineering objective using realized test-day net load.
These application-specific baselines are detailed in \cref{app:exp5b}.

\begin{table}[ht]
\centering
\caption{\footnotesize
\textbf{Battery-scheduling experiments.}
(a) Inference on 90 frozen SOC-ICNN objectives; Wrong denotes an invalid
declared convergence.
(b) Mean operational regret under the realized engineering cost.}
\label{tab:main_learned}
\scriptsize
\setlength{\tabcolsep}{3.5pt}
\begin{tabular}{@{}c@{\hspace{14pt}}c@{}}

\textbf{(a) Frozen-objective inference}
&
\textbf{(b) End-to-end scheduling}
\\[1mm]

\begin{tabular}[t]{@{}lrrrr@{}}
\toprule
Method & Success & Wrong & Iter. & ms\\
\midrule
Readout-GD       & 94.4\% & 0.0\% & 18    & 14.9\\
AD-L-BFGS-B      & 94.4\% & 5.6\% & 7     & 7.00\\
DCI-I            & 96.7\% & 0.0\% & 18    & 21.0\\
DCI-C            & 100\%  & 0.0\% & 2     & 3.09\\
Lifted-PD        & 100\%  & 0.0\% & 537.5 & 241\\
\bottomrule
\end{tabular}

&

\begin{tabular}[t]{@{}lr@{}}
\toprule
Pipeline & Regret\\
\midrule
No Storage      & 1.1567\\
Point Forecast  & 0.0054\\
Historical SAA  & 0.0104\\
SOC-ICNN        & 2.1733\\
PI oracle       & 0.0000\\
\bottomrule
\end{tabular}

\end{tabular}
\end{table}
Across 1089 SOC-ICNN seed--day solves, DCI-C certifies \(99.1\%\)
of decisions with a median runtime of \(3.05\) ms. The results identify
surrogate approximation as the dominant source of operational regret
(\cref{tab:main_learned}(b)); further analysis is given in \cref{app:exp5b}.

\section{Conclusion and Discussion}
\label{sec:conclusion}

We recover the complete first-order geometry of SOC-ICNNs from their dual
structure and develop DCI for certified and efficient inference. While the
current method relies on structured convex architectures and exact geometric
information, combining DCI with quasi-Newton, active-set, or other hybrid
optimization strategies may further improve practical convergence and extend
the framework to broader model classes.

\section*{AI Use Statement}
ChatGPT 5.6 Terra assists with the coding work and the first-round review feedback. The authors take responsibility for the final manuscript.

\bibliographystyle{icml2026}
\bibliography{iclr2027_conference}

@inproceedings{amos2017icnn,
  author    = {Amos, Brandon and Xu, Lei and Kolter, J. Zico},
  title     = {Input Convex Neural Networks},
  booktitle = {Proceedings of the 34th International Conference on Machine Learning},
  series    = {Proceedings of Machine Learning Research},
  volume    = {70},
  publisher = {PMLR},
  year      = {2017},
  url       = {https://proceedings.mlr.press/v70/amos17b.html}
}

@inproceedings{makkuva2020oticnn,
  author    = {Makkuva, Ashok and Taghvaei, Amirhossein and Oh, Sewoong and Lee, Jason},
  title     = {Optimal Transport Mapping via Input Convex Neural Networks},
  booktitle = {Proceedings of the 37th International Conference on Machine Learning},
  series    = {Proceedings of Machine Learning Research},
  volume    = {119},
  publisher = {PMLR},
  year      = {2020},
  url       = {https://proceedings.mlr.press/v119/makkuva20a.html}
}

@inproceedings{huang2021cpflow,
  author    = {Huang, Chin-Wei and Chen, Ricky T. Q. and Tsirigotis, Christos and Courville, Aaron},
  title     = {Convex Potential Flows: Universal Probability Distributions with Optimal Transport and Convex Optimization},
  booktitle = {International Conference on Learning Representations},
  year      = {2021},
  url       = {https://arxiv.org/abs/2012.05942}
}

@inproceedings{bunning2021buildingmpc,
  author    = {B{\"u}nning, Felix and Schalbetter, Adrian and Aboudonia, Ahmed and Hudoba de Badyn, Mathias and Heer, Philipp and Lygeros, John},
  title     = {Input Convex Neural Networks for Building {MPC}},
  booktitle = {Proceedings of the 3rd Conference on Learning for Dynamics and Control},
  series    = {Proceedings of Machine Learning Research},
  volume    = {144},
  publisher = {PMLR},
  year      = {2021},
  url       = {https://proceedings.mlr.press/v144/bunning21a.html}
}

@article{lawrynczuk2022icnnmpc,
  author  = {{\L}awry{\'n}czuk, Maciej},
  title   = {Input Convex Neural Networks in Nonlinear Predictive Control: A Multi-Model Approach},
  journal = {Neurocomputing},
  volume  = {513},
  pages   = {273--293},
  year    = {2022},
  doi     = {10.1016/j.neucom.2022.09.108}
}

@misc{LiuHu2026SOCICNN,
  author        = {Liu, Kang and Hu, Jianchen and Peng, Wei},
  title         = {{SOC-ICNN}: From Polyhedral to Conic Geometry for Learning Convex Surrogate Functions},
  year          = {2026},
  eprint        = {2604.22355},
  archivePrefix = {arXiv},
  primaryClass  = {cs.LG},
  url           = {https://arxiv.org/abs/2604.22355}
}

@book{rockafellar1998variational,
  author    = {Rockafellar, R. Tyrrell and Wets, Roger J. B.},
  title     = {Variational Analysis},
  series    = {Grundlehren der mathematischen Wissenschaften},
  volume    = {317},
  publisher = {Springer},
  year      = {1998},
  doi       = {10.1007/978-3-642-02431-3}
}

@book{bonnans2000perturbation,
  author    = {Bonnans, J. Fr{\'e}d{\'e}ric and Shapiro, Alexander},
  title     = {Perturbation Analysis of Optimization Problems},
  publisher = {Springer},
  year      = {2000},
  doi       = {10.1007/978-1-4612-1394-9}
}

@article{aravkin2013value,
  author  = {Aravkin, Aleksandr Y. and Burke, James V. and Friedlander, Michael P.},
  title   = {Variational Properties of Value Functions},
  journal = {SIAM Journal on Optimization},
  volume  = {23},
  number  = {3},
  pages   = {1689--1717},
  year    = {2013},
  doi     = {10.1137/120899157}
}

@inproceedings{bolte2021nonsmoothimplicit,
  author    = {Bolte, J{\'e}r{\^o}me and Le, Tam and Pauwels, Edouard and Silveti-Falls, Tony},
  title     = {Nonsmooth Implicit Differentiation for Machine-Learning and Optimization},
  booktitle = {Advances in Neural Information Processing Systems},
  volume    = {34},
  year      = {2021},
  url       = {https://proceedings.neurips.cc/paper/2021/hash/70afbf2259b4449d8ae1429e054df1b1-Abstract.html}
}

@inproceedings{bolte2023complexity,
  author    = {Bolte, J{\'e}r{\^o}me and Boustany, Ryan and Pauwels, Edouard and Pesquet-Popescu, B{\'e}atrice},
  title     = {On the Complexity of Nonsmooth Automatic Differentiation},
  booktitle = {International Conference on Learning Representations},
  year      = {2023},
  url       = {https://openreview.net/forum?id=uqg3FhRZaq}
}

@article{rockafellar1976proximal,
  author  = {Rockafellar, R. Tyrrell},
  title   = {Monotone Operators and the Proximal Point Algorithm},
  journal = {SIAM Journal on Control and Optimization},
  volume  = {14},
  number  = {5},
  pages   = {877--898},
  year    = {1976},
  doi     = {10.1137/0314056}
}

@inproceedings{ali2017newtonadmm,
  author    = {Ali, Alnur and Wong, Eric and Kolter, J. Zico},
  title     = {A Semismooth {Newton} Method for Fast, Generic Convex Programming},
  booktitle = {Proceedings of the 34th International Conference on Machine Learning},
  series    = {Proceedings of Machine Learning Research},
  volume    = {70},
  publisher = {PMLR},
  year      = {2017},
  url       = {https://proceedings.mlr.press/v70/ali17a.html}
}

@inproceedings{lin2024semismoothot,
  author    = {Lin, Tianyi and Cuturi, Marco and Jordan, Michael},
  title     = {A Specialized Semismooth {Newton} Method for Kernel-Based Optimal Transport},
  booktitle = {Proceedings of the 27th International Conference on Artificial Intelligence and Statistics},
  series    = {Proceedings of Machine Learning Research},
  volume    = {238},
  publisher = {PMLR},
  year      = {2024},
  url       = {https://proceedings.mlr.press/v238/lin24a.html}
}

@inproceedings{korotin2021benchmark,
  author    = {Korotin, Alexander and Li, Lingxiao and Genevay, Aude and Solomon, Justin M. and Filippov, Alexander and Burnaev, Evgeny},
  title     = {Do Neural Optimal Transport Solvers Work? A Continuous {Wasserstein-2} Benchmark},
  booktitle = {Advances in Neural Information Processing Systems},
  volume    = {34},
  year      = {2021},
  url       = {https://proceedings.neurips.cc/paper/2021/hash/7a6a6127ff85640ec69691fb0f7cb1a2-Abstract.html}
}

@inproceedings{bolte2020ad,
  author    = {Bolte, J{\'e}r{\^o}me and Pauwels, Edouard},
  title     = {A Mathematical Model for Automatic Differentiation in Machine Learning},
  booktitle = {Advances in Neural Information Processing Systems},
  volume    = {33},
  year      = {2020},
  url       = {https://papers.nips.cc/paper_files/paper/2020/hash/7a674153c63cff1ad7f0e261c369ab2c-Abstract.html}
}

@inproceedings{lee2023correctness,
  author    = {Lee, Wonyeol and Park, Sejun and Aiken, Alex},
  title     = {On the Correctness of Automatic Differentiation for Neural Networks with Machine-Representable Parameters},
  booktitle = {Proceedings of the 40th International Conference on Machine Learning},
  series    = {Proceedings of Machine Learning Research},
  volume    = {202},
  publisher = {PMLR},
  year      = {2023},
  url       = {https://proceedings.mlr.press/v202/lee23p.html}
}

@article{burke2005gradient,
  author  = {Burke, James V. and Lewis, Adrian S. and Overton, Michael L.},
  title   = {A Robust Gradient Sampling Algorithm for Nonsmooth, Nonconvex Optimization},
  journal = {SIAM Journal on Optimization},
  volume  = {15},
  number  = {3},
  pages   = {751--779},
  year    = {2005},
  doi     = {10.1137/030601296}
}

@article{milzarek2024ssn,
  author  = {Milzarek, Andre and Schaipp, Fabian and Ulbrich, Michael},
  title   = {A Semismooth {Newton} Stochastic Proximal Point Algorithm with Variance Reduction},
  journal = {SIAM Journal on Optimization},
  volume  = {34},
  number  = {1},
  pages   = {1157--1185},
  year    = {2024},
  doi     = {10.1137/22M1488181}
}

@article{feng2025semismooth,
  author  = {Feng, Fuxiaoyue and Ding, Chao and Li, Xudong},
  title   = {A Quadratically Convergent Semismooth {Newton} Method for Nonlinear Semidefinite Programming without Generalized {Jacobian} Regularity},
  journal = {Mathematical Programming},
  volume  = {214},
  number  = {1--2},
  pages   = {643--683},
  year    = {2025},
  doi     = {10.1007/s10107-025-02198-0}
}

@article{alphonse2026globalized,
  author  = {Alphonse, Amal and Christof, Constantin and Hinterm{\"u}ller, Michael and Papadopoulos, Ioannis P. A.},
  title   = {A Globalized Inexact Semismooth {Newton} Method for Nonsmooth Fixed-Point Equations Involving Variational Inequalities},
  journal = {Computational Optimization and Applications},
  volume  = {93},
  number  = {1},
  pages   = {1--55},
  year    = {2026},
  doi     = {10.1007/s10589-025-00722-8}
}

@article{kiwiel1990proximity,
  author  = {Kiwiel, Krzysztof C.},
  title   = {Proximity Control in Bundle Methods for Convex Nondifferentiable Minimization},
  journal = {Mathematical Programming},
  volume  = {46},
  pages   = {105--122},
  year    = {1990},
  doi     = {10.1007/BF01585731}
}

@book{boyd2004convex,
  author    = {Boyd, Stephen and Vandenberghe, Lieven},
  title     = {Convex Optimization},
  publisher = {Cambridge University Press},
  year      = {2004},
  doi       = {10.1017/CBO9780511804441}
}

@article{danskin1966,
  author  = {Danskin, John M.},
  title   = {The Theory of {Max-Min}, with Applications},
  journal = {SIAM Journal on Applied Mathematics},
  volume  = {14},
  number  = {4},
  pages   = {641--664},
  year    = {1966},
  doi     = {10.1137/0114053}
}

@article{armijo1966minimization,
  author  = {Armijo, Larry},
  title   = {Minimization of Functions Having {Lipschitz} Continuous First Partial Derivatives},
  journal = {Pacific Journal of Mathematics},
  volume  = {16},
  number  = {1},
  pages   = {1--3},
  year    = {1966},
  doi     = {10.2140/pjm.1966.16.1},
  url     = {https://msp.org/pjm/1966/16-1/pjm-v16-n1-p01-s.pdf}
}

@book{nocedal2006numerical,
  author    = {Nocedal, Jorge and Wright, Stephen J.},
  title     = {Numerical Optimization},
  edition   = {2},
  series    = {Springer Series in Operations Research and Financial Engineering},
  publisher = {Springer},
  year      = {2006},
  doi       = {10.1007/978-0-387-40065-5}
}

@article{liu1989lbfgs,
  author  = {Liu, Dong C. and Nocedal, Jorge},
  title   = {On the Limited Memory {BFGS} Method for Large Scale Optimization},
  journal = {Mathematical Programming},
  volume  = {45},
  pages   = {503--528},
  year    = {1989},
  doi     = {10.1007/BF01589116}
}

@article{byrd1995lbfgsb,
  author  = {Byrd, Richard H. and Lu, Peihuang and Nocedal, Jorge and Zhu, Ciyou},
  title   = {A Limited Memory Algorithm for Bound Constrained Optimization},
  journal = {SIAM Journal on Scientific Computing},
  volume  = {16},
  number  = {5},
  pages   = {1190--1208},
  year    = {1995},
  doi     = {10.1137/0916069}
}

@inproceedings{paszke2019pytorch,
  author    = {Paszke, Adam and Gross, Sam and Massa, Francisco and Lerer, Adam and Bradbury, James and Chanan, Gregory and Killeen, Trevor and Lin, Zeming and Gimelshein, Natalia and Antiga, Luca and Desmaison, Alban and Kopf, Andreas and Yang, Edward and DeVito, Zachary and Raison, Martin and Tejani, Alykhan and Chilamkurthy, Sasank and Steiner, Benoit and Fang, Lu and Bai, Junjie and Chintala, Soumith},
  title     = {{PyTorch}: An Imperative Style, High-Performance Deep Learning Library},
  booktitle = {Advances in Neural Information Processing Systems},
  volume    = {32},
  year      = {2019},
  url       = {https://proceedings.neurips.cc/paper/2019/hash/bdbca288fee7f92f2bfa9f7012727740-Abstract.html}
}

@misc{opsd2020timeseries,
  author       = {{Open Power System Data}},
  title        = {Data Package Time Series},
  year         = {2020},
  howpublished = {Data set},
  note         = {Version 2020-10-06},
  doi          = {10.25832/time_series/2020-10-06},
  url          = {https://data.open-power-system-data.org/time_series/2020-10-06/}
}

@article{bertsimas2020predictive,
  author  = {Bertsimas, Dimitris and Kallus, Nathan},
  title   = {From Predictive to Prescriptive Analytics},
  journal = {Management Science},
  volume  = {66},
  number  = {3},
  pages   = {1025--1044},
  year    = {2020},
  doi     = {10.1287/mnsc.2018.3253}
}

\appendix
\onecolumn

\section{Related Work}
\label{sec:related_work}

\paragraph{ICNN inference and automatic differentiation.}
ICNNs have been used as learned convex objectives in optimal transport,
invertible modeling, system identification, and predictive control
\citep{amos2017icnn,makkuva2020oticnn,huang2021cpflow,korotin2021benchmark,
bunning2021buildingmpc,lawrynczuk2022icnnmpc}. SOC-ICNNs extend this class with
quadratic and second-order-cone modules and admit a structured conic value
representation \citep{LiuHu2026SOCICNN}. Gradient-based inference commonly
relies on derivatives supplied by AD. At nonsmooth
inputs, however, AD follows an executed computational
branch and returns a branch-dependent derivative
\citep{bolte2020ad,bolte2021nonsmoothimplicit,
bolte2023complexity,lee2023correctness}. Such a derivative need not encode the
complete convex subdifferential required for stationarity certification or
directional-derivative calculations. Our work instead exploits the structured
dual representation of SOC-ICNNs to recover the complete set of supporting
slopes active at the current input.

\paragraph{Nonsmooth optimization and sensitivity analysis.}
Classical sensitivity results for marginal functions and parametric
optimization relate optimal primal or dual solutions to subgradients of value
functions \citep{danskin1966,rockafellar1998variational,
bonnans2000perturbation,aravkin2013value}. General nonsmooth optimization
methods use gradient sampling, bundle models, proximal regularization, and
generalized Newton systems to construct reliable search directions and obtain
global or fast local convergence
\citep{burke2005gradient,rockafellar1976proximal,
ali2017newtonadmm,lin2024semismoothot,
milzarek2024ssn,feng2025semismooth,alphonse2026globalized}.
These methods are typically formulated for generic objective or oracle models
and obtain the required nonsmooth information through sampling, local models,
or proximal constructions. In contrast, we recover the exact subdifferential
directly from the internal dual structure of a trained SOC-ICNN and combine it
with feasible-set geometry for certified inference, while using the same
network structure to obtain explicit curvature on smooth regions.

\section{Proofs for the First- and Second-Order Geometry of SOC-ICNNs}
\label{sec:app_geometry_proofs}

This appendix expands the SOC-ICNN model in \cref{sec:preliminaries}
and provides complete proofs for the geometric results in
\cref{sec:geometry}. We first describe the architecture and its convexity,
then establish its exact conic representation, derive the structured dual,
and characterize the complete set of optimal multipliers. We subsequently
prove the subdifferential formula, establish closedness and local
boundedness of the associated mappings, and derive the local Hessian.
The canonical multiplier selector and the blockwise subdifferential
formula are included below.

\subsection{Architecture and notation}
\label{sec:app_soc_icnn_model}

Fix a context \(y\) and trained network parameters \(\theta\).
We suppress this dependence and write \(f(x)=f_\theta(x;y)\), with
\(x\in\mathbb R^d\) the decision input. The network contains a ReLU ICNN
backbone, \(H\) quadratic modules, and \(G\) Euclidean norm modules.
All three components receive \(x\); the quadratic and norm modules act
in parallel with the backbone. Their scalar contributions are added to
form the network output.

\paragraph{ReLU backbone.}
Let \(n_\ell\) be the width of hidden layer \(\ell\), for
\(\ell=1,\ldots,L\). Its preactivation
\(a_\ell(x)\in\mathbb R^{n_\ell}\) is the vector before the nonlinearity,
and its hidden activation \(z_\ell(x)\in\mathbb R^{n_\ell}\) is the vector
after applying ReLU. The recursion is
\begin{equation}
a_\ell(x)=W_\ell x+U_\ell z_{\ell-1}(x)+b_\ell,
~~
z_\ell(x)=\ReLU(a_\ell(x)),
\label{eq:app_soc_backbone_detailed}
\end{equation}
where ReLU acts coordinatewise as \(\ReLU(t)=\max\{0,t\}\).
The matrix \(W_\ell\in\mathbb R^{n_\ell\times d}\) supplies a direct
connection from the decision input to layer \(\ell\).
For \(\ell\geq2\), the matrix
\(U_\ell\in\mathbb R^{n_\ell\times n_{\ell-1}}\) maps the previous
hidden activation to the current layer, and
\(b_\ell\in\mathbb R^{n_\ell}\) is the bias.
We use \(z_0\equiv0\), so the first layer is
\(a_1(x)=W_1x+b_1\) and the term \(U_1z_0\) is omitted.
The matrices \(U_\ell\) are nonnegative coordinatewise; the input
matrices \(W_\ell\) and biases \(b_\ell\) have no sign restrictions.
The scalar contribution from the final hidden layer is
\(c^\top z_L(x)\), where \(c\in\mathbb R_+^{n_L}\).

\paragraph{Quadratic and norm modules.}
For \(h=1,\ldots,H\), quadratic module \(h\) uses the affine residual
\(q_h(x)=B_hx+e_h\), with
\(B_h\in\mathbb R^{m_h\times d}\) and
\(e_h\in\mathbb R^{m_h}\), and contributes
\(\alpha_h\|q_h(x)\|_2^2/2\). Here \(m_h\) is its residual dimension
and \(\alpha_h>0\) is its weight; zero-weight quadratic modules are
omitted. For \(g=1,\ldots,G\), norm module \(g\) uses
\(u_g(x)=A_gx+d_g\), with
\(A_g\in\mathbb R^{k_g\times d}\) and
\(d_g\in\mathbb R^{k_g}\), and contributes
\(\lambda_g\|u_g(x)\|_2\), where \(k_g\) is its residual dimension
and \(\lambda_g\geq0\). The matrices and offsets in these affine
residuals have no sign restrictions.

\paragraph{Output and convexity.}
With \(v\in\mathbb R^d\) and \(b_0\in\mathbb R\) denoting the direct
linear output coefficient and output bias, respectively, the network is
\begin{equation}
f(x)=c^\top z_L(x)+v^\top x+b_0
+\sum_{h=1}^{H}\frac{\alpha_h}{2}\|B_hx+e_h\|_2^2
+\sum_{g=1}^{G}\lambda_g\|A_gx+d_g\|_2.
\label{eq:app_soc_output_detailed}
\end{equation}
This is the model in \eqref{eq:model}. To see its convexity, each
coordinate of \(z_1(x)=\ReLU(W_1x+b_1)\) is convex. If the coordinates
of \(z_{\ell-1}\) are convex, nonnegativity of \(U_\ell\) makes every
coordinate of \(a_\ell\) convex. Taking its maximum with zero preserves
convexity, so the coordinates of \(z_\ell\) are convex by induction.
The nonnegative output weights \(c\) preserve convexity of the scalar
backbone contribution. Each squared norm and norm of an affine map is
convex, and its coefficient is nonnegative. Adding these terms and the
affine output term proves that \(f\) is finite and convex on
\(\mathbb R^d\). The case \(H=G=0\) recovers a standard ReLU ICNN.

\paragraph{Parameters and forward states.}
For a fixed context, the matrices, biases, offsets, and output weights
above remain constant during inference. In contrast,
\(a_\ell(x),z_\ell(x),q_h(x)\), and \(u_g(x)\) vary with the decision
input and are computed during a forward pass. Their distinction is
important for the geometric analysis: the parameters determine the maps
from multiplier space to input space, while the forward states determine
which multiplier conditions apply at the current input. Structural
boundaries occur when a ReLU preactivation is zero or a norm residual
with positive weight is zero. The results below use these states to
recover the complete subdifferential; local Hessians are obtained on the
structurally nondegenerate regions of Definition \ref{def:nondegenerate}.

\subsection{Exact primal conic representation}
\label{sec:app_primal_representation}

Recall the second-order and rotated second-order cone
\(
\mathcal Q^{k+1}
=
\left\{
(t,\omega)\in\R\times\R^k:
\|\omega\|_2\leq t
\right\},
\)
and
\(
\mathcal Q_r^{k+2}
=
\left\{
(s,\tau,\eta):
\|\eta\|_2^2\leq 2s\tau,\;
s\geq0,\;
\tau\geq0
\right\}.
\)

For a fixed decision input \(x\), the SOC-ICNN output in
\eqref{eq:model} admits the equivalent conic representation
\begin{equation}
\begin{aligned}
\min_{z,s,\eta,t,\omega}~
&
c^\top z_L+v^\top x+b_0
+\sum_{h=1}^{H}\alpha_hs_h
+\sum_{g=1}^{G}\lambda_gt_g
\\
\st~
&
z_\ell
\geq
W_\ell x+U_\ell z_{\ell-1}+b_\ell,
~~
z_\ell\geq0,
\\
&
\eta_h=B_hx+e_h,
~~
(s_h,1,\eta_h)\in\mathcal Q_r^{m_h+2},
\\
&
\omega_g=A_gx+d_g,
~~
(t_g,\omega_g)\in\mathcal Q^{k_g+1}.
\end{aligned}
\label{eq:app_primal_conic}
\end{equation}

\begin{lemma}[Exact primal value]
\label{lem:app_exact_primal}
For every fixed \(x\), the optimal value of
\eqref{eq:app_primal_conic} equals the SOC-ICNN output \(f(x)\) in
\eqref{eq:model}. In particular, the forward states
\(z_\ell(x)\), together with
\(
\eta_h=q_h(x),
~
s_h=\frac{1}{2}\|q_h(x)\|_2^2,
~
\omega_g=u_g(x),
~
t_g=\|u_g(x)\|_2,
\)
form an optimal primal solution.
\end{lemma}

\begin{proof}
Fix \(x\). The equality constraints impose
\(
\eta_h=q_h(x),
~
\omega_g=u_g(x).
\)
The rotated-cone constraint
\((s_h,1,\eta_h)\in\mathcal Q_r^{m_h+2}\) implies
\(
\|\eta_h\|_2^2\leq 2s_h,
\)
and hence
\(
s_h\geq\frac{1}{2}\|q_h(x)\|_2^2.
\)
Because \(\alpha_h>0\), the smallest possible contribution of this block is
attained at
\(
s_h=\frac{1}{2}\|q_h(x)\|_2^2.
\)

Likewise,
\((t_g,\omega_g)\in\mathcal Q^{k_g+1}\) implies
\(
t_g\geq\|\omega_g\|_2=\|u_g(x)\|_2.
\)
If \(\lambda_g>0\), the minimum contribution is attained at equality. If
\(\lambda_g=0\), equality remains a feasible optimal choice.

It remains to consider the ReLU backbone. The forward states satisfy
\(
z_\ell(x)
=
\max\left\{
0,\,
W_\ell x+U_\ell z_{\ell-1}(x)+b_\ell
\right\}
\)
componentwise and are therefore feasible.

Let \(\widetilde z\) be any other feasible hidden-state sequence. We show by
induction that
\(
\widetilde z_\ell\geq z_\ell(x)
\)
componentwise for every layer. At the first layer,
\(
\widetilde z_1
\geq
\max\{0,W_1x+b_1\}
=
z_1(x).
\)
Suppose that
\(\widetilde z_{\ell-1}\geq z_{\ell-1}(x)\). Since
\(U_\ell\geq0\),
\(
W_\ell x+U_\ell\widetilde z_{\ell-1}+b_\ell
\geq
W_\ell x+U_\ell z_{\ell-1}(x)+b_\ell.
\)
Together with \(\widetilde z_\ell\geq0\), this gives
\(
\widetilde z_\ell
\geq
\max\left\{
0,\,
W_\ell x+U_\ell z_{\ell-1}(x)+b_\ell
\right\}
=
z_\ell(x).
\)
Thus \(\widetilde z_L\geq z_L(x)\). Since \(c\geq0\),
\(
c^\top\widetilde z_L
\geq
c^\top z_L(x).
\)
The forward states therefore minimize the ReLU block. Adding the optimal
values of the ReLU, quadratic, and norm blocks yields exactly the SOC-ICNN
output in \eqref{eq:model}.
\end{proof}

\subsection{Derivation of the structured dual representation}
\label{sec:app_dual_derivation}

We first derive the reduced dual representation of the ReLU backbone.

\begin{lemma}[Dual representation of the ReLU backbone]
\label{lem:app_relu_dual}
For every fixed \(x\),
\begin{equation}
\begin{aligned}
c^\top z_L(x)+v^\top x+b_0
=
\max_{\nu\in\mathcal K_\nu}
\left\{
v^\top x+b_0
+
\sum_{\ell=1}^{L}
\nu_\ell^\top(W_\ell x+b_\ell)
\right\}.
\end{aligned}
\label{eq:app_relu_dual}
\end{equation}
\end{lemma}

\begin{proof}
For each layer, write the primal inequalities as
\(
W_\ell x+U_\ell z_{\ell-1}+b_\ell-z_\ell\leq0,
~~
-z_\ell\leq0.
\)
Associate nonnegative multipliers \(\nu_\ell\geq0\) and
\(\mu_\ell\geq0\), respectively. The terms in the Lagrangian involving the
hidden variables are
\begin{equation}
\begin{aligned}
\sum_{\ell=1}^{L-1}
z_\ell^\top
\left(
U_{\ell+1}^{\top}\nu_{\ell+1}
-\nu_\ell-\mu_\ell
\right)
+
z_L^\top
\left(
c-\nu_L-\mu_L
\right).
\end{aligned}
\label{eq:app_relu_lagrangian_hidden}
\end{equation}
The infimum over the unrestricted Lagrangian variables \(z_\ell\) is finite if
and only if
\(
\mu_L=c-\nu_L
\)
and
\(
\mu_\ell
=
U_{\ell+1}^{\top}\nu_{\ell+1}-\nu_\ell,
~~
\ell=1,\ldots,L-1.
\)
The conditions \(\nu_\ell\geq0\) and \(\mu_\ell\geq0\) are therefore
equivalent to
\(
0\leq\nu_L\leq c,
~~
0\leq\nu_\ell
\leq
U_{\ell+1}^{\top}\nu_{\ell+1},
~
\ell<L,
\)
which define \(\mathcal K_\nu\).

After eliminating the hidden variables and \(\mu\), the dual objective is
\(
v^\top x+b_0
+
\sum_{\ell=1}^{L}
\nu_\ell^\top(W_\ell x+b_\ell).
\)

The primal ReLU epigraph problem is feasible by the forward construction in
Lemma \ref{lem:app_exact_primal}. It is also bounded below because
\(z_L\geq0\) and \(c\geq0\). Standard linear-programming strong duality
therefore gives \eqref{eq:app_relu_dual}.
\end{proof}

\begin{lemma}[Compactness of the ReLU multiplier set]
\label{lem:app_knu_compact}
The set \(\mathcal K_\nu\) is a nonempty compact convex polytope.
\end{lemma}

\begin{proof}
The zero sequence belongs to \(\mathcal K_\nu\), so the set is nonempty.
It is defined by finitely many linear inequalities and is therefore closed
and convex.

To prove boundedness, define recursively
\(
\beta_L:=c,
~~
\beta_\ell:=U_{\ell+1}^{\top}\beta_{\ell+1},
~
\ell=L-1,\ldots,1.
\)
Because \(c\geq0\) and \(U_\ell\geq0\), every \(\beta_\ell\) is finite and
componentwise nonnegative. For any \(\nu\in\mathcal K_\nu\),
\(
0\leq\nu_L\leq\beta_L.
\)
If \(0\leq\nu_{\ell+1}\leq\beta_{\ell+1}\), then
\(
0
\leq
\nu_\ell
\leq
U_{\ell+1}^{\top}\nu_{\ell+1}
\leq
U_{\ell+1}^{\top}\beta_{\ell+1}
=
\beta_\ell.
\)
Backward induction therefore gives
\(0\leq\nu_\ell\leq\beta_\ell\) for every layer. Hence
\(\mathcal K_\nu\) is bounded and thus compact.
\end{proof}

The quadratic and norm modules admit the following standard representations.

\begin{lemma}[Quadratic conjugate and norm support]
\label{lem:app_module_duals}
For every \(q\in\R^m\), \(u\in\R^k\), \(\alpha>0\), and
\(\lambda\geq0\),
\begin{align}
\frac{\alpha}{2}\|q\|_2^2
&=
\max_{p\in\R^m}
\left\{
p^\top q-\frac{1}{2\alpha}\|p\|_2^2
\right\},
\label{eq:app_quadratic_conjugate}\\
\lambda\|u\|_2
&=
\max_{\|r\|_2\leq\lambda}r^\top u.
\label{eq:app_norm_support}
\end{align}
The maximizer in \eqref{eq:app_quadratic_conjugate} is uniquely
\(p=\alpha q\). The optimizer set in \eqref{eq:app_norm_support} is
\begin{equation}
\argmax_{\|r\|_2\leq\lambda}r^\top u
=
\begin{cases}
\left\{
\displaystyle
\lambda\frac{u}{\|u\|_2}
\right\},
&u\neq0,\\[3mm]
\left\{
r\in\R^k:\|r\|_2\leq\lambda
\right\},
&u=0.
\end{cases}
\label{eq:app_norm_argmax}
\end{equation}
\end{lemma}

\begin{proof}
Completing the square gives
\(
p^\top q-\frac{1}{2\alpha}\|p\|_2^2
=
\frac{\alpha}{2}\|q\|_2^2
-
\frac{1}{2\alpha}\|p-\alpha q\|_2^2.
\)
The maximum is therefore
\(\alpha\|q\|_2^2/2\), uniquely attained at \(p=\alpha q\).

For the norm block, Cauchy--Schwarz gives
\(
r^\top u
\leq
\|r\|_2\|u\|_2
\leq
\lambda\|u\|_2.
\)
When \(u\neq0\), equality is attained uniquely by
\(r=\lambda u/\|u\|_2\). When \(u=0\), every feasible \(r\) attains value
zero.
\end{proof}

\begin{proof}[Proof of Proposition \ref{prop:dual_representation}]
Apply Lemma \ref{lem:app_relu_dual} to the ReLU backbone,
\eqref{eq:app_quadratic_conjugate} to every quadratic module, and
\eqref{eq:app_norm_support} to every norm module. Their dual variables belong
to independent feasible blocks, so the maximum of the sum equals the sum of
the blockwise maxima. This gives
\(
f(x)
=
\max_{\xi\in
\mathcal K_\nu\times\mathcal P\times\mathcal K_r}
\Psi(x;\xi).
\)

Let \(\xi\in\mathcal D^\star(x)\). Its feasibility does not depend on the
input. Therefore, for every \(x'\in\R^d\),
\(
\Psi(x';\xi)\leq f(x').
\)
Optimality at \(x\) gives
\(
\Psi(x;\xi)=f(x).
\)
Since \(\Psi(\cdot;\xi)\) is affine with slope \(\mathcal G(\xi)\),
\(
\Psi(x';\xi)
=
\Psi(x;\xi)
+
\mathcal G(\xi)^\top(x'-x).
\)
Combining these identities yields
\(
f(x')
\geq
f(x)+\mathcal G(\xi)^\top(x'-x),
~~
x'\in\R^d.
\)
Thus \(\Psi(\cdot;\xi)\) is an affine support of \(f\) at \(x\), and
\(\mathcal G(\xi)\in\partial f(x)\).
\end{proof}

\subsection{complete dual solution set}
\label{sec:app_optimal_multiplier_set}

\begin{proof}[Proof of Proposition \ref{prop:blockwise_optimal_set}]
The dual feasible set is a Cartesian product, and \(\Psi\) is a sum of
functions depending on separate multiplier blocks. Consequently, its argmax
factorizes into the Cartesian product of the blockwise argmax sets.

For a quadratic module, Lemma \ref{lem:app_module_duals} gives the unique
optimizer
\(
p_h^\star(x)=\alpha_hq_h(x).
\)
For a norm module, the same lemma gives exactly
\(\mathcal R_g^\star(x)\).

It remains to characterize the ReLU multipliers. Consider the forward primal
states from Lemma \ref{lem:app_exact_primal}. The two inequalities associated
with coordinate \((\ell,i)\) are
\(
a_{\ell,i}(x)-z_{\ell,i}(x)\leq0,
~~
-z_{\ell,i}(x)\leq0,
\)
with multipliers \(\nu_{\ell,i}\) and \(\mu_{\ell,i}\).

If \(a_{\ell,i}(x)<0\), then \(z_{\ell,i}(x)=0\), and the first inequality
has strictly negative slack. Complementary slackness therefore gives
\(
\nu_{\ell,i}=0.
\)

If \(a_{\ell,i}(x)>0\), then
\(z_{\ell,i}(x)=a_{\ell,i}(x)>0\). The second inequality is strictly slack,
so \(\mu_{\ell,i}=0\). Lagrangian stationarity from
\eqref{eq:app_relu_lagrangian_hidden} gives
\(
\nu_{\ell,i}
=
\overline\nu_{\ell,i}(\nu).
\)

If \(a_{\ell,i}(x)=0\), both primal inequalities are active. Dual feasibility
and stationarity impose only
\(
0
\leq
\nu_{\ell,i}
\leq
\overline\nu_{\ell,i}(\nu).
\)
These conditions are also sufficient: together with the forward primal
states, they satisfy primal feasibility, dual feasibility, stationarity, and
complementary slackness. They therefore characterize the complete ReLU
optimal set.

Combining the three blocks proves the stated product representation. The
ReLU optimizer set is a nonempty closed convex subset of the compact polytope
\(\mathcal K_\nu\): equivalently, it is the argmax set of a linear function
over \(\mathcal K_\nu\). Each quadratic optimizer set is a singleton, and
each norm optimizer set is either a singleton or a compact Euclidean ball.
Their finite Cartesian product is therefore nonempty, compact, and convex.
\end{proof}

\subsubsection{Canonical minimum-multiplier selector}

Although DCI minimizes the norm of the input-space readout rather than the
multiplier norm, a canonical multiplier provides a deterministic subgradient representative.

Define
\begin{align}
\widehat\nu_L(x)
&=
c\odot\mathbf 1_{\{a_L(x)>0\}},
\label{eq:canonical_relu_last}\\
\widehat\nu_\ell(x)
&=
\left(
U_{\ell+1}^{\top}\widehat\nu_{\ell+1}(x)
\right)
\odot
\mathbf 1_{\{a_\ell(x)>0\}},
~~
\ell=L-1,\ldots,1,
\label{eq:canonical_relu_recursion}
\end{align}
and
\begin{equation}
\widehat r_g(x)
=
\begin{cases}
\displaystyle
\lambda_g\frac{u_g(x)}{\|u_g(x)\|_2},
&u_g(x)\neq0,\\[3mm]
0,
&u_g(x)=0.
\end{cases}
\label{eq:canonical_norm_selector}
\end{equation}
Together with
\(\widehat p_h(x)=\alpha_hq_h(x)\), define
\(
\widehat\xi(x)
:=
\left(
\widehat\nu(x),
\widehat p(x),
\widehat r(x)
\right).
\)

\begin{corollary}[Canonical multiplier selector]
\label{cor:canonical_multiplier}
For every \(x\), the canonical multiplier is the unique minimum-Euclidean-norm
element of the complete dual solution set:
\begin{equation}
\widehat\xi(x)
=
\argmin_{\xi\in\mathcal D^\star(x)}
\|\xi\|_2.
\label{eq:canonical_minimum_multiplier}
\end{equation}
\end{corollary}

\begin{proof}
We first verify that \(\widehat\xi(x)\in\mathcal D^\star(x)\). The
quadratic and norm blocks satisfy the characterization in
Proposition \ref{prop:blockwise_optimal_set} by construction. For the ReLU block,
\(\widehat\nu_L\) equals \(c\) on strictly active final-layer coordinates
and zero otherwise. Recursively, \(\widehat\nu_\ell\) equals its induced
upper bound on strictly active coordinates and zero on inactive or
zero-preactivation coordinates. Hence it satisfies every equality and interval
condition in \(\mathcal D_{\mathrm R}^\star(x)\).

The quadratic multipliers are fixed throughout
\(\mathcal D^\star(x)\). For a zero norm residual, the optimizer set is a
Euclidean ball centered at the origin, whose unique minimum-norm element is
zero. For a nonzero norm residual, the multiplier is fixed.

It remains to prove minimality for the ReLU block. We first show that
\(
\widehat\nu_\ell(x)
\leq
\nu_\ell
\)
componentwise for every
\(\nu\in\mathcal D_{\mathrm R}^\star(x)\).

At the final layer:
if \(a_{L,i}(x)>0\), then both
\(\widehat\nu_{L,i}\) and \(\nu_{L,i}\) equal \(c_i\);
if \(a_{L,i}(x)<0\), both equal zero;
if \(a_{L,i}(x)=0\), then
\(\widehat\nu_{L,i}=0\leq\nu_{L,i}\).
Thus \(\widehat\nu_L\leq\nu_L\).

Suppose now that
\(\widehat\nu_{\ell+1}\leq\nu_{\ell+1}\). Since
\(U_{\ell+1}\geq0\),
\(
U_{\ell+1}^{\top}\widehat\nu_{\ell+1}
\leq
U_{\ell+1}^{\top}\nu_{\ell+1}.
\)
If \(a_{\ell,i}(x)>0\), both multipliers equal their respective recursive
upper bounds, so
\(
\widehat\nu_{\ell,i}
=
\left(
U_{\ell+1}^{\top}\widehat\nu_{\ell+1}
\right)_i
\leq
\left(
U_{\ell+1}^{\top}\nu_{\ell+1}
\right)_i
=
\nu_{\ell,i}.
\)
If \(a_{\ell,i}(x)<0\), both are zero. If
\(a_{\ell,i}(x)=0\), then
\(\widehat\nu_{\ell,i}=0\leq\nu_{\ell,i}\).
Backward induction proves
\(
0\leq\widehat\nu\leq\nu
\)
componentwise for every optimal ReLU multiplier.

Consequently,
\(
\|\widehat\nu\|_2\leq\|\nu\|_2,
\)
with strict inequality whenever \(\nu\neq\widehat\nu\). Combining the ReLU,
quadratic, and norm blocks proves that \(\widehat\xi(x)\) is the unique
minimum-norm element of \(\mathcal D^\star(x)\).
\end{proof}

The corresponding canonical input-space slope is
\begin{equation}
\begin{aligned}
g^{\mathrm{can}}(x)
:={}
\mathcal G\bigl(\widehat\xi(x)\bigr)
={}
v
+\sum_{\ell=1}^{L}
W_\ell^\top\widehat\nu_\ell(x)
+\sum_{h=1}^{H}
\alpha_hB_h^\top q_h(x)
+\sum_{g=1}^{G}
A_g^\top\widehat r_g(x).
\end{aligned}
\label{eq:canonical_readout}
\end{equation}
This is one valid supporting slope, but it does not generally minimize the
norm of the readout in input space.

\subsection{Exact subdifferential and directional readout}
\label{sec:app_exact_subdifferential}

\begin{proof}[Proof of Theorem \ref{thm:exact_subdifferential}]
Fix \(x\in\R^d\). By Proposition \ref{prop:dual_representation}, every
\(\xi\in\mathcal D^\star(x)\) defines an affine support of \(f\) at \(x\).
Hence
\begin{equation}
\mathcal G\bigl(\mathcal D^\star(x)\bigr)
\subseteq
\partial f(x).
\label{eq:app_readout_inclusion}
\end{equation}

We next justify the application of Danskin's theorem despite the unbounded
quadratic multiplier space \(\mathcal P\). Choose a compact neighborhood
\(X\subset\R^d\) containing \(x\). Since every \(q_h\) is affine,
\(
C_h
:=
\max_{x'\in X}
\alpha_h\|q_h(x')\|_2
<\infty.
\)
Choose \(R_h>C_h\) and define
\(
\mathcal P_X
:=
\prod_{h=1}^{H}
\left\{
p_h\in\R^{m_h}:
\|p_h\|_2\leq R_h
\right\}.
\)
For every \(x'\in X\), the unique maximizing quadratic multiplier is
\(
p_h^\star(x')=\alpha_hq_h(x'),
\)
which lies in the interior of the corresponding ball. Therefore, throughout
\(X\),
\(
f(x')
=
\max_{\xi\in
\mathcal K_\nu\times\mathcal P_X\times\mathcal K_r}
\Psi(x';\xi).
\)
The restricted multiplier set is compact and independent of \(x'\), and
its optimal set agrees with the original optimal set for every \(x'\in X\).

Danskin's directional theorem now yields, for every direction \(d\),
\(
f'(x;d)
=
\max_{\xi\in\mathcal D^\star(x)}
\nabla_x\Psi(x;\xi)^\top d.
\)
Since
\(\nabla_x\Psi(x;\xi)=\mathcal G(\xi)\),
\(
f'(x;d)
=
\max_{\xi\in\mathcal D^\star(x)}
\mathcal G(\xi)^\top d.
\)

It remains to identify the full subdifferential. By
Proposition \ref{prop:blockwise_optimal_set}, the set
\(\mathcal D^\star(x)\) is nonempty, compact, and convex. Because
\(\mathcal G\) is affine,
\(
C(x)
:=
\mathcal G\bigl(\mathcal D^\star(x)\bigr)
\)
is also nonempty, compact, and convex. Its support function satisfies
\(
\sigma_{C(x)}(d)
=
\max_{g\in C(x)}g^\top d
=
f'(x;d).
\)
For a finite convex function, the directional derivative is also the support
function of its subdifferential:
\(
f'(x;d)
=
\max_{g\in\partial f(x)}g^\top d.
\)
Thus \(C(x)\) and \(\partial f(x)\) are nonempty closed convex sets with
identical support functions. They must therefore coincide:
\(
\partial f(x)
=
\mathcal G\bigl(\mathcal D^\star(x)\bigr).
\)

Finally, a finite convex function is differentiable at \(x\) if and only if
its subdifferential is a singleton. The differentiability statement follows
immediately.
\end{proof}

\subsubsection{Blockwise subdifferential decomposition}

Define the ReLU readout set
\begin{equation}
\mathcal S_{\mathrm R}(x)
:=
\left\{
\sum_{\ell=1}^{L}W_\ell^\top\nu_\ell:
\nu\in\mathcal D_{\mathrm R}^\star(x)
\right\}.
\label{eq:relu_readout_set}
\end{equation}

\begin{corollary}[Blockwise subdifferential decomposition]
\label{cor:blockwise_subdifferential}
For every \(x\),
\begin{equation}
\begin{aligned}
\partial f(x)
={}
v+\mathcal S_{\mathrm R}(x)
+\sum_{h=1}^{H}
\alpha_hB_h^\top q_h(x)
+
\sum_{g=1}^{G}
A_g^\top\mathcal R_g^\star(x),
\end{aligned}
\label{eq:blockwise_subdifferential}
\end{equation}
where all additions involving sets are Minkowski sums.
\end{corollary}

\begin{proof}
Insert the product representation from
Proposition \ref{prop:blockwise_optimal_set} into the affine readout map
\(\mathcal G\). The image of a Cartesian product under a sum of blockwise
linear maps is the Minkowski sum of the corresponding block images. The
quadratic block is single-valued, while the ReLU and norm blocks may be
set-valued. Applying Theorem \ref{thm:exact_subdifferential} gives
\eqref{eq:blockwise_subdifferential}.
\end{proof}

The decomposition shows that quadratic modules contribute a unique smooth
slope, zero ReLU preactivations contribute recursively coupled interval-valued
slopes, and zero norm residuals contribute affine images of Euclidean balls.

\subsubsection{Context-conditioned extension}

\begin{corollary}[Fixed-context geometry]
\label{cor:context_conditioned_geometry}
Suppose that the parameters of the SOC-ICNN depend on a context \(y\),
while the convexity sign conditions in \eqref{eq:model} are preserved for every
admissible context. For each fixed \(y\), all results in
Proposition \ref{prop:blockwise_optimal_set} and \cref{thm:exact_subdifferential} hold for the partial subdifferential with respect
to the decision input:
\(
\partial_x f(x;y)
=
\mathcal G_y\bigl(\mathcal D^\star(x;y)\bigr).
\)
\end{corollary}

\begin{proof}
Fix \(y\). All context-dependent matrices and offsets then become constants,
and the resulting function of \(x\) has exactly the SOC-ICNN form analyzed
above. Applying the preceding results and restoring the dependence on \(y\)
gives the claim.
\end{proof}

\subsection{Stability of the optimal-dual and readout correspondences}
\label{sec:app_readout_stability}

\begin{proof}[Proof of Proposition \ref{prop:readout_stability}]
We prove nonemptiness, compactness, local boundedness, and closedness in turn.

\paragraph{Nonempty compact values.}
Nonemptiness follows directly from the explicit blockwise characterization in
Proposition \ref{prop:blockwise_optimal_set}. The ReLU block is a closed subset of the
compact set \(\mathcal K_\nu\), each norm block is either a singleton or a
closed Euclidean ball, and each quadratic block is the singleton
\(\{\alpha_hq_h(x)\}\). Hence \(\mathcal D^\star(x)\) is compact for every
\(x\).

\paragraph{Local boundedness.}
Fix \(\bar x\). Let \(X\) be a compact neighborhood of \(\bar x\). The ReLU
multipliers lie in the input-independent compact set \(\mathcal K_\nu\), and
the norm multipliers lie in input-independent compact balls. For the quadratic
multipliers,
\(
p_h^\star(x)=\alpha_hq_h(x).
\)
Since \(q_h\) is affine, it is bounded on \(X\). Thus there exists a common bounded set containing \(\mathcal D^\star(x)\) for every \(x\in X\).

\paragraph{Closed graph.}
Suppose
\(
x_n\to x,
~~
\xi_n\to\xi,
~~
\xi_n\in\mathcal D^\star(x_n).
\)
The dual feasible set is closed and independent of \(x\), so \(\xi\) is
feasible. Since each \(\xi_n\) is optimal,
\(
\Psi(x_n;\xi_n)=f(x_n).
\)
Both \(f\) and \(\Psi\) are continuous. Passing to the limit gives
\(
\Psi(x;\xi)=f(x).
\)
Thus \(\xi\in\mathcal D^\star(x)\), proving that the graph of
\(\mathcal D^\star\) is closed.

We now consider the readout image. Suppose
\(
x_n\to x,
~~
g_n\to g,
~~
g_n\in
\mathcal G\bigl(\mathcal D^\star(x_n)\bigr).
\)
Choose \(\xi_n\in\mathcal D^\star(x_n)\) such that
\(g_n=\mathcal G(\xi_n)\). Local boundedness gives a convergent subsequence,
which we denote by
\(
\xi_{n_j}\to\bar\xi.
\)
Closedness of the multiplier graph gives
\(
\bar\xi\in\mathcal D^\star(x).
\)
Since \(\mathcal G\) is continuous,
\[
g
=
\lim_{j\to\infty}\mathcal G(\xi_{n_j})
=
\mathcal G(\bar\xi)
\in
\mathcal G\bigl(\mathcal D^\star(x)\bigr).
\]
Hence the readout correspondence has a closed graph. Its values are nonempty
and compact because they are continuous images of nonempty compact sets. It is
locally bounded because \(\mathcal G\) is affine and
\(\mathcal D^\star\) is locally bounded. By
\cref{thm:exact_subdifferential}, this readout correspondence is exactly
\(x\mapsto\partial f(x)\).
\end{proof}

\subsection{Smooth branches and explicit local curvature}
\label{sec:app_smooth_geometry}

Let
\(
\mathcal I_+
:=
\left\{
g\in\{1,\ldots,G\}:
\lambda_g>0
\right\}.
\)

\begin{proof}[Proof of Proposition \ref{prop:smooth_branch_geometry}]
Let \(\bar x\) be structurally nondegenerate; thus every ReLU preactivation
is nonzero at \(\bar x\), and \(u_g(\bar x)\neq0\) for every
\(g\in\mathcal I_+\). The finite collection
\(
\left\{
|a_{\ell,i}(\bar x)|:
\ell=1,\ldots,L,\ i
\right\}
\cup
\left\{
\|u_g(\bar x)\|_2:
g\in\mathcal I_+
\right\}
\)
contains only strictly positive numbers. Since all preactivations and affine
residuals are continuous, there exists \(\rho>0\) such that, for every
\(x\in\mathbb B(\bar x,\rho)\),

1. each \(a_{\ell,i}(x)\) has the same sign as
   \(a_{\ell,i}(\bar x)\); and
2. \(u_g(x)\neq0\) for every \(g\in\mathcal I_+\).

Thus the ReLU activation pattern is fixed throughout the neighborhood.
Replacing every active ReLU by the identity and every inactive ReLU by zero,
a layerwise induction shows that each hidden state is affine in \(x\).
Consequently, the entire ReLU-output component, including the affine term
\(v^\top x+b_0\), can be written as
\(
\bar a^\top x+\bar b
\)
for constants \(\bar a\in\R^d\) and \(\bar b\in\R\).

The objective therefore has the local form
\[
f(x)
=
\bar a^\top x+\bar b
+
\sum_{h=1}^{H}
\frac{\alpha_h}{2}\|q_h(x)\|_2^2
+
\sum_{g\in\mathcal I_+}
\lambda_g\|u_g(x)\|_2.
\]
The blockwise multiplier characterization also shows that every multiplier is
uniquely determined in this neighborhood: no ReLU preactivation is zero, all
positive-weight norm residuals are nonzero, zero-weight norm balls reduce to
the singleton \(\{0\}\), and the quadratic multipliers are always unique.
Hence \(\mathcal D^\star(x)\) is a singleton.

For a quadratic module,
\(
\nabla_x
\frac{\alpha_h}{2}\|B_hx+e_h\|_2^2
=
\alpha_hB_h^\top q_h(x),
\)
and
\(
\nabla_x^2
\frac{\alpha_h}{2}\|B_hx+e_h\|_2^2
=
\alpha_hB_h^\top B_h.
\)

For a positive-weight norm module, write
\(
u=u_g(x),
~~
A=A_g,
~~
\widehat u=\frac{u}{\|u\|_2}.
\)
Because \(u\neq0\),
\(
\nabla_x\lambda_g\|u\|_2
=
\lambda_gA^\top\widehat u.
\)
For an arbitrary perturbation \(\delta u\),
\begin{align}
D\widehat u[\delta u]
&=
D\left(
\frac{u}{\|u\|_2}
\right)[\delta u]
\nonumber
=
\frac{\delta u}{\|u\|_2}
-
\frac{u}{\|u\|_2^2}
\frac{u^\top\delta u}{\|u\|_2}
\nonumber
=
\frac{1}{\|u\|_2}
\left(
I-\widehat u\widehat u^\top
\right)\delta u.
\label{eq:app_normalized_residual_derivative}
\end{align}
Since \(\delta u=A\delta x\),
\(
D\left(
\lambda_gA^\top\widehat u
\right)[\delta x]
=
\lambda_gA^\top
\frac{
I-\widehat u\widehat u^\top
}{
\|u\|_2
}
A\delta x.
\)

Summing all blocks gives
\[
\nabla f(x)
=
\bar a
+
\sum_{h=1}^{H}
\alpha_hB_h^\top q_h(x)
+
\sum_{g\in\mathcal I_+}
\lambda_gA_g^\top\widehat u_g(x),
\]
and
\[
\begin{aligned}
\nabla^2 f(x)
={}
\sum_{h=1}^{H}
\alpha_hB_h^\top B_h
+
\sum_{g\in\mathcal I_+}
\lambda_gA_g^\top
\frac{
I-\widehat u_g(x)\widehat u_g(x)^\top
}{
\|u_g(x)\|_2
}
A_g.
\end{aligned}
\]

Each quadratic term is positive semidefinite. Moreover,
\(
I-\widehat u_g\widehat u_g^\top
\)
is the orthogonal projector onto the subspace perpendicular to
\(\widehat u_g\), and is therefore positive semidefinite. Hence
\(
\nabla^2f(x)\succeq0.
\)

It remains to establish local Lipschitz continuity of the Hessian. Choose a
closed ball
\(
\overline{\mathbb B}(\bar x,\rho')
\subset
\mathbb B(\bar x,\rho).
\)
For every \(g\in\mathcal I_+\), continuity and compactness imply the existence
of \(\delta_g>0\) such that
\[
\|u_g(x)\|_2\geq\delta_g
~~
\forall~
x\in\overline{\mathbb B}(\bar x,\rho').
\]
The mapping
\[
u
\mapsto
\frac{
I-(u/\|u\|_2)(u/\|u\|_2)^\top
}{
\|u\|_2
}
\]
is continuously differentiable on \(\R^{k_g}\setminus\{0\}\). The compact
convex image of \(\overline{\mathbb B}(\bar x,\rho')\) under the affine map
\(u_g\) stays a positive distance from the origin. The derivative is therefore
bounded on this image, and the mean-value inequality shows that the mapping is
Lipschitz there. Composition with the affine map \(u_g(x)\), multiplication by
fixed matrices, and finite summation preserve Lipschitz continuity. The
quadratic contribution to the Hessian is constant. Therefore
\(\nabla^2f\) is Lipschitz continuous on every sufficiently small closed ball
contained in the nondegenerate neighborhood.
\end{proof}

\subsection{Geometry of the complete inference objective}
\label{sec:app_complete_objective_geometry}

Since \(\Phi:\R^d\to\R\) is finite, convex, and continuously
differentiable, the convex sum rule gives
\[
\partial F(x)
=
\partial f(x)+\nabla\Phi(x)
=
\mathcal G\bigl(\mathcal D^\star(x)\bigr)
+\nabla\Phi(x).
\]
If \(\Phi\) is twice continuously differentiable on a structurally
nondegenerate neighborhood, then
\(
\nabla F(x)
=
\nabla f(x)+\nabla\Phi(x),
\)
and
\(
\nabla^2F(x)
=
\nabla^2f(x)+\nabla^2\Phi(x).
\)
If, in addition, \(\nabla^2\Phi\) is locally Lipschitz, then
\(\nabla^2F\) is locally Lipschitz on the same neighborhood.

\section{Proofs for Dual-Certified Inference}
\label{sec:app_dci}

This section provides complete proofs for the constrained DCI results in
\cref{sec:dci}. Throughout, \(\mathcal X\subseteq\R^d\) is a nonempty closed
convex set and
\(
F(x)=f(x)+\Phi(x)
\)
is finite and convex on \(\R^d\). The feasible retraction
\(\mathsf P_{\mathcal X}\) satisfies \eqref{eq:feasible_retraction}.
We use the normal-cone convention
\[
N_{\mathcal X}(x)
=
\left\{
n\in\R^d:
n^\top(z-x)\leq0
~\forall~z\in\mathcal X
\right\},
\]
and define the tangent cone by
\(
T_{\mathcal X}(x)
=
\operatorname{cl}
\left\{
t(z-x):
t\geq0,\ z\in\mathcal X
\right\}.
\)
For a closed convex set,
\begin{equation}
T_{\mathcal X}(x)
=
N_{\mathcal X}(x)^\circ
=
\left\{
d:
n^\top d\leq0
~\forall~n\in N_{\mathcal X}(x)
\right\}.
\label{eq:app_tangent_normal_polar}
\end{equation}

\subsection{The constrained first-order set}
\label{sec:app_constrained_first_order}

Recall that
\(
\partial F(x)
=
\mathcal G\bigl(\mathcal D^\star(x)\bigr)
+\nabla\Phi(x).
\)
For \(x\in\mathcal X\), define
\(
\mathcal C_{\mathcal X}(x)
:=
\partial F(x)+N_{\mathcal X}(x).
\)

\begin{lemma}[Geometry of the constrained first-order set]
\label{lem:app_constrained_set}
For every \(x\in\mathcal X\), the set
\(\mathcal C_{\mathcal X}(x)\) is nonempty, closed, and convex. Consequently,
for every \(M\succ0\), the metric projection
\begin{equation}
p_M(x)
:=
\argmin_{p\in\mathcal C_{\mathcal X}(x)}
\frac{1}{2}\|p\|_{M^{-1}}^2
\label{eq:app_projected_constrained_vector}
\end{equation}
exists and is unique.
\end{lemma}

\begin{proof}
Because \(F\) is finite and convex on \(\R^d\),
\(\partial F(x)\) is nonempty, compact, and convex for every \(x\). The normal
cone \(N_{\mathcal X}(x)\) is a nonempty closed convex cone because it
contains the origin.

The Minkowski sum of two convex sets is convex. Moreover, the sum of a compact
set and a closed set is closed in finite-dimensional space. Therefore,
\(
\mathcal C_{\mathcal X}(x)
=
\partial F(x)+N_{\mathcal X}(x)
\)
is nonempty, closed, and convex.

The function
\(
p\mapsto\frac{1}{2}p^\top M^{-1}p
\)
is continuous, coercive, and strictly convex because \(M\succ0\). It therefore
has a unique minimizer over the nonempty closed convex set
\(\mathcal C_{\mathcal X}(x)\).
\end{proof}

By \eqref{eq:full_objective_subdifferential}, every
\(s\in\partial F(x)\) can be represented as
\(
s
=
\mathcal G(\xi)+\nabla\Phi(x)
\)
for some \(\xi\in\mathcal D^\star(x)\). Hence the projection in
\eqref{eq:app_projected_constrained_vector} is equivalent to
\begin{equation}
\min_{\substack{
\xi\in\mathcal D^\star(x)\\
n\in N_{\mathcal X}(x)
}}
\frac{1}{2}
\left\|
\mathcal G(\xi)+\nabla\Phi(x)+n
\right\|_{M^{-1}}^2.
\label{eq:app_multiplier_normal_projection}
\end{equation}
Although the optimizing pair \((\xi,n)\) may be nonunique, its readout
\(
p_M(x)
=
\mathcal G(\xi)+\nabla\Phi(x)+n
\)
is unique.
The notation in the main text satisfies
\(
g_M^\dagger(x)=p_M(x),
~
\mathcal R_M(x)=\|p_M(x)\|_{M^{-1}}.
\)

\subsection{Constrained certificate and common descent}
\label{sec:app_certificate}

\begin{proof}[Proof of \cref{thm:certificate}]
Fix \(x\in\mathcal X\), let
\(
p:=g_M^\dagger(x),
~
q:=M^{-1}p,
~
d:=-q.
\)

\paragraph{Exact stationarity certificate.}
By definition,
\(
\mathcal R_M(x)=0
~\Longleftrightarrow~
p=0
~\Longleftrightarrow~
0\in\partial F(x)+N_{\mathcal X}(x).
\)
The convex Fermat rule gives
\(
0\in\partial F(x)+N_{\mathcal X}(x)
~\Longleftrightarrow~
x\in\argmin_{z\in\mathcal X}F(z).
\)
This proves
\(
\mathcal R_M(x)=0
\Longleftrightarrow
x\in\argmin_{z\in\mathcal X}F(z).
\)

\paragraph{Projection variational inequality.}
Since \(p\) is the metric projection of the origin onto
\(\mathcal C_{\mathcal X}(x)\), it satisfies
\begin{equation}
(c-p)^\top M^{-1}p
\geq0
~~
\forall~
c\in\mathcal C_{\mathcal X}(x).
\label{eq:app_projection_vi}
\end{equation}

For an arbitrary \(g\in\partial F(x)\), the vector
\[
c=g+0
\]
belongs to \(\mathcal C_{\mathcal X}(x)\), since
\(0\in N_{\mathcal X}(x)\). Substituting \(c=g\) into
\eqref{eq:app_projection_vi} gives
\[
g^\top M^{-1}p
\geq
p^\top M^{-1}p
=
\mathcal R_M(x)^2.
\]
Because \(d=-M^{-1}p\),
\[
g^\top d
\leq
-\mathcal R_M(x)^2
~~
\forall~g\in\partial F(x),
\]
which proves \eqref{eq:common_descent}.

\paragraph{Tangential feasibility.}
Choose one decomposition
\[
p=\bar g+\bar n,
~~
\bar g\in\partial F(x),
~~
\bar n\in N_{\mathcal X}(x).
\]
For any \(n\in N_{\mathcal X}(x)\), the vector
\(\bar g+n\) belongs to \(\mathcal C_{\mathcal X}(x)\). Applying
\eqref{eq:app_projection_vi} with \(c=\bar g+n\) gives
\[
(n-\bar n)^\top q
\geq0
~~
\forall~n\in N_{\mathcal X}(x).
\]
Taking \(n=0\) yields
\[
\bar n^\top q\leq0,
\]
whereas taking \(n=2\bar n\) yields
\[
\bar n^\top q\geq0.
\]
Therefore,
\[
\bar n^\top q=0.
\]
Returning to the general inequality gives
\[
n^\top q\geq0
~~
\forall~n\in N_{\mathcal X}(x).
\]
Equivalently,
\[
n^\top d\leq0
~~
\forall~n\in N_{\mathcal X}(x).
\]
By \eqref{eq:app_tangent_normal_polar},
\[
d\in T_{\mathcal X}(x).
\]

\paragraph{Directional descent.}
For a finite convex function,
\[
F'(x;d)
=
\max_{g\in\partial F(x)}g^\top d.
\]
The common-descent inequality therefore implies
\[
F'(x;d)
\leq
-\mathcal R_M(x)^2.
\]
If \(\mathcal R_M(x)>0\), the right-hand side is strictly negative, proving
\[
F'\bigl(x;d_M(x)\bigr)
\leq
-\mathcal R_M(x)^2
<0.
\]
\end{proof}

\begin{corollary}[Residual objective-gap bound]
\label{cor:app_residual_gap}
Suppose that \(F\) is \(\mu_M\)-strongly convex on \(\mathcal X\) with
respect to the \(M\)-norm. Then, with
\(F^\star:=\min_{z\in\mathcal X}F(z)\),
\[
0
\leq
F(x)-F^\star
\leq
\frac{\mathcal R_M(x)^2}{2\mu_M}.
\]
\end{corollary}

\begin{proof}
Let \(p=g_M^\dagger(x)\) and choose a decomposition
\(
p=\bar g+\bar n
\)
with
\(
\bar g\in\partial F(x)
\)
and
\(
\bar n\in N_{\mathcal X}(x).
\)
Strong convexity gives, for every \(z\in\mathcal X\),
\[
F(z)
\geq
F(x)
+
\bar g^\top(z-x)
+
\frac{\mu_M}{2}\|z-x\|_M^2.
\]
Since
\(
\bar n^\top(z-x)\leq0,
\)
we have
\[
\bar g^\top(z-x)
=
p^\top(z-x)-\bar n^\top(z-x)
\geq
p^\top(z-x).
\]
Hence
\[
F(z)
\geq
F(x)
+
p^\top(z-x)
+
\frac{\mu_M}{2}\|z-x\|_M^2.
\]
Minimizing the right-hand side over all \(z\in\R^d\), which can only lower
the minimum relative to restricting \(z\) to \(\mathcal X\), gives
\begin{align*}
F^\star
&\geq
F(x)
+
\min_{s\in\R^d}
\left\{
p^\top s
+
\frac{\mu_M}{2}s^\top Ms
\right\}\\
&=
F(x)
-
\frac{1}{2\mu_M}p^\top M^{-1}p\\
&=
F(x)
-
\frac{\mathcal R_M(x)^2}{2\mu_M}.
\end{align*}
Rearranging proves the claim.
\end{proof}

\subsubsection{Equivalent constrained steepest-descent problem}

The following proposition establishes the equivalence between the
direction model in \eqref{eq:dci_direction_model} and the joint
projection in \eqref{eq:constrained_projection}.

\begin{proposition}[Metric steepest feasible descent]
\label{prop:app_steepest_feasible_descent}
Let \(p=g_M^\dagger(x)\) and
\(d_M(x)=-M^{-1}p\). Then
\begin{equation}
d_M(x)
=
\argmin_{d\in T_{\mathcal X}(x)}
\left\{
F'(x;d)
+
\frac{1}{2}\|d\|_M^2
\right\}.
\label{eq:app_steepest_feasible_problem}
\end{equation}
The optimal value is
\[
-\frac{1}{2}\mathcal R_M(x)^2.
\]
\end{proposition}

\begin{proof}
For a nonempty set \(S\subseteq\R^d\), let
\(\sigma_S(d):=\sup_{s\in S}s^\top d\) denote its support function,
and let \(\delta_S\) denote its indicator function, equal to zero
on \(S\) and \(+\infty\) elsewhere.
Because
\[
T_{\mathcal X}(x)=N_{\mathcal X}(x)^\circ,
\]
the support function of the normal cone satisfies
\[
\sigma_{N_{\mathcal X}(x)}(d)
=
\begin{cases}
0, & d\in T_{\mathcal X}(x),\\
+\infty, & d\notin T_{\mathcal X}(x).
\end{cases}
\]
Therefore,
\begin{align*}
F'(x;d)+\delta_{T_{\mathcal X}(x)}(d)
&=
\sigma_{\partial F(x)}(d)
+
\sigma_{N_{\mathcal X}(x)}(d)\\
&=
\sigma_{\mathcal C_{\mathcal X}(x)}(d).
\end{align*}
Thus the problem in \eqref{eq:app_steepest_feasible_problem} is equivalent to
\[
\min_{d\in\R^d}
\left\{
\sigma_{\mathcal C_{\mathcal X}(x)}(d)
+
\frac{1}{2}\|d\|_M^2
\right\}.
\]

For every \(d\),
\[
\sigma_{\mathcal C_{\mathcal X}(x)}(d)
\geq
p^\top d.
\]
Hence
\[
\sigma_{\mathcal C_{\mathcal X}(x)}(d)
+
\frac{1}{2}\|d\|_M^2
\geq
p^\top d+\frac{1}{2}d^\top Md
\geq
-\frac{1}{2}p^\top M^{-1}p.
\]
At \(d=-M^{-1}p\), the projection variational inequality implies
\[
c^\top d
\leq
p^\top d
~~
\forall~c\in\mathcal C_{\mathcal X}(x).
\]
Since \(p\in\mathcal C_{\mathcal X}(x)\), this gives
\[
\sigma_{\mathcal C_{\mathcal X}(x)}(d)
=
p^\top d.
\]
Therefore the lower bound is attained at
\(d=-M^{-1}p\), with value
\[
-\frac{1}{2}p^\top M^{-1}p
=
-\frac{1}{2}\mathcal R_M(x)^2.
\]
Strict convexity of the quadratic term gives uniqueness.
\end{proof}

\subsection{Boundary-local projection}
\label{sec:app_boundary_local}

Assume that, at the current point \(x\), the feasible-set normal cone admits
the representation
\begin{equation}
N_{\mathcal X}(x)
=
\left\{
B_{\mathcal X}(x)\zeta:
\zeta\in\mathcal K_{\mathcal X}(x)
\right\},
\label{eq:app_general_normal_representation}
\end{equation}
where
\(B_{\mathcal X}(x)\in\R^{d\times q_{\mathcal X}(x)}\) and
\(\mathcal K_{\mathcal X}(x)\subseteq\R^{q_{\mathcal X}(x)}\) is a simple
closed convex cone. When \(N_{\mathcal X}(x)=\{0\}\), we take
\(q_{\mathcal X}(x)=0\).

For a polyhedron
\(
\mathcal X=\{x\in\R^d:Cx\leq h\},
\)
let
\(
\mathcal I_{\mathcal X}(x)
=
\{j:C_jx=h_j\}.
\)
Then
\begin{equation}
N_{\mathcal X}(x)
=
\left\{
C_{\mathcal I_{\mathcal X}(x)}^\top\zeta:
\zeta\in\R_+^{|\mathcal I_{\mathcal X}(x)|}
\right\}.
\label{eq:app_polyhedral_normal}
\end{equation}
Thus,
\(B_{\mathcal X}(x)=C_{\mathcal I_{\mathcal X}(x)}^\top\),
\(\mathcal K_{\mathcal X}(x)=
\R_+^{|\mathcal I_{\mathcal X}(x)|}\), and
\(q_{\mathcal X}(x)=|\mathcal I_{\mathcal X}(x)|\).

For the Euclidean ball
\(
\mathbb B_r:=\{x:\|x\|_2\leq r\}
\)
with \(r>0\),
\begin{equation}
N_{\mathbb B_r}(x)
=
\begin{cases}
\{0\}, & \|x\|_2<r,\\
\{\alpha x:\alpha\geq0\}, & \|x\|_2=r.
\end{cases}
\label{eq:app_ball_normal}
\end{equation}
Hence a boundary point contributes one nonnegative scalar multiplier.

For the Lorentz cone
\(
\mathcal Q^{d}
=
\{(t,u)\in\R\times\R^{d-1}:\|u\|_2\leq t\},
\)
\begin{equation}
N_{\mathcal Q^{d}}(t,u)
=
\begin{cases}
\{0\},
& \|u\|_2<t,\\
\left\{
\alpha
\left(
-1,\,
u/\|u\|_2
\right):
\alpha\geq0
\right\},
& t=\|u\|_2>0,\\
-\mathcal Q^{d},
& (t,u)=0.
\end{cases}
\label{eq:app_lorentz_normal}
\end{equation}
Thus, a nonzero Lorentz-cone boundary point contributes one nonnegative
scalar, while the apex contributes one full cone block.

\begin{proof}[Proof of Proposition \ref{prop:projection_size}]
By Proposition \ref{prop:blockwise_optimal_set}, each quadratic multiplier is uniquely
determined by
\[
p_h=\alpha_hq_h(x).
\]
Likewise, if \(u_g(x)\neq0\), then
\[
r_g
=
\lambda_g
\frac{u_g(x)}{\|u_g(x)\|_2}
\]
is uniquely determined. If \(u_g(x)=0\) and \(\lambda_g=0\), then its
dual ball is the singleton \(\{0\}\). Only indices
\[
g\in\mathcal Z_{\mathrm N}(x)
\]
leave a free vector
\[
r_g\in\R^{k_g},
~~
\|r_g\|_2\leq\lambda_g.
\]
These blocks contribute
\[
\sum_{g\in\mathcal Z_{\mathrm N}(x)}k_g
\]
scalar variables.

For the ReLU block, the elimination proceeds backward through the network.
At layer \(L\),

- a negative preactivation fixes \(\nu_{L,i}=0\);
- a positive preactivation fixes \(\nu_{L,i}=c_i\);
- a zero preactivation leaves one scalar variable in \([0,c_i]\).

Suppose all nonboundary variables in layer \(\ell+1\) have already been
eliminated as affine functions of the surviving zero-preactivation variables.
At layer \(\ell\),

- if \(a_{\ell,i}(x)<0\), then \(\nu_{\ell,i}=0\);
- if \(a_{\ell,i}(x)>0\), then
  \[
  \nu_{\ell,i}
  =
  \left(
  U_{\ell+1}^\top\nu_{\ell+1}
  \right)_i,
  \]
  so it can be eliminated as an affine function of later-layer free
  variables;
- if \(a_{\ell,i}(x)=0\), then one scalar variable remains, satisfying
  \[
  0
  \leq
  \nu_{\ell,i}
  \leq
  \left(
  U_{\ell+1}^\top\nu_{\ell+1}
  \right)_i.
  \]

Thus, the ReLU block leaves at most one scalar variable for every
coordinate in \(\mathcal Z_{\mathrm R}(x)\). Multipliers associated
with strictly positive preactivations are affine functions of
later-layer free variables and introduce no additional decision
variables.

By \eqref{eq:app_general_normal_representation}, the feasible-set normal
component is parameterized by
\(
\zeta\in\mathcal K_{\mathcal X}(x)
\subseteq\R^{q_{\mathcal X}(x)},
\)
and therefore contributes at most \(q_{\mathcal X}(x)\) scalar variables.
After all substitutions, the projected vector is
\[
\mathcal G(\xi)+\nabla\Phi(x)+B_{\mathcal X}(x)\zeta,
\]
which is affine in the surviving variables. Its squared \(M^{-1}\)-norm is
therefore convex quadratic, and the feasible set consists of recursive affine
intervals, Euclidean balls, and the cone
\(\mathcal K_{\mathcal X}(x)\).

The total number of scalar variables is at most
\[
|\mathcal Z_{\mathrm R}(x)|
+
\sum_{g\in\mathcal Z_{\mathrm N}(x)}k_g
+
q_{\mathcal X}(x).
\]
If
\(
\mathcal Z_{\mathrm R}(x)=
\mathcal Z_{\mathrm N}(x)=\varnothing
\)
and \(q_{\mathcal X}(x)=0\), then every network multiplier is uniquely
determined, the normal cone is \(\{0\}\), and no projection optimization is
required.
\end{proof}

\subsection{Finite line search and proximal safeguard}
\label{sec:app_safeguard_properties}

We first record why Euclidean projection is an admissible feasible retraction.

\begin{lemma}[First-order consistency of Euclidean projection]
\label{lem:app_projection_retraction}
Let \(\mathcal X\) be nonempty, closed, and convex,
\(x\in\mathcal X\), and \(d\in T_{\mathcal X}(x)\). Then
\begin{equation}
\Pi_{\mathcal X}(x+td)
=
x+td+o(t)
~~
\text{as }t\downarrow0.
\label{eq:app_projection_first_order}
\end{equation}
Moreover,
\(\Pi_{\mathcal X}(y)=y\) for every \(y\in\mathcal X\).
\end{lemma}

\begin{proof}
For a closed convex set, the tangent-cone characterization gives
\[
d\in T_{\mathcal X}(x)
~\Longleftrightarrow~
\operatorname{dist}(x+td,\mathcal X)=o(t)
~
\text{as }t\downarrow0.
\]
Since
\[
\left\|
\Pi_{\mathcal X}(x+td)-(x+td)
\right\|_2
=
\operatorname{dist}(x+td,\mathcal X),
\]
\eqref{eq:app_projection_first_order} follows. The second statement is the
defining fixed-point property of the Euclidean projection.
\end{proof}

\begin{lemma}[Existence of a projected directional Armijo step]
\label{lem:app_armijo_exists}
Let \(x\in\mathcal X\) satisfy \(\mathcal R_M(x)>0\), let
\[
d=-M^{-1}g_M^\dagger(x),
\]
and define
\[
z(t)
=
\mathsf P_{\mathcal X}(x,td),
~~
s(t)
=
z(t)-x.
\]
For every \(\eta,\sigma\in(0,1)\), there exists \(\bar t>0\) such that
\begin{equation}
F'(x;s(t))
\leq
-\eta t\mathcal R_M(x)^2
\label{eq:app_projected_slope_test}
\end{equation}
and
\begin{equation}
F(z(t))
\leq
F(x)
+
\sigma F'(x;s(t))
\label{eq:app_projected_armijo}
\end{equation}
for every \(t\in(0,\bar t]\).
\end{lemma}

\begin{proof}
By \cref{thm:certificate},
\[
d\in T_{\mathcal X}(x)
\]
and
\[
F'(x;d)
\leq
-\mathcal R_M(x)^2
<0.
\]
The first-order consistency of the retraction gives
\[
\frac{s(t)}{t}\longrightarrow d.
\]
Because \(F\) is finite and convex on \(\R^d\),
\[
F'(x;v)
=
\max_{g\in\partial F(x)}g^\top v
\]
is the support function of the nonempty compact set \(\partial F(x)\), and
is therefore finite, positively homogeneous, and continuous in \(v\).
Consequently,
\[
\frac{F'(x;s(t))}{t}
=
F'\left(x;\frac{s(t)}{t}\right)
\longrightarrow
F'(x;d)
\leq
-\mathcal R_M(x)^2.
\]
Since \(\eta<1\), \eqref{eq:app_projected_slope_test} holds for all
sufficiently small \(t>0\).

A finite convex function is Hadamard directionally differentiable. Hence,
whenever \(s(t)/t\to d\),
\[
\frac{F(x+s(t))-F(x)}{t}
\longrightarrow
F'(x;d).
\]
Together with the preceding limit,
\[
\frac{
F(x+s(t))-F(x)
-
\sigma F'(x;s(t))
}{t}
\longrightarrow
(1-\sigma)F'(x;d)
<0.
\]
Thus \eqref{eq:app_projected_armijo} also holds for all sufficiently small
\(t>0\). Taking the smaller of the two resulting thresholds proves the
claim.
\end{proof}

Thus, an unlimited projected backtracking search would always terminate.
DCI uses a finite backtracking budget to bound the local search cost and
invokes the proximal safeguard if the tested sequence does not reach a
sufficiently small step.

\begin{lemma}[Proximal safeguard]
\label{lem:app_proximal_safeguard}
For \(x_k\in\mathcal X\), \(M_k\succ0\), and \(\gamma>0\), the problem
\begin{equation}
x_{k+1}
=
\argmin_{z\in\mathcal X}
\left\{
F(z)
+
\frac{1}{2\gamma}
\|z-x_k\|_{M_k}^2
\right\}
\label{eq:app_proximal_problem}
\end{equation}
has a unique solution. This solution satisfies
\begin{equation}
\frac{1}{\gamma}
M_k(x_k-x_{k+1})
\in
\partial F(x_{k+1})
+
N_{\mathcal X}(x_{k+1}),
\label{eq:app_proximal_residual}
\end{equation}
and
\begin{equation}
F(x_k)-F(x_{k+1})
\geq
\frac{1}{2\gamma}
\|x_{k+1}-x_k\|_{M_k}^2.
\label{eq:app_proximal_decrease}
\end{equation}
In particular,
\begin{equation}
\mathcal R_{M_k}(x_{k+1})
\leq
\frac{1}{\gamma}\|x_{k+1}-x_k\|_{M_k}.
\label{eq:app_proximal_residual_bound}
\end{equation}
\end{lemma}

\begin{proof}
The proximal objective is strongly convex because \(M_k\succ0\). It is also
coercive on \(\mathcal X\). Indeed, since \(F\) is finite and convex, choose
\(g_k\in\partial F(x_k)\). Then
\[
F(z)
\geq
F(x_k)+g_k^\top(z-x_k),
\]
so the proximal objective is bounded below by an affine function plus a
positive-definite quadratic. It therefore tends to \(+\infty\) as
\(\|z\|_2\to\infty\).

Since \(\mathcal X\) is nonempty and closed, a minimizer exists. Strong
convexity makes it unique.

The first-order optimality condition is
\[
0
\in
\partial F(x_{k+1})
+
N_{\mathcal X}(x_{k+1})
+
\frac{1}{\gamma}
M_k(x_{k+1}-x_k),
\]
which is equivalent to \eqref{eq:app_proximal_residual}.

Because \(x_k\in\mathcal X\) is a feasible comparison point,
\[
F(x_{k+1})
+
\frac{1}{2\gamma}
\|x_{k+1}-x_k\|_{M_k}^2
\leq
F(x_k).
\]
Rearranging gives \eqref{eq:app_proximal_decrease}.

By \eqref{eq:app_proximal_residual}, the vector
\(\gamma^{-1}M_k(x_k-x_{k+1})\) belongs to the constrained
first-order set at \(x_{k+1}\). Therefore,
\[
\mathcal R_{M_k}(x_{k+1})
\leq
\left\|
\gamma^{-1}M_k(x_k-x_{k+1})
\right\|_{M_k^{-1}}
=
\gamma^{-1}\|x_{k+1}-x_k\|_{M_k},
\]
which proves \eqref{eq:app_proximal_residual_bound}.
\end{proof}

\subsection{Complete safeguarded DCI iteration}
\label{sec:app_complete_dci}

For the global theory, DCI only requires a sequence of metrics satisfying
\eqref{eq:metric_bounds}. On a structurally nondegenerate iterate where
\(\Phi\) is twice differentiable, the explicit network curvature gives
\[
H_k
:=
\nabla^2 f(x_k)+\nabla^2\Phi(x_k)
=
\nabla^2F(x_k).
\]
A curvature-based choice is
\[
M_k
=
\operatorname{clip}_{[m,\overline m]}
\bigl(H_k+\tau_k I\bigr),
~~
\tau_k\geq0,
\]
where \(\operatorname{clip}_{[m,\overline m]}\) denotes spectral clipping of
the eigenvalues to \([m,\overline m]\). When this curvature is unavailable,
any positive-definite fallback satisfying \eqref{eq:metric_bounds} may be used.
The complete safeguarded iteration is summarized below.

\begin{algorithm}[h]
\caption{Safeguarded Dual-Certified Inference}
\label{alg:app_dci}
\begin{algorithmic}[1]
\REQUIRE
\(x_0\in\mathcal X\), tolerance \(\varepsilon\geq0\),
\(0<\sigma<1/2\), \(0<\eta<1\), \(0<\beta<1\),
\(J_{\max}\geq0\), and \(\gamma>0\).
\FOR{\(k=0,1,\ldots\)}
    \STATE Evaluate the SOC-ICNN at \(x_k\) and construct
    \(\mathcal D^\star(x_k)\) from the forward-pass states.
    \STATE Construct a representation of \(N_{\mathcal X}(x_k)\) and choose
    \(M_k\succ0\) satisfying \eqref{eq:metric_bounds}.
    \STATE Solve \eqref{eq:constrained_projection} to obtain
    \(g_k^\dagger=g_{M_k}^\dagger(x_k)\) and
    \(\mathcal R_k=\|g_k^\dagger\|_{M_k^{-1}}\).
    \IF{\(\mathcal R_k\leq\varepsilon\)}
        \STATE \textbf{return} \(x_k\).
    \ENDIF
    \STATE Set \(d_k=-M_k^{-1}g_k^\dagger\).
    \STATE Set \(\mathrm{accepted}\leftarrow\mathrm{false}\).
    \FOR{\(j=0,\ldots,J_{\max}\)}
        \STATE Set \(t=\beta^j\),
        \(z_k(t)=\mathsf P_{\mathcal X}(x_k,td_k)\), and
        \(s_k(t)=z_k(t)-x_k\).
        \IF{both conditions in \eqref{eq:armijo} hold}
            \STATE Set \(x_{k+1}\leftarrow z_k(t)\) and
            \(\mathrm{accepted}\leftarrow\mathrm{true}\).
            \STATE \textbf{break}.
        \ENDIF
    \ENDFOR
    \IF{\(\mathrm{accepted}=\mathrm{false}\)}
        \STATE Compute \(x_{k+1}\) by the metric-proximal safeguard
        \eqref{eq:proximal}.
    \ENDIF
\ENDFOR
\end{algorithmic}
\end{algorithm}

\subsection{Closedness of constrained stationarity}
\label{sec:app_constrained_stability}

The geometry appendix establishes the closed-graph property of
\(x\mapsto\partial F(x)\). The constrained convergence proof additionally
requires stability after adding the normal cone of a general closed convex
feasible set.

\begin{lemma}[Closed graph of the constrained first-order mapping]
\label{lem:app_constrained_closed_graph}
Let \(x_j\to x\) with \(x_j,x\in\mathcal X\), and let
\[
c_j
\in
\partial F(x_j)+N_{\mathcal X}(x_j),
~~
c_j\to c.
\]
Then
\[
c
\in
\partial F(x)+N_{\mathcal X}(x).
\]
\end{lemma}

\begin{proof}
Choose decompositions
\[
c_j=g_j+n_j,
~~
g_j\in\partial F(x_j),
~~
n_j\in N_{\mathcal X}(x_j).
\]

Because \(F\) is finite and convex on \(\R^d\), its subdifferential is locally
bounded. Since \(x_j\to x\), the sequence \(\{g_j\}\) is bounded. Passing to
a subsequence if necessary, let
\[
g_j\to g.
\]
The closed-graph property of the convex subdifferential gives
\[
g\in\partial F(x).
\]

Moreover,
\[
n_j=c_j-g_j\to c-g=:n.
\]
For every \(z\in\mathcal X\),
\[
n_j^\top(z-x_j)\leq0.
\]
Passing to the limit gives
\[
n^\top(z-x)\leq0
~~
\forall~z\in\mathcal X.
\]
Hence
\[
n\in N_{\mathcal X}(x).
\]
Therefore
\[
c=g+n
\in
\partial F(x)+N_{\mathcal X}(x).
\]
Since every convergent subsequence yields the same conclusion, the entire
mapping has a closed graph.
\end{proof}

\subsection{Global convergence}
\label{sec:app_global_convergence}

\begin{proof}[Proof of the global part of \cref{thm:convergence}]
If the algorithm terminates finitely with \(\varepsilon=0\), then
\(\mathcal R_k=0\), and \cref{thm:certificate} proves global optimality.
We therefore consider an infinite run. Let
\[
\mathcal A
=
\left\{
k:
\text{iteration \(k\) accepts a projected local step}
\right\},
\]
and
\[
\mathcal S
=
\left\{
k:
\text{iteration \(k\) invokes the proximal safeguard}
\right\}.
\]

\paragraph{Monotonicity and boundedness.}
If \(k\in\mathcal A\), let \(t_k\) be the accepted step and
\(s_k=s_k(t_k)\). The two line-search conditions give
\begin{align*}
F(x_{k+1})
&\leq
F(x_k)+\sigma F'(x_k;s_k)\\
&\leq
F(x_k)-\sigma\eta t_k\mathcal R_k^2
\leq
F(x_k).
\end{align*}
If \(k\in\mathcal S\), Lemma \ref{lem:app_proximal_safeguard} gives
\[
F(x_{k+1})\leq F(x_k).
\]
Thus
\[
F(x_{k+1})\leq F(x_k)
\]
for every \(k\), and all iterates remain in the compact level set
\(\mathcal L_0\).

Since \(F\) is continuous and \(\mathcal L_0\) is compact, \(F\) is bounded
below on \(\mathcal L_0\). Therefore the nonincreasing sequence
\(\{F(x_k)\}\) converges to some value \(\bar F\).

We consider two cases.

\paragraph{Case 1: infinitely many accepted local steps.}
For every \(k\in\mathcal A\),
\[
t_k\geq\beta^{J_{\max}},
\]
and therefore
\[
F(x_k)-F(x_{k+1})
\geq
\sigma\eta\beta^{J_{\max}}\mathcal R_k^2.
\]
Summing over accepted iterations gives
\[
\sigma\eta\beta^{J_{\max}}
\sum_{k\in\mathcal A}
\mathcal R_k^2
\leq
F(x_0)-\inf_{\mathcal L_0}F
<\infty.
\]
Hence
\[
\mathcal R_k\to0
~~
\text{along }k\in\mathcal A.
\]

Let
\[
p_k=g_{M_k}^\dagger(x_k).
\]
The metric bound
\(M_k\preceq\overline m I\) implies
\[
\mathcal R_k^2
=
p_k^\top M_k^{-1}p_k
\geq
\frac{1}{\overline m}\|p_k\|_2^2.
\]
Thus
\[
p_k\to0
~~
\text{along }k\in\mathcal A.
\]

Compactness of \(\mathcal L_0\) gives a subsequence
\[
x_{k_j}\to\bar x,
~~
k_j\in\mathcal A.
\]
Since
\[
p_{k_j}
\in
\partial F(x_{k_j})+N_{\mathcal X}(x_{k_j})
\]
and \(p_{k_j}\to0\), Lemma \ref{lem:app_constrained_closed_graph} gives
\[
0
\in
\partial F(\bar x)+N_{\mathcal X}(\bar x).
\]
Therefore \(\bar x\) is a global minimizer of \(F\) over \(\mathcal X\).
Continuity gives
\[
\bar F
=
\lim_{j\to\infty}F(x_{k_j})
=
F(\bar x)
=
F^\star.
\]

\paragraph{Case 2: only finitely many accepted local steps.}
In this case, every sufficiently large iteration invokes the proximal
safeguard. From \eqref{eq:app_proximal_decrease},
\[
F(x_k)-F(x_{k+1})
\geq
\frac{1}{2\gamma}
\|x_{k+1}-x_k\|_{M_k}^2.
\]
Summing over all sufficiently large \(k\) yields
\[
\sum_k
\|x_{k+1}-x_k\|_{M_k}^2
<\infty.
\]
Using \(M_k\succeq mI\),
\[
\|x_{k+1}-x_k\|_2\to0.
\]

Define the proximal residual
\[
r_{k+1}
=
\frac{1}{\gamma}
M_k(x_k-x_{k+1}).
\]
By \eqref{eq:app_proximal_residual},
\[
r_{k+1}
\in
\partial F(x_{k+1})+N_{\mathcal X}(x_{k+1}).
\]
Moreover,
\[
\|r_{k+1}\|_2
\leq
\frac{\overline m}{\gamma}
\|x_{k+1}-x_k\|_2
\to0.
\]

Compactness gives a subsequence
\[
x_{k_j+1}\to\bar x.
\]
Applying Lemma \ref{lem:app_constrained_closed_graph} yields
\[
0
\in
\partial F(\bar x)+N_{\mathcal X}(\bar x).
\]
Thus \(\bar x\) is globally optimal and, as in the first case,
\[
\bar F=F^\star.
\]

\paragraph{All accumulation points are optimal.}
In both cases, the full objective sequence satisfies
\[
F(x_k)\to F^\star.
\]
Let \(\widetilde x\) be any accumulation point, with
\(x_{k_j}\to\widetilde x\). Continuity gives
\[
F(\widetilde x)
=
\lim_{j\to\infty}F(x_{k_j})
=
F^\star.
\]
Hence every accumulation point is a global solution of
\(\min_{x\in\mathcal X}F(x)\).
\end{proof}

\begin{corollary}[Convergence under uniqueness]
\label{cor:app_unique_global_convergence}
Under the assumptions of the global part of \cref{thm:convergence}, if the
constrained minimizer \(x^\star\) is unique, then
\[
x_k\to x^\star.
\]
\end{corollary}

\begin{proof}
The sequence lies in the compact set \(\mathcal L_0\), and every accumulation
point is a global minimizer. Uniqueness implies that \(x^\star\) is the only
possible accumulation point. A bounded sequence with a unique accumulation
point converges to that point.
\end{proof}

\subsection{Interior Newton regime and quadratic convergence}
\label{sec:app_local_quadratic}

\begin{proof}[Proof of the local part of \cref{thm:convergence}]
By Corollary \ref{cor:app_unique_global_convergence},
\[
x_k\to x^\star.
\]
Because \(x^\star\in\operatorname{int}\mathcal X\), there exists a convex
neighborhood \(\mathcal N_0\) of \(x^\star\) whose closure is contained in
\(\operatorname{int}\mathcal X\). Hence, for all sufficiently large \(k\),
\[
N_{\mathcal X}(x_k)=\{0\}.
\]

Structural nondegeneracy provides a neighborhood \(\mathcal N_1\) in which
the ReLU activation pattern is fixed, all positive-weight norm residuals are
nonzero, and the optimal-dual set \(\mathcal D^\star(x)\) is a singleton. On
\[
\mathcal N
=
\mathcal N_0\cap\mathcal N_1,
\]
the complete readout therefore equals the ordinary gradient:
\[
g_{M_k}^\dagger(x_k)=\nabla F(x_k).
\]

Shrink \(\mathcal N\), if necessary, to a convex ball contained
in the neighborhoods where the Hessian metric is used and the
retraction is exact, and choose constants \(\mu_H,L_H>0\)
satisfying the following bounds.
\begin{equation}
\nabla^2F(x)\succeq\mu_HI
~~
\forall~x\in\mathcal N,
\label{eq:app_local_hessian_lower}
\end{equation}
and
\begin{equation}
\|\nabla^2F(x)-\nabla^2F(z)\|_2
\leq
L_H\|x-z\|_2
~~
\forall~x,z\in\mathcal N.
\label{eq:app_local_hessian_lipschitz}
\end{equation}
Such constants exist by positive definiteness at \(x^\star\), continuity, and
local Lipschitz continuity of the Hessian.

For all sufficiently large \(k\),
\[
M_k
=
H_k
:=
\nabla^2F(x_k),
\]
and the DCI direction is exactly the Newton direction
\begin{equation}
d_k
=
-H_k^{-1}\nabla F(x_k).
\label{eq:app_newton_direction}
\end{equation}
The residual satisfies
\begin{equation}
\mathcal R_k^2
=
\nabla F(x_k)^\top H_k^{-1}\nabla F(x_k)
=
d_k^\top H_kd_k
=
-\nabla F(x_k)^\top d_k.
\label{eq:app_newton_residual}
\end{equation}

\paragraph{Eventual acceptance of the unit step.}
Continuity of \(\nabla F\), the identity \(\nabla F(x^\star)=0\), and the
uniform inverse bound from \eqref{eq:app_local_hessian_lower} imply
\(d_k=-H_k^{-1}\nabla F(x_k)\to0\). Hence both \(x_k\) and \(x_k+d_k\) eventually lie in \(\mathcal N\).
Its convexity ensures that the segment joining them also lies in
\(\mathcal N\). Local exactness of the retraction therefore gives
\[
z_k(1)=x_k+d_k,
~~
s_k(1)=d_k.
\]
Because \(F\) is differentiable in \(\mathcal N\),
\[
F'(x_k;s_k(1))
=
\nabla F(x_k)^\top d_k
=
-\mathcal R_k^2
\leq
-\eta\mathcal R_k^2,
\]
so the slope test holds for every \(\eta\in(0,1)\).

The Lipschitz-Hessian Taylor estimate gives
\[
F(x_k+d_k)
\leq
F(x_k)
+
\nabla F(x_k)^\top d_k
+
\frac{1}{2}d_k^\top H_kd_k
+
\frac{L_H}{6}\|d_k\|_2^3.
\]
Using \eqref{eq:app_newton_residual},
\[
F(x_k+d_k)
\leq
F(x_k)
-
\frac{1}{2}\mathcal R_k^2
+
\frac{L_H}{6}\|d_k\|_2^3.
\]
From \eqref{eq:app_local_hessian_lower},
\[
\mathcal R_k^2
=
d_k^\top H_kd_k
\geq
\mu_H\|d_k\|_2^2.
\]
Therefore,
\[
\frac{L_H}{6}\|d_k\|_2^3
\leq
\frac{L_H}{6\mu_H}
\|d_k\|_2\mathcal R_k^2.
\]
Since \(d_k\to0\) and \(\sigma<1/2\), for all sufficiently large \(k\),
\[
\frac{L_H}{6\mu_H}\|d_k\|_2
\leq
\frac{1}{2}-\sigma.
\]
Hence
\[
F(x_k+d_k)
\leq
F(x_k)-\sigma\mathcal R_k^2
=
F(x_k)+\sigma F'(x_k;s_k(1)).
\]
Both line-search conditions are therefore satisfied at \(t=1\), so unit
steps are eventually accepted.

\paragraph{Quadratic error estimate.}
Let
\[
e_k=x_k-x^\star.
\]
Because \(\nabla F(x^\star)=0\),
\[
\nabla F(x_k)
=
\int_0^1
\nabla^2F(x^\star+t e_k)e_k\,dt.
\]
Using the Newton update,
\begin{align*}
e_{k+1}
&=
e_k-H_k^{-1}\nabla F(x_k)\\
&=
H_k^{-1}
\int_0^1
\left[
H_k-\nabla^2F(x^\star+t e_k)
\right]
e_k\,dt.
\end{align*}
By \eqref{eq:app_local_hessian_lower},
\[
\|H_k^{-1}\|_2\leq\frac{1}{\mu_H}.
\]
By \eqref{eq:app_local_hessian_lipschitz},
\[
\left\|
H_k-\nabla^2F(x^\star+t e_k)
\right\|_2
\leq
L_H(1-t)\|e_k\|_2.
\]
Therefore,
\begin{align*}
\|e_{k+1}\|_2
&\leq
\frac{1}{\mu_H}
\int_0^1
L_H(1-t)\|e_k\|_2^2\,dt\\
&=
\frac{L_H}{2\mu_H}
\|e_k\|_2^2.
\end{align*}
Thus
\[
\|x_{k+1}-x^\star\|_2
\leq
C\|x_k-x^\star\|_2^2
\]
with
\[
C=\frac{L_H}{2\mu_H},
\]
which proves \(\|x_{k+1}-x^\star\|_2\leq C\|x_k-x^\star\|_2^2\).
\end{proof}

\subsection{Unconstrained specialization}
\label{sec:app_unconstrained_specialization}

When
\[
\mathcal X=\R^d,
\]
the normal cone is
\[
N_{\mathcal X}(x)=\{0\}
\]
at every input. The constrained projection reduces to
\[
\xi_M^\dagger(x)
\in
\argmin_{\xi\in\mathcal D^\star(x)}
\frac{1}{2}
\left\|
\mathcal G(\xi)+\nabla\Phi(x)
\right\|_{M^{-1}}^2,
\]
and
\[
g_M^\dagger(x)
=
\mathcal G\bigl(\xi_M^\dagger(x)\bigr)
+\nabla\Phi(x).
\]
The certificate becomes
\[
\mathcal R_M(x)=0
~\Longleftrightarrow~
0\in\partial F(x)
~\Longleftrightarrow~
x\in\argmin F,
\]
and the common-descent inequality becomes
\[
g^\top d_M(x)
\leq
-\mathcal R_M(x)^2
~~
\forall~g\in\partial F(x).
\]
Thus the original unconstrained DCI is recovered exactly as a special case of
the constrained framework.

\section{Additional Experimental Details and Results}
\label{app:experiments}

This appendix gives the complete constructions, evaluation criteria, and
numerical results for the experiments in \cref{sec:experiments}.
Experiments~1--4 use controlled objectives with analytically known geometry.
Experiment~5 trains context conditioned SOC-ICNNs and evaluates inference on
their frozen objectives, separating optimization error from approximation
error of the learned objective. The convergence theorem concerns the exact algorithm. The experiments use numerical implementations with the activity tolerances, approximate proximal solves, and line-search variants specified below.

\subsection{Common Model and Notation}
\label{app:notation}

\paragraph{Generic inference problem.}
The paper studies
\begin{equation}
x^\star(y)\in\arg\min_{x\in\mathcal X}F_y(x),
~~
F_y(x)=f_\theta(x;y)+\Phi(x;y),
\label{eq:app_generic_problem}
\end{equation}
where \(x\in\mathbb R^{d_x}\) is the decision variable optimized at
inference time, \(y\) is an observed context held fixed during that
optimization, \(\mathcal X\) is a closed convex feasible set,
\(f_\theta(\cdot;y)\) is the learned convex network objective, and
\(\Phi\) denotes any additional explicit convex term. Experiments~1--4
suppress \(y\) and construct \(F\) directly. Experiment~5 uses a daily context
and denotes the battery decision by \(z\).

\paragraph{Full SOC-ICNN objective.}
To avoid overloading \(z\), this appendix denotes the hidden activation at
layer \(\ell\) by \(h_\ell\). For fixed context \(y\),
\begin{align}
a_1(x)&=W_1x+b_1,
&
h_1(x)&=[a_1(x)]_+,
\nonumber\\
a_\ell(x)&=W_\ell x+U_\ell h_{\ell-1}(x)+b_\ell,
&
h_\ell(x)&=[a_\ell(x)]_+,
~ \ell=2,\ldots,L .
\label{eq:app_soc_backbone}
\end{align}
The ICNN sign restrictions are \(U_\ell\ge0\) elementwise and \(c\ge0\).
The quadratic and norm residuals are
\[
q_h(x)=B_hx+e_h,~~
u_g(x)=A_gx+d_g,
\]
and the scalar objective is
\begin{equation}
f(x)=
c^\top h_L(x)+v^\top x+b_0
+\sum_{h=1}^{H}\frac{\alpha_h}{2}\|q_h(x)\|_2^2
+\sum_{g=1}^{G}\lambda_g\|u_g(x)\|_2 ,
\label{eq:app_soc_model}
\end{equation}
with \(\alpha_h>0\) and \(\lambda_g\ge0\).
The network therefore does \emph{not} directly output a decision. It outputs
a scalar convex objective as a function of the decision, and inference then
minimizes that objective.

\begin{table}[h]
\centering
\caption{\textbf{Model symbols used throughout the appendix.}}
\label{tab:app_symbol_model}
\small
\setlength{\tabcolsep}{4pt}
\begin{tabular}{p{0.18\linewidth}p{0.72\linewidth}}
\toprule
Symbol & Meaning\\
\midrule
\(d_x\) & Decision dimension. This notation is used to distinguish dimension from the norm offset \(d_g\) and from day indices in Experiment~5.\\
\(L,m\) & ReLU depth and hidden width.\\
\(W_\ell\) & Decision-to-layer matrix, \(W_\ell\in\mathbb R^{m\times d_x}\).\\
\(U_\ell\) & Nonnegative hidden-to-hidden matrix, \(U_\ell\in\mathbb R_+^{m\times m}\).\\
\(a_\ell,h_\ell\) & ReLU preactivation and activation in \eqref{eq:app_soc_backbone}.\\
\(c,v,b_0\) & Final hidden weight, direct linear decision coefficient, and scalar bias.\\
\(H\) & Number of quadratic modules.\\
\(B_h,e_h,\alpha_h\) & Matrix, offset, and coefficient of \(q_h=B_hx+e_h\).\\
\(G\) & Number of Euclidean-norm modules.\\
\(A_g,d_g,\lambda_g\) & Matrix, offset, and coefficient of \(u_g=A_gx+d_g\).\\
\(k_g\) & Dimension of \(u_g\), so \(A_g\in\mathbb R^{k_g\times d_x}\).\\
\bottomrule
\end{tabular}
\end{table}

\paragraph{Set-valued geometry and DCI symbols.}
A ReLU coordinate is structurally ambiguous when \(a_{\ell,j}(x)=0\);
a norm block is ambiguous when \(u_g(x)=0\). At such points
\(\partial F(x)\) is set-valued. For constrained problems, the complete
first-order object is
\(
\partial F(x)+N_{\mathcal X}(x),
\)
where \(N_{\mathcal X}(x)\) is the convex normal cone.

For \(M\succ0\), DCI computes
\begin{equation}
(s_M^\dagger,n_M^\dagger)
\in
\arg\min_{\substack{s\in\partial F(x)\\ n\in N_{\mathcal X}(x)}}
\frac12(s+n)^\top M^{-1}(s+n).
\label{eq:app_common_readout}
\end{equation}
We use
\[
g_M^\dagger=s_M^\dagger+n_M^\dagger,~~
\mathcal R_M(x)=\|g_M^\dagger\|_{M^{-1}},~~
d_M=-M^{-1}g_M^\dagger.
\]
DCI-I sets \(M=I\). DCI-C uses the explicit smooth-branch curvature metric.
All final certificates are recomputed using the Euclidean residual
\begin{equation}
\mathcal R_{\mathcal X}(x)
=
\operatorname{dist}\!\left(
0,\partial F(x)+N_{\mathcal X}(x)
\right).
\label{eq:app_euclidean_certificate}
\end{equation}

\begin{table}[h]
\centering
\caption{\textbf{Experimental map.} The controlled experiments are not five different learned models; they isolate different pieces of the same inference geometry.}
\label{tab:app_experiment_map}
\small
\setlength{\tabcolsep}{4pt}
\begin{tabular}{p{0.08\linewidth}p{0.39\linewidth}p{0.43\linewidth}}
\toprule
Exp. & Object & Main question\\
\midrule
1 & Frozen random deep SOC-ICNN \(f(x)\). & Are the analytic gradient, Hessian, and boundary readout correct?\\
2 & Strongly convex controlled ReLU--norm objective with known \(x^\star=0\). & How does the recovered subdifferential affect inference at
structural boundaries?\\
3 & Smooth controlled objectives with prescribed Hessian condition number. & Once first-order geometry is exact, what does explicit curvature add?\\
4 & Controlled strongly convex objectives over free/box/ball/simplex/SOC sets. & Why are the normal cone, curvature metric, and safeguard separately needed?\\
5 & Context-conditioned learned battery-cost surface. & Does DCI remain reliable and efficient on objectives produced by trained SOC-ICNNs?\\
\bottomrule
\end{tabular}
\end{table}

\paragraph{Computational environment and timing.}
All experiments use Python~3.11, PyTorch~2.0, CUDA~12, and an NVIDIA
GeForce RTX~4060~Ti. Controlled geometric calculations use float64. Reported
times include the network evaluation and the optimization operations performed
by each method; certificate recomputation is excluded unless explicitly
stated. The pointwise oracle benchmark in Experiment~2 separately reports its
warm-up and repetition protocol.

\subsection{Experiment 1: Validation of the Exact SOC-ICNN Geometry}
\label{app:exp1}

\paragraph{What this experiment does.}
This experiment uses the full deep model in
\eqref{eq:app_soc_backbone}. We freeze random parameters and evaluate the same scalar
function \(f:\mathbb R^{d_x}\to\mathbb R\) in several independent ways.

A forward pass exposes
\[
\{a_\ell(x),h_\ell(x)\}_{\ell=1}^L,~~
\{q_h(x)\}_{h=1}^H,~~
\{u_g(x)\}_{g=1}^G.
\]
A point is structurally nondegenerate when
\[
\min_{\ell,j}|a_{\ell,j}(x)|>10^{-9},
~~
\min_g\|u_g(x)\|_2>10^{-9}.
\]
At such points the subdifferential is a singleton and the analytic readout must
equal the ordinary gradient. At the deliberately constructed mixed boundary,
the ReLU multiplier is an interval and the zero-norm multiplier is a Euclidean
ball, so the experiment checks the complete directional support function
rather than one branch-selected derivative.

\paragraph{Random SOC-ICNN construction.}
For decision dimension $d_x$, depth $L$, and width $m$, standard-normal base
matrices are scaled as
\[
W_\ell=0.65\widetilde W_\ell/\sqrt {d_x},~~
U_\ell=0.55|\widetilde U_\ell|/\sqrt m.
\]
Hidden biases are $0.35\widetilde b_\ell$, positive output weights are
$0.6|\widetilde c|/\sqrt m$, the direct linear coefficient is
$0.25\widetilde v$, and the output bias is $0.1\widetilde b_0$.
Quadratic modules use
\[
B_h=0.75\widetilde B_h/\sqrt {d_x},~ e_h=0.35\widetilde e_h,
~ \alpha_h\sim U[0.4,1.2],
\]
and norm modules use
\[
A_g=0.70\widetilde A_g/\sqrt {d_x},~ d_g=0.45\widetilde d_g,
~ \lambda_g\sim U[0.4,1.2].
\]
The nonnegativity of $U_\ell$ and of the output weights preserves input
convexity.

\paragraph{First-order, curvature, and Taylor tests.}
The gradient test uses 250 nondegenerate inputs for a depth-4, width-64 model
with $d_x=20$, two quadratic modules and two 20-dimensional norm modules. The
curvature test uses 100 inputs for depth 3, width 32, $d_x=10$, two quadratic and
two 10-dimensional norm modules. The explicit smooth-branch Hessian is
\[
H_{\rm exp}=\sum_h\alpha_hB_h^\top B_h+
\sum_g\lambda_gA_g^\top\frac{I-\widehat u_g\widehat u_g^\top}{\|u_g\|_2}A_g,
~~ \widehat u_g=u_g/\|u_g\|_2.
\]
For the Taylor test, 500 directions are sampled at each radius
$r\in\{10^{-4},3\times10^{-4},10^{-3}\}$ and we measure
\[
E_2(\delta)=|f(x+\delta)-f(x)-\nabla f(x)^\top\delta-
\tfrac12\delta^\top H_{\rm exp}(x)\delta|.
\]
The selected anchor has minimum ReLU margin $1.800\times10^{-2}$ and minimum
norm residual $2.197$; all 1500 perturbations remain on the same branch.
Gradient errors use the Euclidean norm, Hessian errors use the Frobenius norm,
and relative errors divide by the corresponding reference norm.

\paragraph{Mixed structural boundary.}
At $x_0=0\in\mathbb R^2$ we use one zero ReLU and one zero norm residual:
\[
\begin{aligned}
f(x)={}&v^\top x+b_0+0.7[(0.4,0.9)x+0.8]_+ +1.3[(1,-0.7)x]_+\\
&+\frac{0.9}{2}\left\|\begin{bmatrix}0.8&-0.2\\0.1&0.6\end{bmatrix}x+
\begin{bmatrix}0.25\\-0.35\end{bmatrix}\right\|_2^2
+1.1\left\|\begin{bmatrix}1&0.25\\-0.15&0.85\end{bmatrix}x\right\|_2,
\end{aligned}
\]
with $v=(0.2,-0.15)^\top$. We sample 1000 directions, use one-sided finite
differences with step $10^{-7}$, sample 5000 admissible dual branches, and
test 5000 global support inequalities.
The canonical gap is the difference between the complete directional
derivative and the derivative returned by the canonical branch; its positive
value records directions for which the single branch omits an active support.

\begin{table}[h]
\centering
\caption{\textbf{Validation of the exact SOC-ICNN geometry.} AD is an independent automatic-differentiation reference.}
\label{tab:app_exp1_geometry}
\small
\setlength{\tabcolsep}{4.5pt}
\begin{tabular}{llc}
\toprule
Component & Diagnostic & Result\\
\midrule
First order & Mean / max gradient error & $2.118\times10^{-15}$ / $3.731\times10^{-15}$\\
First order & Mean relative gradient error & $1.297\times10^{-16}$\\
First order & Mean readout time & 0.866 ms\\
Curvature & Mean / max Hessian error & $4.117\times10^{-16}$ / $6.046\times10^{-16}$\\
Curvature & Mean relative Hessian error & $1.305\times10^{-16}$\\
Curvature & Mean minimum eigenvalue & 0.205\\
Taylor & Empirical log--log slope & 3.006\\
Boundary & Mean / max finite-difference error & $2.374\times10^{-8}$ / $3.245\times10^{-8}$\\
Boundary & Max readout-maximization error & $8.882\times10^{-16}$\\
Boundary & Canonical-gap fraction & 1.000\\
Boundary & Minimum canonical gap & 0.942\\
Boundary & Maximum support violation & 0\\
\bottomrule
\end{tabular}
\end{table}

\begin{table}[h]
\centering
\caption{\textbf{Second-order Taylor residual on the fixed smooth branch.}}
\label{tab:app_exp1_taylor}
\small
\begin{tabular}{ccccc}
\toprule
Radius & Mean & Median & Maximum & Same branch\\
\midrule
$1.000\times10^{-4}$ & $8.825\times10^{-15}$ & $7.105\times10^{-15}$ & $3.730\times10^{-14}$ & 1.000 \\
$3.000\times10^{-4}$ & $2.374\times10^{-13}$ & $1.865\times10^{-13}$ & $1.089\times10^{-12}$ & 1.000 \\
$1.000\times10^{-3}$ & $8.946\times10^{-12}$ & $7.569\times10^{-12}$ & $3.128\times10^{-11}$ & 1.000 \\
\bottomrule
\end{tabular}
\end{table}

These checks validate the two distinct objects used by DCI: ordinary gradient
and curvature on smooth branches, and the complete support-function geometry at
structural boundaries.

\subsection{Experiment 2: Complete First-Order Geometry at Controlled Structural Boundaries}
\label{app:exp2}

Experiment~2 uses a controlled special case of the
ReLU--quadratic--norm family in \eqref{eq:app_soc_model}, chosen so
that the unique minimizer lies on prescribed structural boundaries.
Every run solves \(\min_{x\in\mathbb R^{d_x}}F(x)\), with
\(x^\star=0\) and its subdifferential known by construction.
This lets us compare a single canonical slope with the complete set
of supporting slopes without surrogate-model error.

\paragraph{Symbols specific to Experiment~2.}
The objective contains \(m_{\rm R}\) scalar ReLU terms and
\(m_{\rm N}\) norm modules. The matrix
\(W\in\mathbb R^{m_{\rm R}\times d_x}\) has row \(w_i^\top\);
\(b_i\) is the ReLU bias and \(\alpha_i>0\) its output coefficient.
Norm module \(g\) has matrix
\(A_g\in\mathbb R^{k_g\times d_x}\), offset
\(d_g\in\mathbb R^{k_g}\), and weight \(\lambda_g>0\).
The linear coefficient is \(\ell\in\mathbb R^{d_x}\).

At a point \(x\), define
\[
\mathcal Z_{\rm R}(x)=\{i:w_i^\top x+b_i=0\},
\qquad
\mathcal Z_{\rm N}(x)=\{g:A_gx+d_g=0\}.
\]
Only these blocks contribute free local multipliers. The number of
scalar variables in the boundary-local projection is
\begin{equation}
p(x)=|\mathcal Z_{\rm R}(x)|
+\sum_{g\in\mathcal Z_{\rm N}(x)}k_g,
\label{eq:app_exp2_p}
\end{equation}
whereas the full lifted support representation uses
\(p_{\rm lift}=m_{\rm R}+\sum_{g=1}^{m_{\rm N}}k_g\) scalar
multipliers. The scaling cases below report \(p=p(0)\),
\(p_{\rm lift}\), and \(d_x\).

At the known optimum, the complete subdifferential is
\begin{align}
\partial F(0)
={}&\ell
+\sum_{i:b_i>0}\alpha_iw_i
+\sum_{i\in\mathcal Z_{\rm R}(0)}[0,\alpha_i]w_i
\nonumber\\
&+\sum_{g:d_g\neq0}
\lambda_gA_g^\top\frac{d_g}{\|d_g\|_2}
+\sum_{g\in\mathcal Z_{\rm N}(0)}
A_g^\top\{r\in\mathbb R^{k_g}:\|r\|_2\leq\lambda_g\}.
\label{eq:app_exp2_subdiff_full}
\end{align}
The affine term \(\ell\) is chosen so that
\(0\in\partial F(0)\), while the canonical subgradient remains
nonzero. Strong convexity makes \(x^\star=0\) the unique minimizer.

\paragraph{Controlled objective family.}
We use the strongly convex objective
\begin{equation}
F(x)=\tfrac12x^\top Qx+\ell^\top x
+\sum_{i=1}^{m_{\rm R}}\alpha_i[w_i^\top x+b_i]_+
+\sum_{g=1}^{m_{\rm N}}\lambda_g\|A_gx+d_g\|_2,
\label{eq:app_exp2_objective}
\end{equation}
where \([s]_+=\max\{0,s\}\). The quadratic matrix is
\(Q=0.8I+0.35R_Q^\top R_Q/d_x\), where
\(R_Q\in\mathbb R^{d_x\times d_x}\) has independent
standard-normal entries and \(I\) is the identity matrix.
ReLU rows are normalized Gaussian directions, with
\(\alpha_i\sim U[0.5,1.5]\), where \(U[a,b]\) denotes the uniform
distribution on \([a,b]\). Boundary ReLU biases are zero; the
remaining biases have random signs and magnitudes in \([0.8,1.2]\).
Norm matrices have Gaussian entries scaled by \(1/\sqrt{d_x}\),
with \(\lambda_g\sim U[0.5,1.5]\), zero offsets for boundary
modules, and random nonzero offsets with norms in \([0.8,1.2]\).

For each zero ReLU preactivation, we select a multiplier in
\([0.2\alpha_i,0.8\alpha_i]\); for each zero norm residual, we
select a multiplier with norm in \([0.2\lambda_g,0.8\lambda_g]\).
The affine term \(\ell\) cancels the resulting sum of supporting
slopes. Instances are retained only when the canonical subgradient
at the optimum has Euclidean norm at least \(0.1\).

\paragraph{Small correctness suite and methods.}
We test \(d_x\in\{10,30\}\), ReLU/norm/mixed boundaries, and
boundary-block counts of two and four, with two generated instances
per setting. This gives 24 objectives, each evaluated from two
Gaussian initializations of scale \(1.5\), for 48 matched runs per
method. Each objective additionally contains six smooth ReLU terms
and two smooth norm modules; all norm residuals have dimension two.

Single-Branch uses the canonical subgradient with the same curvature
metric as DCI-C, without a proximal safeguard. Gradient Sampling
uses at most 40 sampled gradients, a sampling radius decreasing from
\(0.05\) to \(10^{-7}\), and 250 iterations. L-BFGS uses memory
20 and 400 iterations. Lifted-PD applies a primal--dual solver to the
full lifted support representation, with tolerance \(10^{-10}\)
and at most 10,000 iterations.

DCI-I uses the identity metric, while DCI-C uses the explicit
curvature with damping \(10^{-6}\) and eigenvalues clipped to
\([10^{-3},10^3]\). Both use a projection tolerance of
\(10^{-11}\), at most 5000 projection iterations, and 200 outer
iterations. The computed metric residual
\(\widehat{\mathcal R}_k\) has stopping tolerance \(10^{-7}\).
The implementation tests \(t=2^{-j}\), \(j=0,\ldots,8\), and
accepts the first step satisfying
\(
F(x_k+td_k)\leq F(x_k)-0.1t\widehat{\mathcal R}_k^2.
\)
If backtracking fails, the primal--dual routine approximately solves
the metric proximal subproblem with \(\gamma=100\) within
10,000 iterations, and its final primal iterate is used for the
update.

\paragraph{Evaluation criteria.}
The normalized objective gap is
\(\operatorname{Gap}(x)
=\max\{F(x)-F(x^\star),0\}/(1+|F(x^\star)|)\).
Final residuals are recomputed in the Euclidean norm using the same
numerical projection for all methods. DCI and the final evaluation
treat ReLU preactivations with absolute value at most \(10^{-7}\)
and norm residuals with Euclidean norm at most \(10^{-7}\) as
structural boundaries, retaining their full multiplier intervals
and balls, respectively.
Success requires normalized gap \(\leq10^{-8}\), distance to
\(x^\star\) \(\leq2\times10^{-4}\), and numerical residual
\(\leq2\times10^{-5}\), irrespective of the solver's stopping
status. Wrong denotes the percentage of all runs that declare
convergence but fail at least one success criterion.
Reported solver runtimes exclude final residual recomputation.

\begin{table}[h]
\centering
\caption{\textbf{Aggregate controlled-boundary results.}
Forty-eight runs per method. Gap, distance, and numerical residual are
medians over all runs; iteration counts and runtimes are means.
A dash indicates that convergence declarations were not identified by
the solver interface.}
\label{tab:app_exp2_small}
\footnotesize
\setlength{\tabcolsep}{2.2pt}
\begin{tabular}{lccccccc}
\toprule
Method & Success & Wrong & Gap & Dist. & Residual & Iter. & Time (s)\\
\midrule
Single-Branch & 0.0\% & 0.0\% & $3.264\times10^{-3}$ & $1.303\times10^{-1}$ & $4.143\times10^{-1}$ & 44.2 & 0.714 \\
L-BFGS & 0.0\% & -- & $2.068\times10^{-5}$ & $1.492\times10^{-2}$ & $7.239\times10^{-2}$ & 76.8 & 0.710 \\
Gradient Sampling & 68.8\% & 12.5\% & $7.029\times10^{-9}$ & $1.358\times10^{-7}$ & $1.592\times10^{-7}$ & 93.1 & 25.067 \\
DCI-I & 100.0\% & 0.0\% & $3.623\times10^{-12}$ & $9.465\times10^{-9}$ & $1.239\times10^{-8}$ & 23.4 & 0.925 \\
DCI-C & 100.0\% & 0.0\% & $4.416\times10^{-11}$ & $3.879\times10^{-9}$ & $4.996\times10^{-9}$ & 13.3 & 0.672 \\
Lifted-PD & 100.0\% & 0.0\% & $5.570\times10^{-12}$ & $4.165\times10^{-11}$ & $2.391\times10^{-11}$ & 303.3 & 0.094 \\
\bottomrule
\end{tabular}
\end{table}

\begin{table}[h]
\centering
\caption{\textbf{Success rate by boundary type.}
Each cell uses four matched runs. Active counts the boundary blocks;
Single and GS denote Single-Branch and Gradient Sampling.}
\label{tab:app_exp2_settings}
\small
\begin{tabular}{rllcccc}
\toprule
$d_x$ & Type & Active & Single & GS & DCI-I & DCI-C\\
\midrule
10 & mixed & 2 & 0\% & 75\% & 100\% & 100\% \\
10 & mixed & 4 & 0\% & 25\% & 100\% & 100\% \\
10 & norm & 2 & 0\% & 50\% & 100\% & 100\% \\
10 & norm & 4 & 0\% & 0\% & 100\% & 100\% \\
10 & relu & 2 & 0\% & 100\% & 100\% & 100\% \\
10 & relu & 4 & 0\% & 50\% & 100\% & 100\% \\
30 & mixed & 2 & 0\% & 100\% & 100\% & 100\% \\
30 & mixed & 4 & 0\% & 75\% & 100\% & 100\% \\
30 & norm & 2 & 0\% & 50\% & 100\% & 100\% \\
30 & norm & 4 & 0\% & 100\% & 100\% & 100\% \\
30 & relu & 2 & 0\% & 100\% & 100\% & 100\% \\
30 & relu & 4 & 0\% & 100\% & 100\% & 100\% \\
\bottomrule
\end{tabular}
\end{table}

\paragraph{Scaling design.}
We use three sweeps with five independent instances per case and
one initialization per instance. The total-size sweep fixes
\(d_x\) and \(p\) while varying \(p_{\rm lift}\); the active-size
sweep fixes \(d_x\) and \(p_{\rm lift}\) while varying \(p\).
The dimension sweep fixes \(p\) while jointly increasing the input
dimension and the number of modules.
\Cref{tab:app_exp2_design} lists the cases, and
\cref{tab:app_exp2_large} reports the aggregate results over
60 runs per method using the same evaluation criteria.

\begin{table}[h]
\centering
\caption{\textbf{Large-scale case design.}
Active is the number of boundary blocks, \(k_g\) is the norm residual
dimension, \(p=p(0)\), and Lifted denotes \(p_{\rm lift}\).}
\label{tab:app_exp2_design}
\footnotesize
\setlength{\tabcolsep}{1.8pt}
\begin{tabular}{llrrrrrrr}
\toprule
Sweep & Case & $d_x$ & Active & ReLU & Norm & $k_g$ & $p$ & Lifted\\
\midrule
total size & N-small & 100 & 4 & 128 & 8 & 4 & 10 & 160 \\
total size & N-medium & 100 & 4 & 512 & 32 & 4 & 10 & 640 \\
total size & N-large & 100 & 4 & 2048 & 128 & 4 & 10 & 2560 \\
active size & q-2 & 100 & 2 & 512 & 64 & 4 & 5 & 768 \\
active size & q-4 & 100 & 4 & 512 & 64 & 4 & 10 & 768 \\
active size & q-8 & 100 & 8 & 512 & 64 & 4 & 20 & 768 \\
active size & q-16 & 100 & 16 & 512 & 64 & 4 & 40 & 768 \\
active size & q-32 & 100 & 32 & 512 & 64 & 4 & 80 & 768 \\
dimension & d-25 & 25 & 4 & 100 & 8 & 4 & 10 & 132 \\
dimension & d-50 & 50 & 4 & 200 & 12 & 4 & 10 & 248 \\
dimension & d-100 & 100 & 4 & 400 & 25 & 4 & 10 & 500 \\
dimension & d-200 & 200 & 4 & 800 & 50 & 4 & 10 & 1000 \\
\bottomrule
\end{tabular}
\end{table}

\begin{table}[h]
\centering
\caption{\textbf{Large-suite aggregate optimization results.}
The numerical residual is the median over all runs; iteration counts
and runtimes are means. Success uses the same criteria as the
small-scale suite.}
\label{tab:app_exp2_large}
\small
\begin{tabular}{lccccc}
\toprule
Method & Runs & Success & Residual & Iter. & Time (s)\\
\midrule
DCI-C & 60 & 100.0\% & $4.570\times10^{-9}$ & 15.7 & 29.213 \\
DCI-I & 60 & 100.0\% & $1.598\times10^{-8}$ & 56.7 & 69.449 \\
L-BFGS & 60 & 0.0\% & $2.155\times10^{-1}$ & 223.1 & 6.475 \\
Lifted-PD & 60 & 100.0\% & $2.398\times10^{-10}$ & 1065.3 & 11.118 \\
\bottomrule
\end{tabular}
\end{table}

\paragraph{Pointwise boundary-oracle microbenchmark.}
The oracle is evaluated at the known minimizer \(x^\star=0\)
with zero activity tolerance. Each instance receives five warm-up
calls followed by 50 timed calls. Mean and SD report the mean and
standard deviation across the five per-instance mean oracle times;
Projection reports the corresponding mean time spent on the local
projection. These are pointwise oracle timings, not complete
inference runtimes.

\begin{table}[h]
\centering
\caption{\textbf{DCI-C boundary-oracle timing.} Times are milliseconds.}
\label{tab:app_exp2_oracle_c}
\footnotesize
\setlength{\tabcolsep}{2.0pt}
\begin{tabular}{llrrrrrr}
\toprule
Sweep & Case & $d_x$ & $p$ & Lifted & Mean & SD & Projection\\
\midrule
active size & q-2 & 100 & 5 & 768 & 624.1 & 195.2 & 608.3 \\
active size & q-4 & 100 & 10 & 768 & 675.1 & 23.5 & 659.8 \\
active size & q-8 & 100 & 20 & 768 & 727.1 & 68.3 & 709.7 \\
active size & q-16 & 100 & 40 & 768 & 738.5 & 14.5 & 724.6 \\
active size & q-32 & 100 & 80 & 768 & 1054.4 & 90.4 & 1042.3 \\
dimension & d-25 & 25 & 10 & 132 & 60.3 & 7.5 & 59.3 \\
dimension & d-50 & 50 & 10 & 248 & 132.5 & 76.8 & 130.4 \\
dimension & d-100 & 100 & 10 & 500 & 710.5 & 59.3 & 700.8 \\
dimension & d-200 & 200 & 10 & 1000 & 1893.4 & 34.1 & 1852.5 \\
total size & N-small & 100 & 10 & 160 & 687.1 & 16.7 & 680.5 \\
total size & N-medium & 100 & 10 & 640 & 686.3 & 43.2 & 676.5 \\
total size & N-large & 100 & 10 & 2560 & 740.5 & 52.5 & 697.3 \\
\bottomrule
\end{tabular}
\end{table}

\begin{table}[h]
\centering
\caption{\textbf{DCI-I boundary-oracle timing.} Times are milliseconds.}
\label{tab:app_exp2_oracle_i}
\footnotesize
\setlength{\tabcolsep}{2.0pt}
\begin{tabular}{llrrrrrr}
\toprule
Sweep & Case & $d_x$ & $p$ & Lifted & Mean & SD & Projection\\
\midrule
active size & q-2 & 100 & 5 & 768 & 550.3 & 197.6 & 547.2 \\
active size & q-4 & 100 & 10 & 768 & 672.7 & 44.8 & 669.2 \\
active size & q-8 & 100 & 20 & 768 & 745.7 & 80.1 & 740.7 \\
active size & q-16 & 100 & 40 & 768 & 727.5 & 19.2 & 724.2 \\
active size & q-32 & 100 & 80 & 768 & 961.3 & 110.8 & 958.3 \\
dimension & d-25 & 25 & 10 & 132 & 56.2 & 5.0 & 55.7 \\
dimension & d-50 & 50 & 10 & 248 & 125.6 & 76.3 & 124.7 \\
dimension & d-100 & 100 & 10 & 500 & 717.2 & 15.2 & 715.4 \\
dimension & d-200 & 200 & 10 & 1000 & 1878.7 & 69.1 & 1873.4 \\
total size & N-small & 100 & 10 & 160 & 675.3 & 27.3 & 674.4 \\
total size & N-medium & 100 & 10 & 640 & 678.4 & 23.0 & 676.2 \\
total size & N-large & 100 & 10 & 2560 & 676.5 & 95.9 & 660.7 \\
\bottomrule
\end{tabular}
\end{table}

At fixed input and boundary dimensions, the local projection times
vary little as the full representation grows. These measurements
distinguish the cost of the reduced multiplier problem from the
remaining oracle computations; the dimension sweep additionally
captures the joint growth of input dimension and model size.

\subsection{Experiment 3: Explicit Curvature under Ill Conditioning}
\label{app:exp3}

\paragraph{What is isolated.}
Every trajectory in this experiment remains strictly away from ReLU and
zero-norm boundaries. Therefore
\[
\partial F(x)=\{\nabla F(x)\},
\]
and the network-side first-order oracle is single-valued along the entire
trajectory. In particular, AD-GD and Readout-GD use the same first-order
information, so any improvement of DCI-C over these methods is attributable to
curvature rather than set-valued geometry.

\paragraph{Controlled smooth objectives.}
The decision dimension is $d_x=20$. For a random orthogonal matrix $R$, let
$z=Rx$. We construct
\begin{align}
F_{\rm quad}(x)&=\tfrac12 z^\top Qz+a^\top z+
\sum_{j=1}^{8}\rho_j[w_j^\top z+b_j]_+,
\label{eq:app_exp3_quad_full}\\
F_{\rm qnorm}(x)&=F_{\rm quad}(x)+
\sum_{g=1}^{3}\lambda_g\|D_gz+e_g\|_2.
\label{eq:app_exp3_qnorm_full}
\end{align}
The eight ReLU directions have norm $0.65$, $\rho_j\in[0.25,0.75]$, and
biases alternating between $+5$ and $-5$. In the quadratic--norm family,
$D_g$ is diagonal with entries in $[0.12,0.28]$, $\lambda_g\in[0.25,0.50]$,
and $e_g$ has one nonzero coordinate of magnitude four. The target full-Hessian
condition numbers are
\[
\kappa\in\{1,2,5,10,20,50,100,200,500,1000,2000,5000,10^4\}.
\]
For each $\kappa$, the quadratic part uses logarithmically spaced eigenvalues
from $1$ to $\kappa$, followed by a random permutation. In the
quadratic--norm family, the analytic norm curvature at the origin is subtracted
from the quadratic diagonal so that the \emph{full} Hessian at $x^\star=0$
has the prescribed condition number. The affine term cancels all smooth-module
slopes at the origin. Starts have radius one with
$z_{0,i}=s_i/\sqrt{d_x}$ and $s_i\in\{-1,1\}$. Five seeds for each family and
condition number yield 130 matched objectives.

\paragraph{Methods and success criteria.}
AD-GD and Readout-GD use identical Armijo search. DCI-C-Local uses the
analytic readout and explicit Hessian without the nonsmooth safeguard.
Canonical-Newton uses AD gradients and the full AD Hessian. AD-L-BFGS uses
memory 12. Iteration budgets are 20{,}000 for gradient descent, 80 for Newton
methods, and 500 for L-BFGS. Armijo uses parameter $10^{-4}$, contraction
$0.5$, and at most 45 reductions. Metric clipping is $[10^{-10},10^{10}]$.
Success requires normalized gap $\le 10^{-10}$, residual
$\le 2.5\times10^{-7}$, and decision error $\le 2\times10^{-5}$.
For compactness, the tables abbreviate DCI-C-Local as DCI-C.

\begin{table}[h]
\centering
\caption{\textbf{Smooth-branch isolation checks.}}
\label{tab:app_exp3_validation}
\small
\begin{tabular}{lc}
\toprule
Diagnostic & Result\\
\midrule
Mean analytic--AD gradient error & $2.854\times10^{-14}$\\
Mean explicit--AD Hessian error & $2.734\times10^{-13}$\\
Minimum branch-stability rate & 1.000\\
Minimum ReLU margin & 4.586\\
Minimum norm residual & 3.921\\
\bottomrule
\end{tabular}
\end{table}

\begin{table}[h]
\centering
\caption{\textbf{Conditioning study.} Each entry is success (\%)/median iterations over five seeds. Representative condition numbers are shown; \cref{fig:main_conditioning} plots the full 13-point sweep.}
\label{tab:app_exp3_conditioning}
\footnotesize
\setlength{\tabcolsep}{2.8pt}
\begin{tabular}{llccccc}
\toprule
Family & Method & $1$ & $10$ & $10^2$ & $10^3$ & $10^4$\\
\midrule
Quadratic & AD-GD & 100/1 & 100/65 & 100/645 & 0/5947 & 0/20000 \\
Quadratic & Readout-GD & 100/1 & 100/65 & 100/646 & 0/5856 & 0/20000 \\
Quadratic & AD-L-BFGS & 100/1 & 100/24 & 100/73 & 100/199 & 0/500 \\
Quadratic & Canonical-Newton & 100/1 & 100/1 & 100/1 & 100/1 & 100/1 \\
Quadratic & DCI-C & 100/1 & 100/1 & 100/1 & 100/1 & 100/1 \\
Quadratic--norm & AD-GD & 100/2 & 100/64 & 80/650 & 0/5859 & 0/20000 \\
Quadratic--norm & Readout-GD & 100/2 & 100/64 & 80/650 & 0/5859 & 0/20000 \\
Quadratic--norm & AD-L-BFGS & 100/2 & 100/24 & 100/70 & 60/208 & 0/500 \\
Quadratic--norm & Canonical-Newton & 100/2 & 100/2 & 100/2 & 100/2 & 100/2 \\
Quadratic--norm & DCI-C & 100/2 & 100/2 & 100/2 & 100/2 & 100/2 \\
\bottomrule
\end{tabular}
\end{table}

\begin{table}[h]
\centering
\caption{\textbf{Largest condition number $\kappa=10^4$.} Values are medians over five seeds.}
\label{tab:app_exp3_hard}
\footnotesize
\setlength{\tabcolsep}{2.4pt}
\begin{tabular}{llccccc}
\toprule
Family & Method & Success & Iter. & Gap & Residual & Time (s)\\
\midrule
Quadratic & AD-GD & 0\% & 20000 & $4.218\times10^{-7}$ & $7.689\times10^{-3}$ & 12.9449 \\
Quadratic & Readout-GD & 0\% & 20000 & $4.218\times10^{-7}$ & $7.689\times10^{-3}$ & 9.1467 \\
Quadratic & AD-L-BFGS & 0\% & 500 & $3.613\times10^{-14}$ & $4.562\times10^{-6}$ & 0.7932 \\
Quadratic & Canonical-Newton & 100\% & 1 & 0 & $2.943\times10^{-12}$ & 0.0076 \\
Quadratic & DCI-C & 100\% & 1 & 0 & $3.243\times10^{-12}$ & 0.0011 \\
Quadratic--norm & AD-GD & 0\% & 20000 & $2.921\times10^{-7}$ & $7.689\times10^{-3}$ & 31.5713 \\
Quadratic--norm & Readout-GD & 0\% & 20000 & $2.921\times10^{-7}$ & $7.689\times10^{-3}$ & 22.9283 \\
Quadratic--norm & AD-L-BFGS & 0\% & 500 & $1.343\times10^{-14}$ & $4.951\times10^{-6}$ & 1.5096 \\
Quadratic--norm & Canonical-Newton & 100\% & 2 & 0 & $1.429\times10^{-13}$ & 0.0430 \\
Quadratic--norm & DCI-C & 100\% & 2 & 0 & $1.429\times10^{-13}$ & 0.0041 \\
\bottomrule
\end{tabular}
\end{table}

The branch-stability diagnostics and the near-identical AD-GD/Readout-GD
iteration profiles verify that this experiment isolates curvature rather than
set-valued geometry. DCI-C and Canonical-Newton have identical iteration
counts throughout the sweep, while the structured Hessian is cheaper in the
tested implementation, yielding median runtime reductions of $7.2\times$ on
the quadratic family and $10.4\times$ on the quadratic--norm family at
$\kappa=10^4$.

\subsection{Experiment 4: Generalized Feasible-Set Geometry and Safeguarding}
\label{app:exp4}

\paragraph{What is being optimized.}
We solve \(\min_{x\in\mathcal X}F(x)\) with a known feasible
minimizer \(x^\star\). Unlike Experiment~2, constrained optimality
can require a nonzero normal vector. The experiment examines the
joint use of the objective subdifferential and feasible-set geometry,
then compares local inference with and without a proximal safeguard.

\paragraph{Symbols specific to Experiment~4.}
Write \(\Delta x=x-x^\star\). The matrix \(Q\succ0\) supplies
quadratic curvature, and \(m_{\rm R}\) and \(m_{\rm N}\) denote
the numbers of ReLU and norm modules. ReLU module \(i\) has direction
\(w_i\in\mathbb R^{d_x}\), offset \(b_i\), and coefficient
\(\alpha_i>0\); norm module \(g\) has matrix
\(A_g\in\mathbb R^{3\times d_x}\), offset
\(d_g\in\mathbb R^3\), and weight \(\lambda_g>0\).
At the known optimum, select \(\theta_i^\star\in[0.25,0.75]\)
when \(b_i=0\), and set \(\theta_i^\star=1\) for \(b_i>0\)
and \(0\) for \(b_i<0\). For \(d_g=0\), choose
\(r_g^\star\) with norm in
\([0.25\lambda_g,0.65\lambda_g]\); otherwise set
\(r_g^\star=\lambda_gd_g/\|d_g\|_2\).
The selected module support is
\[
g_{\rm net}^\star
=\sum_{i=1}^{m_{\rm R}}\alpha_i\theta_i^\star w_i
+\sum_{g=1}^{m_{\rm N}}A_g^\top r_g^\star.
\]
Choose \(n^\star\in N_{\mathcal X}(x^\star)\), with
\(n^\star=0\) in the unconstrained case, and set
\(a=-(g_{\rm net}^\star+n^\star)\). This gives
\(0\in\partial F(x^\star)+N_{\mathcal X}(x^\star)\);
strong convexity makes \(x^\star\) the unique minimizer.

\paragraph{Controlled objective class.}
All objectives have \(d_x=10\) and take the form
\begin{equation}
\begin{aligned}
F(x)={}&\tfrac12\Delta x^\top Q\Delta x+a^\top\Delta x
+\sum_{i=1}^{m_{\rm R}}\alpha_i[w_i^\top\Delta x+b_i]_+\\
&+\sum_{g=1}^{m_{\rm N}}\lambda_g\|A_g\Delta x+d_g\|_2,
\qquad x\in\mathcal X,
\end{aligned}
\label{eq:app_exp4_objective}
\end{equation}
where \([t]_+=\max\{0,t\}\).
The matrix \(Q\) is a random orthogonal rotation of a geometric
spectrum from \(1\) to the condition number specified in
\cref{tab:app_exp4_design}. ReLU directions are normalized Gaussian
vectors, \(\alpha_i\in[0.40,0.95]\), and nonzero ReLU offsets
alternate between \(\pm2.5\). Norm matrices have the form
\(A_g=0.35\widetilde A_g\), where \(\widetilde A_g\) has
independent standard-normal entries; \(\lambda_g\in[0.40,0.85]\)
and nonzero norm offsets have norm \(3.5\). Modules designated as
structural boundaries have zero offsets.

\paragraph{Feasible sets and regimes.}
We use unconstrained space, the box \([-1,1]^{d_x}\), the unit
Euclidean ball, the probability simplex, and the second-order cone
\(\{(t,u):t\geq\|u\|_2\}\).
Box constraint-boundary, joint-boundary, and stress cases activate
one, two, and five faces, respectively. Ball regular optima have
radius \(0.35\); the other ball optima have radius one and normal
magnitudes \(0.8\), \(1.2\), and \(1.8\), respectively.
Simplex support sizes are nine, five, and two for the corresponding
boundary regimes, with the equality normal always included.
For a unit vector \(u\), second-order cone regular optima are
\((0.9,0.3u)\), ordinary boundary optima are \((1,u)\), and
stress optima are \((0.25,0.25u)\), with opposed initial directions
in the stress cases.
Unconstrained problems use the regular, structural-boundary, and
stress regimes; each constrained set uses the regular,
constraint-boundary, joint-boundary, and stress regimes.
Five seeds per setting give 95 objectives and 570 main solver runs,
with a common initialization across methods for each objective.

\begin{table}[t]
\centering
\caption{\textbf{Experiment 4 difficulty regimes.} Zero modules lie at structural boundaries at the known optimum.}
\label{tab:app_exp4_design}
\small
\setlength{\tabcolsep}{3.0pt}
\begin{tabular}{lccccc}
\toprule
Regime & ReLUs & Zero ReLUs & Norms & Zero norms & $\kappa(Q)$\\
\midrule
Regular & 4 & 0 & 1 & 0 & 10 \\
Constraint boundary & 4 & 0 & 1 & 0 & 30 \\
Structural boundary & 6 & 2 & 2 & 1 & 40 \\
Joint boundary & 6 & 2 & 2 & 1 & 50 \\
Stress & 10 & 6 & 3 & 2 & 100 \\
\bottomrule
\end{tabular}
\end{table}

\paragraph{Joint constrained readout.}
For \(M\succ0\), the joint projection solves
\[
(s_M^\dagger,n_M^\dagger)
\in\mathop{\arg\min}_{s\in\partial F(x),\,n\in N_{\mathcal X}(x)}
\tfrac12(s+n)^\top M^{-1}(s+n).
\]
Its reduced variables consist of intervals for boundary ReLUs,
Euclidean balls for zero norm residuals, free equality-normal
coefficients, and nonnegative inequality-normal coefficients.
Accelerated projected gradient uses at most 8000 iterations and
tolerance \(10^{-11}\).
The network-only residual omits the normal cone and targets
\(\operatorname{dist}(0,\partial F(x))\), whereas the joint residual
targets \(\operatorname{dist}(0,\partial F(x)+N_{\mathcal X}(x))\).
Both are evaluated numerically with the activity tolerances below.

\paragraph{Compared methods.}
AD-PGD combines a branch-selected AD vector with Euclidean feasible
projection. Readout-PGD instead uses the minimum-norm objective
subgradient without jointly optimizing over the normal cone.
DCI-I-Local uses the joint projection with \(M=I\).
For \(u_g(x)=A_g(x-x^\star)+d_g\) and
\(\widehat u_g(x)=u_g(x)/\|u_g(x)\|_2\), DCI-C-Local uses
\[
M_C(x)=Q+
\sum_{g:\|u_g(x)\|_2>10^{-7}}
\lambda_gA_g^\top
\frac{I-\widehat u_g(x)\widehat u_g(x)^\top}{\|u_g(x)\|_2}
A_g+10^{-7}I,
\]
with eigenvalues clipped to \([10^{-5},10^6]\).
DCI-C-Full uses the same local path and invokes a metric proximal
step after finite line-search failure. Prox-Only uses the same
curvature metric but takes a proximal step at every outer iteration.
AD-PGD and Readout-PGD provide projected-gradient references;
the DCI variants compare identity and curvature metrics, followed
by the addition of the proximal safeguard.

\paragraph{Local feasible search and safeguard.}
For the computed joint vector
\(g_k^\dagger=s_{M_k}^\dagger+n_{M_k}^\dagger\), set
\(d_k=-M_k^{-1}g_k^\dagger\).
Local DCI tests \(z_k(t)=\Pi_{\mathcal X}(x_k+td_k)\), where
\(\Pi_{\mathcal X}\) is Euclidean projection, and writes
\(s_k(t)=z_k(t)-x_k\).
The default schedule is \(t=2^{-j}\), \(j=0,\ldots,10\).
The numerical implementation accepts a trial when
\[
F'(x_k;s_k(t))<0,
\qquad
F(z_k(t))\leq F(x_k)+10^{-4}F'(x_k;s_k(t)).
\]
The projected-gradient references use up to 45 trial steps with
contraction \(0.5\).
If local search fails, Full approximately solves
\[
\min_{z\in\mathcal X}
F(z)+\frac{1}{2\gamma}\|z-x_k\|_{M_k}^2,
\qquad \gamma=8,
\]
using at most 5000 primal--dual iterations and tolerance
\(5\times10^{-7}\); the final primal iterate is used for the
update. The local methods allow 300 outer iterations and Prox-Only
allows 40. Accepted local and proximal candidates are adjusted by
the same numerical boundary-identification procedure, using a
snap tolerance of \(2\times10^{-4}\).

\paragraph{Numerical evaluation and termination.}
The algorithmic structural activity tolerance is \(10^{-7}\);
final residuals and directional derivatives use \(10^{-11}\).
ReLU preactivations with absolute value at most the applicable
tolerance and norm residuals with norm at most it retain their full
multiplier intervals and balls. Constraint activity uses tolerance
\(10^{-9}\). Final residuals are recomputed with \(M=I\) for
all methods. DCI and Prox-Only stop using this Euclidean joint
residual, while AD-PGD and Readout-PGD use their projected-gradient
mapping; the stopping threshold is \(2\times10^{-5}\).
For the returned point \(x_{\rm out}\), define the normalized gap
\(\operatorname{Gap}(x_{\rm out})
=\max\{F(x_{\rm out})-F(x^\star),0\}/(1+|F(x^\star)|)\).
Success requires feasibility within \(10^{-9}\), normalized gap
at most \(10^{-8}\), decision error
\(\|x_{\rm out}-x^\star\|_2\leq2\times10^{-4}\), and numerical
joint residual at most \(2\times10^{-5}\), independently of the
internal stopping flag.
Declared records an internal convergence declaration; Certified
requires a declaration and a passing final residual test, and Wrong
records a declaration that fails this residual test.
All rates use the total number of runs as the denominator.
Solver runtimes exclude the final external residual recomputation.

\begin{table}[t]
\centering
\caption{\textbf{Certificate validation at the constructed optima.} Network residuals are means, joint residuals are maxima, and projection iterations and times are medians over five instances. Times are in milliseconds.}
\label{tab:app_exp4_certificate}
\footnotesize
\setlength{\tabcolsep}{3.5pt}
\begin{tabular}{llrrrr}
\toprule
Constraint & Regime & Network res. & Max joint res. & Proj. iter. & Proj. ms\\
\midrule
ball & constraint boundary & 0.800 & $2.293\times10^{-16}$ & 2 & 0.49 \\
ball & joint boundary & 0.892 & $1.630\times10^{-8}$ & 191 & 11.90 \\
ball & regular & 0.000 & $1.520\times10^{-16}$ & 0 & 0.32 \\
ball & stress & 1.043 & $5.634\times10^{-9}$ & 286 & 19.16 \\
box & constraint boundary & 0.800 & $1.327\times10^{-16}$ & 2 & 0.53 \\
box & joint boundary & 1.514 & $1.093\times10^{-9}$ & 158 & 7.14 \\
box & regular & 0.000 & $1.173\times10^{-16}$ & 0 & 0.32 \\
box & stress & 1.970 & $1.823\times10^{-9}$ & 157 & 10.68 \\
simplex & constraint boundary & 1.079 & $4.968\times10^{-10}$ & 105 & 3.90 \\
simplex & joint boundary & 1.530 & $1.064\times10^{-8}$ & 633 & 36.57 \\
simplex & regular & 1.107 & $2.001\times10^{-16}$ & 2 & 0.54 \\
simplex & stress & 1.166 & $1.519\times10^{-9}$ & 229 & 15.17 \\
SOC & constraint boundary & 1.131 & $3.675\times10^{-16}$ & 2 & 0.50 \\
SOC & joint boundary & 1.399 & $3.104\times10^{-10}$ & 130 & 6.04 \\
SOC & regular & 0.000 & $2.230\times10^{-16}$ & 0 & 0.36 \\
SOC & stress & 1.716 & $1.089\times10^{-8}$ & 353 & 25.02 \\
unconstrained & regular & 0.000 & $3.527\times10^{-16}$ & 0 & 0.31 \\
unconstrained & stress & 0.000 & $3.329\times10^{-8}$ & 442 & 36.73 \\
unconstrained & structural boundary & 0.000 & $1.038\times10^{-9}$ & 117 & 4.94 \\
\bottomrule
\end{tabular}
\end{table}

\begin{table}[t]
\centering
\caption{\textbf{Overall reliability on 95 objectives per method.} SG is the percentage of runs invoking at least one safeguard step. Prox-Only denotes Metric-Prox-Only.}
\label{tab:app_exp4_rel}
\small
\setlength{\tabcolsep}{3.5pt}
\begin{tabular}{lccccc}
\toprule
Method & Success & Declared & Certified & Wrong stop & SG\\
\midrule
AD-PGD & 49.5\% & 47.4\% & 36.8\% & 10.5\% & 0.0\% \\
DCI-C-Full & 100.0\% & 100.0\% & 100.0\% & 0.0\% & 5.3\% \\
DCI-C-Local & 94.7\% & 94.7\% & 94.7\% & 0.0\% & 0.0\% \\
DCI-I-Local & 63.2\% & 63.2\% & 63.2\% & 0.0\% & 0.0\% \\
Prox-Only & 100.0\% & 100.0\% & 100.0\% & 0.0\% & 100.0\% \\
Readout-PGD & 60.0\% & 57.9\% & 47.4\% & 10.5\% & 0.0\% \\
\bottomrule
\end{tabular}
\end{table}

\begin{table}[t]
\centering
\caption{\textbf{Overall accuracy and computation.} Iterations and gaps are medians over all 95 runs; residuals and runtimes are summarized by their medians and 95th percentiles, denoted by $Q_{.95}$. Runtimes are in seconds.}
\label{tab:app_exp4_eff}
\footnotesize
\setlength{\tabcolsep}{3.0pt}
\begin{tabular}{lcccccc}
\toprule
Method & Iter. & Gap & Median res. & $Q_{.95}$ res. & Median s & $Q_{.95}$ s\\
\midrule
AD-PGD & 38 & $1.604\times10^{-11}$ & $2.568\times10^{-5}$ & $5.127\times10^{-1}$ & 0.0869 & 0.5065 \\
DCI-C-Full & 11 & $1.448\times10^{-15}$ & $3.594\times10^{-9}$ & $1.860\times10^{-5}$ & 0.1356 & 1.3989 \\
DCI-C-Local & 11 & $3.830\times10^{-14}$ & $4.052\times10^{-9}$ & $9.636\times10^{-2}$ & 0.0934 & 0.9610 \\
DCI-I-Local & 40 & $1.208\times10^{-11}$ & $1.853\times10^{-5}$ & $1.509\times10^{0}$ & 0.1681 & 1.3205 \\
Prox-Only & 5 & $2.520\times10^{-13}$ & $4.785\times10^{-6}$ & $1.653\times10^{-5}$ & 0.6554 & 4.9452 \\
Readout-PGD & 40 & $1.101\times10^{-11}$ & $1.823\times10^{-5}$ & $3.924\times10^{-1}$ & 0.0715 & 1.1882 \\
\bottomrule
\end{tabular}
\end{table}

\begin{table}[t]
\centering
\caption{\textbf{Success/certified termination by difficulty regime.}
Each entry is Success/Certified (\%) over $n$ objectives per method.
AD and Readout denote the corresponding PGD methods; L and F denote Local and Full.}
\label{tab:app_exp4_reg}
\footnotesize
\setlength{\tabcolsep}{3.0pt}
\begin{tabular}{lrcccccc}
\toprule
Regime & $n$ & AD & Readout & DCI-I & DCI-C-L & DCI-C-F & Prox\\
\midrule
Regular & 25 & 100/100 & 100/100 & 100/100 & 100/100 & 100/100 & 100/100 \\
Constraint boundary & 20 & 50/50 & 50/50 & 100/100 & 100/100 & 100/100 & 100/100 \\
Structural boundary & 5 & 0/0 & 100/100 & 20/20 & 80/80 & 100/100 & 100/100 \\
Joint boundary & 20 & 10/0 & 5/0 & 45/45 & 90/90 & 100/100 & 100/100 \\
Stress & 25 & 40/0 & 64/20 & 20/20 & 92/92 & 100/100 & 100/100 \\
\bottomrule
\end{tabular}
\end{table}

\begin{table}[t]
\centering
\caption{\textbf{Unconstrained results by regime.}
Success and Certified are percentages; residuals and runtimes are medians over all five runs per row.}
\label{tab:app_exp4_unconstrained}
\footnotesize
\setlength{\tabcolsep}{3.5pt}
\begin{tabular}{llcccc}
\toprule
Regime & Method & Success & Certified & Residual & Time (s)\\
\midrule
regular & AD-PGD & 100.0\% & 100.0\% & $1.310\times10^{-5}$ & 0.0279 \\
regular & DCI-C-Full & 100.0\% & 100.0\% & $6.619\times10^{-7}$ & 0.0028 \\
regular & DCI-C-Local & 100.0\% & 100.0\% & $6.619\times10^{-7}$ & 0.0027 \\
regular & DCI-I-Local & 100.0\% & 100.0\% & $1.310\times10^{-5}$ & 0.0424 \\
regular & Prox-Only & 100.0\% & 100.0\% & $6.386\times10^{-6}$ & 0.0477 \\
regular & Readout-PGD & 100.0\% & 100.0\% & $1.310\times10^{-5}$ & 0.0185 \\
\midrule
stress & AD-PGD & 0.0\% & 0.0\% & $3.614\times10^{-1}$ & 0.1821 \\
stress & DCI-C-Full & 100.0\% & 100.0\% & $9.947\times10^{-10}$ & 0.5595 \\
stress & DCI-C-Local & 100.0\% & 100.0\% & $9.947\times10^{-10}$ & 0.5123 \\
stress & DCI-I-Local & 20.0\% & 20.0\% & $1.044\times10^{0}$ & 0.2470 \\
stress & Prox-Only & 100.0\% & 100.0\% & $9.947\times10^{-10}$ & 2.0616 \\
stress & Readout-PGD & 100.0\% & 100.0\% & $9.947\times10^{-10}$ & 0.3400 \\
\midrule
structural boundary & AD-PGD & 0.0\% & 0.0\% & $2.950\times10^{-1}$ & 0.0400 \\
structural boundary & DCI-C-Full & 100.0\% & 100.0\% & $3.608\times10^{-9}$ & 0.1356 \\
structural boundary & DCI-C-Local & 80.0\% & 80.0\% & $3.608\times10^{-9}$ & 0.0934 \\
structural boundary & DCI-I-Local & 20.0\% & 20.0\% & $5.574\times10^{-1}$ & 0.1308 \\
structural boundary & Prox-Only & 100.0\% & 100.0\% & $1.537\times10^{-5}$ & 0.4855 \\
structural boundary & Readout-PGD & 100.0\% & 100.0\% & $1.594\times10^{-5}$ & 0.2204 \\
\bottomrule
\end{tabular}
\end{table}

\begin{table}[t]
\centering
\caption{\textbf{Box results by regime.}
Success and Certified are percentages; residuals and runtimes are medians over all five runs per row.}
\label{tab:app_exp4_box}
\footnotesize
\setlength{\tabcolsep}{3.5pt}
\begin{tabular}{llcccc}
\toprule
Regime & Method & Success & Certified & Residual & Time (s)\\
\midrule
constraint boundary & AD-PGD & 100.0\% & 100.0\% & $1.806\times10^{-5}$ & 0.1502 \\
constraint boundary & DCI-C-Full & 100.0\% & 100.0\% & $1.395\times10^{-10}$ & 0.0062 \\
constraint boundary & DCI-C-Local & 100.0\% & 100.0\% & $1.395\times10^{-10}$ & 0.0050 \\
constraint boundary & DCI-I-Local & 100.0\% & 100.0\% & $1.806\times10^{-5}$ & 0.1826 \\
constraint boundary & Prox-Only & 100.0\% & 100.0\% & $4.082\times10^{-6}$ & 0.1793 \\
constraint boundary & Readout-PGD & 100.0\% & 100.0\% & $1.806\times10^{-5}$ & 0.0751 \\
\midrule
joint boundary & AD-PGD & 0.0\% & 0.0\% & $2.239\times10^{-1}$ & 0.0880 \\
joint boundary & DCI-C-Full & 100.0\% & 100.0\% & $1.025\times10^{-9}$ & 0.2139 \\
joint boundary & DCI-C-Local & 100.0\% & 100.0\% & $1.025\times10^{-9}$ & 0.2845 \\
joint boundary & DCI-I-Local & 60.0\% & 60.0\% & $1.960\times10^{-5}$ & 0.4427 \\
joint boundary & Prox-Only & 100.0\% & 100.0\% & $5.256\times10^{-6}$ & 2.4852 \\
joint boundary & Readout-PGD & 0.0\% & 0.0\% & $2.239\times10^{-1}$ & 0.1401 \\
\midrule
regular & AD-PGD & 100.0\% & 100.0\% & $1.512\times10^{-5}$ & 0.0669 \\
regular & DCI-C-Full & 100.0\% & 100.0\% & $1.337\times10^{-9}$ & 0.0037 \\
regular & DCI-C-Local & 100.0\% & 100.0\% & $1.337\times10^{-9}$ & 0.0034 \\
regular & DCI-I-Local & 100.0\% & 100.0\% & $1.512\times10^{-5}$ & 0.0540 \\
regular & Prox-Only & 100.0\% & 100.0\% & $1.400\times10^{-5}$ & 0.0843 \\
regular & Readout-PGD & 100.0\% & 100.0\% & $1.512\times10^{-5}$ & 0.0228 \\
\midrule
stress & AD-PGD & 100.0\% & 0.0\% & $5.757\times10^{-10}$ & 0.1423 \\
stress & DCI-C-Full & 100.0\% & 100.0\% & $5.757\times10^{-10}$ & 0.2568 \\
stress & DCI-C-Local & 100.0\% & 100.0\% & $5.757\times10^{-10}$ & 0.2451 \\
stress & DCI-I-Local & 0.0\% & 0.0\% & $1.080\times10^{0}$ & 0.2469 \\
stress & Prox-Only & 100.0\% & 100.0\% & $5.757\times10^{-10}$ & 1.4811 \\
stress & Readout-PGD & 100.0\% & 0.0\% & $5.757\times10^{-10}$ & 0.1811 \\
\bottomrule
\end{tabular}
\end{table}

\begin{table}[t]
\centering
\caption{\textbf{Euclidean ball results by regime.}
Success and Certified are percentages; residuals and runtimes are medians over all five runs per row.}
\label{tab:app_exp4_ball}
\footnotesize
\setlength{\tabcolsep}{3.5pt}
\begin{tabular}{llcccc}
\toprule
Regime & Method & Success & Certified & Residual & Time (s)\\
\midrule
constraint boundary & AD-PGD & 0.0\% & 0.0\% & $3.217\times10^{-5}$ & 0.0812 \\
constraint boundary & DCI-C-Full & 100.0\% & 100.0\% & $1.839\times10^{-5}$ & 0.0420 \\
constraint boundary & DCI-C-Local & 100.0\% & 100.0\% & $1.839\times10^{-5}$ & 0.0414 \\
constraint boundary & DCI-I-Local & 100.0\% & 100.0\% & $1.760\times10^{-5}$ & 0.1056 \\
constraint boundary & Prox-Only & 100.0\% & 100.0\% & $1.030\times10^{-5}$ & 0.1185 \\
constraint boundary & Readout-PGD & 0.0\% & 0.0\% & $3.217\times10^{-5}$ & 0.0419 \\
\midrule
joint boundary & AD-PGD & 0.0\% & 0.0\% & $2.526\times10^{-1}$ & 0.0787 \\
joint boundary & DCI-C-Full & 100.0\% & 100.0\% & $1.516\times10^{-5}$ & 0.5507 \\
joint boundary & DCI-C-Local & 100.0\% & 100.0\% & $1.516\times10^{-5}$ & 0.5443 \\
joint boundary & DCI-I-Local & 20.0\% & 20.0\% & $5.822\times10^{-1}$ & 0.1685 \\
joint boundary & Prox-Only & 100.0\% & 100.0\% & $7.338\times10^{-6}$ & 1.6459 \\
joint boundary & Readout-PGD & 0.0\% & 0.0\% & $1.885\times10^{-1}$ & 0.0341 \\
\midrule
regular & AD-PGD & 100.0\% & 100.0\% & $1.365\times10^{-5}$ & 0.0484 \\
regular & DCI-C-Full & 100.0\% & 100.0\% & $3.612\times10^{-8}$ & 0.0028 \\
regular & DCI-C-Local & 100.0\% & 100.0\% & $3.612\times10^{-8}$ & 0.0028 \\
regular & DCI-I-Local & 100.0\% & 100.0\% & $1.365\times10^{-5}$ & 0.0390 \\
regular & Prox-Only & 100.0\% & 100.0\% & $5.834\times10^{-6}$ & 0.0709 \\
regular & Readout-PGD & 100.0\% & 100.0\% & $1.365\times10^{-5}$ & 0.0205 \\
\midrule
stress & AD-PGD & 0.0\% & 0.0\% & $5.059\times10^{-1}$ & 0.1154 \\
stress & DCI-C-Full & 100.0\% & 100.0\% & $3.977\times10^{-10}$ & 0.5326 \\
stress & DCI-C-Local & 60.0\% & 60.0\% & $4.335\times10^{-9}$ & 0.4125 \\
stress & DCI-I-Local & 20.0\% & 20.0\% & $4.647\times10^{-1}$ & 0.3246 \\
stress & Prox-Only & 100.0\% & 100.0\% & $3.977\times10^{-10}$ & 1.9995 \\
stress & Readout-PGD & 20.0\% & 0.0\% & $3.697\times10^{-1}$ & 0.0929 \\
\bottomrule
\end{tabular}
\end{table}

\begin{table}[t]
\centering
\caption{\textbf{Simplex results by regime.}
Success and Certified are percentages; residuals and runtimes are medians over all five runs per row.}
\label{tab:app_exp4_simplex}
\footnotesize
\setlength{\tabcolsep}{3.5pt}
\begin{tabular}{llcccc}
\toprule
Regime & Method & Success & Certified & Residual & Time (s)\\
\midrule
constraint boundary & AD-PGD & 100.0\% & 100.0\% & $1.741\times10^{-5}$ & 0.1649 \\
constraint boundary & DCI-C-Full & 100.0\% & 100.0\% & $1.510\times10^{-9}$ & 0.0399 \\
constraint boundary & DCI-C-Local & 100.0\% & 100.0\% & $1.510\times10^{-9}$ & 0.0409 \\
constraint boundary & DCI-I-Local & 100.0\% & 100.0\% & $1.514\times10^{-5}$ & 1.3149 \\
constraint boundary & Prox-Only & 100.0\% & 100.0\% & $5.041\times10^{-6}$ & 0.6324 \\
constraint boundary & Readout-PGD & 100.0\% & 100.0\% & $1.741\times10^{-5}$ & 0.0937 \\
\midrule
joint boundary & AD-PGD & 40.0\% & 0.0\% & $1.124\times10^{-1}$ & 0.1806 \\
joint boundary & DCI-C-Full & 100.0\% & 100.0\% & $1.265\times10^{-9}$ & 1.3852 \\
joint boundary & DCI-C-Local & 60.0\% & 60.0\% & $3.154\times10^{-9}$ & 0.7801 \\
joint boundary & DCI-I-Local & 60.0\% & 60.0\% & $3.154\times10^{-9}$ & 0.5540 \\
joint boundary & Prox-Only & 100.0\% & 100.0\% & $1.265\times10^{-9}$ & 6.3220 \\
joint boundary & Readout-PGD & 20.0\% & 0.0\% & $1.261\times10^{-1}$ & 0.1028 \\
\midrule
regular & AD-PGD & 100.0\% & 100.0\% & $1.387\times10^{-5}$ & 0.0515 \\
regular & DCI-C-Full & 100.0\% & 100.0\% & $7.056\times10^{-11}$ & 0.0044 \\
regular & DCI-C-Local & 100.0\% & 100.0\% & $7.056\times10^{-11}$ & 0.0043 \\
regular & DCI-I-Local & 100.0\% & 100.0\% & $1.387\times10^{-5}$ & 0.0496 \\
regular & Prox-Only & 100.0\% & 100.0\% & $1.029\times10^{-5}$ & 0.0879 \\
regular & Readout-PGD & 100.0\% & 100.0\% & $1.387\times10^{-5}$ & 0.0251 \\
\midrule
stress & AD-PGD & 100.0\% & 0.0\% & $4.168\times10^{-10}$ & 0.2917 \\
stress & DCI-C-Full & 100.0\% & 100.0\% & $4.168\times10^{-10}$ & 0.9822 \\
stress & DCI-C-Local & 100.0\% & 100.0\% & $4.168\times10^{-10}$ & 0.9567 \\
stress & DCI-I-Local & 60.0\% & 60.0\% & $1.519\times10^{-9}$ & 0.5613 \\
stress & Prox-Only & 100.0\% & 100.0\% & $4.168\times10^{-10}$ & 4.8295 \\
stress & Readout-PGD & 80.0\% & 0.0\% & $1.224\times10^{-9}$ & 0.2797 \\
\bottomrule
\end{tabular}
\end{table}

\begin{table}[t]
\centering
\caption{\textbf{Second-order cone results by regime.}
Success and Certified are percentages; residuals and runtimes are medians over all five runs per row.}
\label{tab:app_exp4_soc}
\footnotesize
\setlength{\tabcolsep}{3.5pt}
\begin{tabular}{llcccc}
\toprule
Regime & Method & Success & Certified & Residual & Time (s)\\
\midrule
constraint boundary & AD-PGD & 0.0\% & 0.0\% & $2.897\times10^{-5}$ & 0.0899 \\
constraint boundary & DCI-C-Full & 100.0\% & 100.0\% & $1.773\times10^{-5}$ & 0.0226 \\
constraint boundary & DCI-C-Local & 100.0\% & 100.0\% & $1.773\times10^{-5}$ & 0.0243 \\
constraint boundary & DCI-I-Local & 100.0\% & 100.0\% & $1.911\times10^{-5}$ & 0.1290 \\
constraint boundary & Prox-Only & 100.0\% & 100.0\% & $3.631\times10^{-6}$ & 0.1446 \\
constraint boundary & Readout-PGD & 0.0\% & 0.0\% & $2.897\times10^{-5}$ & 0.0479 \\
\midrule
joint boundary & AD-PGD & 0.0\% & 0.0\% & $2.870\times10^{-1}$ & 0.0569 \\
joint boundary & DCI-C-Full & 100.0\% & 100.0\% & $1.317\times10^{-5}$ & 0.4419 \\
joint boundary & DCI-C-Local & 100.0\% & 100.0\% & $1.317\times10^{-5}$ & 0.4420 \\
joint boundary & DCI-I-Local & 40.0\% & 40.0\% & $5.425\times10^{-1}$ & 0.5898 \\
joint boundary & Prox-Only & 100.0\% & 100.0\% & $4.078\times10^{-6}$ & 1.2838 \\
joint boundary & Readout-PGD & 0.0\% & 0.0\% & $2.761\times10^{-1}$ & 0.2140 \\
\midrule
regular & AD-PGD & 100.0\% & 100.0\% & $1.251\times10^{-5}$ & 0.0473 \\
regular & DCI-C-Full & 100.0\% & 100.0\% & $1.360\times10^{-8}$ & 0.0029 \\
regular & DCI-C-Local & 100.0\% & 100.0\% & $1.360\times10^{-8}$ & 0.0028 \\
regular & DCI-I-Local & 100.0\% & 100.0\% & $1.251\times10^{-5}$ & 0.0423 \\
regular & Prox-Only & 100.0\% & 100.0\% & $6.203\times10^{-6}$ & 0.0681 \\
regular & Readout-PGD & 100.0\% & 100.0\% & $1.251\times10^{-5}$ & 0.0194 \\
\midrule
stress & AD-PGD & 0.0\% & 0.0\% & $3.948\times10^{-1}$ & 0.1193 \\
stress & DCI-C-Full & 100.0\% & 100.0\% & $1.452\times10^{-9}$ & 0.6874 \\
stress & DCI-C-Local & 100.0\% & 100.0\% & $1.452\times10^{-9}$ & 0.6986 \\
stress & DCI-I-Local & 0.0\% & 0.0\% & $1.600\times10^{0}$ & 0.1581 \\
stress & Prox-Only & 100.0\% & 100.0\% & $1.452\times10^{-9}$ & 1.6843 \\
stress & Readout-PGD & 20.0\% & 0.0\% & $3.182\times10^{-1}$ & 0.1657 \\
\bottomrule
\end{tabular}
\end{table}

\paragraph{Matched safeguard-budget ablation.}
We test every \(J_{\max}\in\{0,1,\ldots,10\}\).
For each objective, Local and Full share the initialization and
metric construction, which are held fixed across budgets.
Three seeds give 12 joint-boundary, 15 stress, and three
structural-boundary objectives per budget, for 330 Local/Full pairs
and 660 solver runs in total.
Rescue is the proportion of pairs for which Local fails and Full
succeeds. In this sweep it equals the Full safeguard invocation
rate, so \cref{tab:app_exp4_safe} and \cref{fig:main_safeguard} reports the common value as SG.

\begin{table}[t]
\centering
\caption{\textbf{Matched Local/Full safeguard ablation.} Local is success rate (\%), SG is Full safeguard use (\%), and $Q_{.95}$ s is the 95th percentile of Full runtime in seconds. Full success and certified termination are $100\%$ in every cell. The number $n$ is the number of paired objectives per budget.}
\label{tab:app_exp4_safe}
\footnotesize
\setlength{\tabcolsep}{3.0pt}
\begin{tabular}{cccccccccc}
\toprule
& \multicolumn{3}{c}{Structural boundary ($n=3$)}
& \multicolumn{3}{c}{Joint boundary ($n=12$)}
& \multicolumn{3}{c}{Stress ($n=15$)}\\
\cmidrule(lr){2-4}\cmidrule(lr){5-7}\cmidrule(lr){8-10}
$J_{\max}$ & Local & SG & $Q_{.95}$ s
& Local & SG & $Q_{.95}$ s
& Local & SG & $Q_{.95}$ s\\
\midrule
0 & 0.0 & 100.0 & 0.277 & 0.0 & 100.0 & 3.971 & 0.0 & 100.0 & 1.970 \\
1 & 0.0 & 100.0 & 0.236 & 0.0 & 100.0 & 3.903 & 0.0 & 100.0 & 1.924 \\
2 & 0.0 & 100.0 & 0.243 & 0.0 & 100.0 & 4.555 & 0.0 & 100.0 & 2.135 \\
3 & 0.0 & 100.0 & 0.260 & 0.0 & 100.0 & 5.268 & 0.0 & 100.0 & 2.339 \\
4 & 0.0 & 100.0 & 0.696 & 0.0 & 100.0 & 4.117 & 6.7 & 93.3 & 3.037 \\
5 & 0.0 & 100.0 & 0.722 & 8.3 & 91.7 & 4.298 & 20.0 & 80.0 & 2.251 \\
6 & 33.3 & 66.7 & 0.772 & 8.3 & 91.7 & 4.432 & 33.3 & 66.7 & 1.881 \\
7 & 33.3 & 66.7 & 1.196 & 33.3 & 66.7 & 3.095 & 53.3 & 46.7 & 1.571 \\
8 & 66.7 & 33.3 & 0.653 & 50.0 & 50.0 & 3.178 & 86.7 & 13.3 & 1.754 \\
9 & 100.0 & 0.0 & 0.300 & 83.3 & 16.7 & 3.098 & 86.7 & 13.3 & 1.756 \\
10 & 100.0 & 0.0 & 0.291 & 91.7 & 8.3 & 3.431 & 86.7 & 13.3 & 1.790 \\
\bottomrule
\end{tabular}
\end{table}

\begin{figure}[t]
\centering
\captionsetup{skip=2pt}

\includegraphics[width=0.8\linewidth]{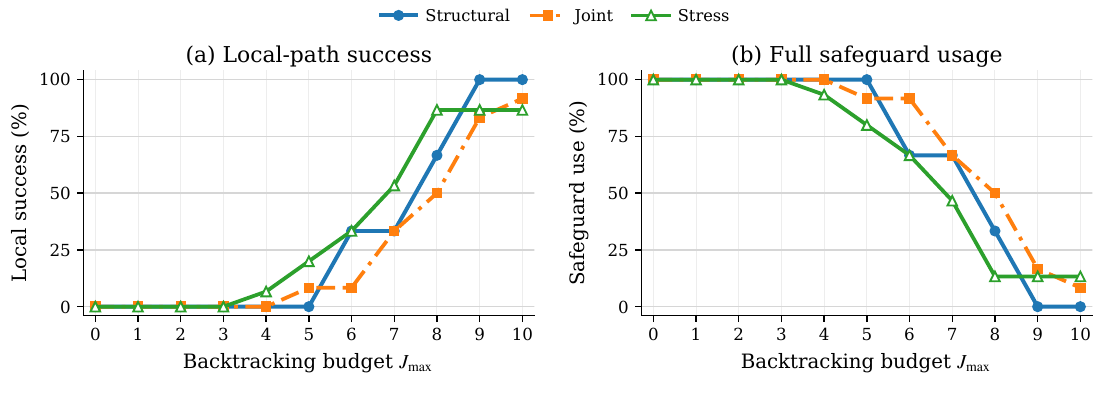}
\caption{\textbf{Finite line-search budget and safeguard.}
Increasing \(J_{\max}\) resolves more instances locally and reduces safeguard
use; the Full method remains successful throughout.}
\label{fig:main_safeguard}
\end{figure}

At \(J_{\max}=10\), Local succeeds on all three structural-boundary
objectives, 11 of 12 joint-boundary objectives, and 13 of 15 stress
objectives. Across the tested budgets, larger budgets reduce
safeguard use, while Full succeeds throughout.

\subsection{Experiment 5: Learned and Explicit Battery-Scheduling Objectives}
\label{app:exp5}

\paragraph{Application and data.}
We study day-ahead battery scheduling using context-conditioned convex neural
objectives. For each day \(r\), the pre-dispatch context \(c_r\) defines a
convex objective over the battery coefficient \(z\in[-1,1]^{12}\):
\[
c_r\longmapsto \widehat F_\theta(\,\cdot\,;c_r),
~~
z_r^\theta\in\argmin_{z\in[-1,1]^{12}}
\widehat F_\theta(z;c_r).
\]
The model therefore predicts an objective rather than a schedule directly.
We evaluate this pipeline at two distinct levels. Experiment~5A freezes a
SOC-ICNN objective and asks how accurately and efficiently different algorithms
minimize that same function; it therefore isolates inference error.
Experiment~5B evaluates the resulting schedules under the realized engineering
cost and compares learned ICNN objectives with explicit engineering objectives.
Its operational regret includes surrogate error as well as any residual
inference error. Reporting the complete certificate separately allows these
two sources to be distinguished.

We use the versioned 2020-10-06 Open Power System Data hourly series for the
Swedish SE3 bidding zone: day-ahead price, actual load, official day-ahead load
forecast, and actual onshore wind generation. Timestamps are converted to
Europe/Stockholm, and incomplete, duplicate, 23-hour, and 25-hour local days
are removed. A causal wind forecast uses the same-hour mean of the previous
seven available days. The chronological split contains 1068 training days
(2015--2017), 363 validation days (2018), and 363 test days (2019).

The context \(c_r\in\mathbb R^{77}\) contains 24 prices, 24 normalized load
forecasts, 24 causal wind forecasts, and five calendar features. Net load is
normalized by the training-set 95th percentile of its absolute value. Realized
test-day load and wind are not included in the context.

\paragraph{Battery model and training objective.}
For \(T=24\), the 12-dimensional decision uses the basis
\[
\overline\Phi_{t,j}=\sin(\pi jt/T),
~~ j=1,\ldots,12.
\]
The normalized battery has capacity one, initial energy \(0.5\), energy bounds
\([0.1,0.9]\), and four-hour duration, so \(P_{\max}=0.25\). We set
\[
\eta=0.98\min\left\{
\frac{\min(e_0-e_{\min},e_{\max}-e_0)}
{\max_t\|\overline\phi_t\|_1},
\frac{P_{\max}}
{\max_t\|\overline\phi_{t+1}-\overline\phi_t\|_1}
\right\},
~~ \Phi=\eta\overline\Phi.
\]
Thus every \(z\in[-1,1]^{12}\) satisfies the energy, power, and terminal-energy
constraints without post-processing. Let \(A\in\mathbb R^{24\times12}\) be
the linear map from \(z\) to the hourly battery injection induced by this
scaled basis, \(g(z)=n+Az\), and \(R\) the hourly first-difference operator.
The realized training target is
\begin{align}
C(z;n,\lambda)={}&
(\lambda-\tfrac\Delta2\mathbf1)^\top g(z)
+\Delta\mathbf1^\top[g(z)]_+
+\kappa_{\rm tp}\|Az\|_1
\nonumber\\
&+\kappa_g\|g(z)\|_2
+\kappa_r\|Rg(z)\|_2
+\tfrac\rho2\|z\|_2^2,
\label{eq:app_exp5_cost}
\end{align}
where \(\Delta=20\), \(\kappa_{\rm tp}=2\), \(\kappa_g=8\),
\(\kappa_r=5\), and \(\rho=0.1\).

For each sampled schedule, the labels are the normalized cost and its exact
decision gradient. Each training day contributes 12 schedules and each
validation day eight. The zero schedule is always included, 20\% of the
remaining schedules place random coordinates
on box faces, and the rest lie in the box interior. Cost values and gradients
use the same affine normalization estimated from the training set. The loss is
\[
\mathcal L
=\frac1N\sum_i(\widehat F_\theta-\widetilde C_i)^2
+0.05\frac1N\sum_i
\|\nabla_z\widehat F_\theta-
\nabla_z\widetilde C_i\|_2^2.
\]
We train three context-conditioned convex architectures under an approximately
matched 19k-parameter budget and context width 64. ReLU-ICNN uses the ReLU
backbone and a strongly convex ridge; Quad-ICNN adds ten quadratic modules;
SOC-ICNN uses six quadratic modules and two norm modules of residual dimension
24. All models use ridge floor \(10^{-3}\). Their configurations are listed in
\cref{tab:app_exp5_models}.

\begin{table}[h]
\centering
\caption{\footnotesize
\textbf{Matched learned convex model configurations.}}
\label{tab:app_exp5_models}
\footnotesize
\setlength{\tabcolsep}{5pt}
\begin{tabular}{lrrrrr}
\toprule
Model & Width & Params. & Quad & Norm & Ridge \\
\midrule
ReLU-ICNN & 117 & 19,135 & -- & -- & $10^{-3}$ \\
Quad-ICNN & 107 & 19,135 & 10 & -- & $10^{-3}$ \\
SOC-ICNN & 64 & 19,167 & 6 & $2\times24$ & $10^{-3}$ \\
\bottomrule
\end{tabular}
\end{table}

Adam uses at most 80 epochs, batch size 256, learning rate
\(2\times10^{-3}\), weight decay \(10^{-6}\), gradient clipping at 10,
validation every five epochs, and early-stopping patience eight. Each
architecture is trained with seeds \(\{0,1,2\}\).

\subsubsection{Experiment 5A: Inference on Frozen SOC-ICNN Objectives}
\label{app:exp5a}

\paragraph{Frozen-objective inference protocol.}
For a selected test day and seed, the context and trained parameters are fixed,
so every solver minimizes the same box-constrained function
\[
z\mapsto\widehat F_{\theta_s}(z;c_r),
~~ z\in[-1,1]^{12}.
\]
Thirty evenly spaced 2019 days for each of three seeds give 90 objectives, and
every method starts from \(z=0\). We compare AD-GD, Readout-GD,
AD-L-BFGS-B, Canonical-Newton, DCI-I, DCI-C, and the high-accuracy Lifted-PD
reference. Local methods receive 120 iterations and use the line search in the
main algorithm, with slope coefficient \(0.1\), Armijo coefficient \(10^{-4}\),
contraction \(0.5\), and at most 30 reductions. The structural tolerance is
\(10^{-7}\). DCI-C uses damping \(10^{-5}\) and metric eigenvalue bounds
\([10^{-3},10^3]\); Lifted-PD uses 30,000 iterations and tolerance \(10^{-9}\).
Success requires learned-objective gap at most \(10^{-6}\) and an independently
recomputed complete residual at most \(2\times10^{-5}\).

\begin{table}[h]
\centering
\caption{\footnotesize
\textbf{Frozen-objective inference reliability on 90 SOC-ICNN objectives.}}
\label{tab:app_exp5_frozen_rel}
\footnotesize
\setlength{\tabcolsep}{2.5pt}
\begin{tabular}{lccccc}
\toprule
Method & Success & Wrong stop & Gap & Median res. & $Q_{.95}$ res.\\
\midrule
AD-GD & 92.2\% & 0.0\% & $6.271\times10^{-15}$ & $1.348\times10^{-5}$ & $1.050\times10^{0}$ \\
Readout-GD & 94.4\% & 0.0\% & $6.190\times10^{-15}$ & $1.339\times10^{-5}$ & $4.890\times10^{-4}$ \\
AD-L-BFGS-B & 94.4\% & 5.6\% & 0 & $9.802\times10^{-8}$ & $4.103\times10^{-2}$ \\
Canonical-Newton & 93.3\% & 0.0\% & $8.575\times10^{-16}$ & $4.559\times10^{-6}$ & $3.945\times10^{-1}$ \\
DCI-I & 96.7\% & 0.0\% & $6.081\times10^{-15}$ & $1.316\times10^{-5}$ & $1.936\times10^{-5}$ \\
DCI-C & 100.0\% & 0.0\% & $2.101\times10^{-16}$ & $1.183\times10^{-6}$ & $1.677\times10^{-5}$ \\
Lifted-PD & 100.0\% & 0.0\% & 0 & $1.384\times10^{-9}$ & $2.789\times10^{-9}$ \\
\bottomrule
\end{tabular}
\end{table}

\begin{table}[h]
\centering
\caption{\footnotesize
\textbf{Frozen-objective inference computation.}
SG is safeguard invocation.}
\label{tab:app_exp5_frozen_eff}
\footnotesize
\setlength{\tabcolsep}{4pt}
\begin{tabular}{lcccc}
\toprule
Method & Median iter. & Median s & $Q_{.95}$ s & SG\\
\midrule
AD-GD & 17.5 & 0.0204 & 0.0449 & 0.0\% \\
Readout-GD & 18.0 & 0.0149 & 0.0468 & 0.0\% \\
AD-L-BFGS-B & 7.0 & 0.0070 & 0.0756 & 0.0\% \\
Canonical-Newton & 3.0 & 0.0258 & 0.1835 & 0.0\% \\
DCI-I & 18.0 & 0.0210 & 0.0896 & 2.2\% \\
DCI-C & 2.0 & 0.0031 & 0.0250 & 3.3\% \\
Lifted-PD & 537.5 & 0.2414 & 0.7436 & 0.0\% \\
\bottomrule
\end{tabular}
\end{table}

\subsubsection{Experiment 5B: Operational Scheduling}
\label{app:exp5b}

\paragraph{Operational comparison.}
Unlike Experiment~5A, this experiment compares complete decision pipelines.
Every returned schedule is evaluated under the realized engineering cost
\(C(z;n_r,\lambda_r)\) in \eqref{eq:app_exp5_cost}. The learned pipelines
minimize their ReLU-, Quad-, or SOC-ICNN surrogate using DCI-C. Point Forecast
and Historical SAA instead construct explicit convex engineering objectives:
Point Forecast substitutes the forecasts available before dispatch, whereas
Historical SAA averages the realized objectives associated with the seven
nearest standardized training contexts. The Perfect Information (PI) oracle
uses the realized test-day net load and is therefore not deployable; it defines
the within-basis reference decision. No Storage uses \(z=0\).

For method \(m\) on test day \(r\), operational regret is
\[
\operatorname{Regret}_r(m)
=
C(z_r^m;n_r,\lambda_r)
-
C(z_r^{\rm PI};n_r,\lambda_r).
\]
Thus, Point Forecast and Historical SAA compare learned objectives with
explicit engineering models under the same information timing, while PI
quantifies the remaining value of unavailable future information. All methods
use the same 12-dimensional feasible battery basis, and all 363 days of 2019
are evaluated. The learned-model results aggregate three seeds, giving 1089
seed--day solves per architecture. ``Certified'' means that the returned
decision satisfies the complete optimality certificate for the objective that
the corresponding pipeline actually minimizes; it does not certify agreement
between a learned surrogate and the realized cost. Cost and regret are daily
means, peak exchange is the mean daily value of \(\max_t|g_t(z)|\), and ramp
norm is the mean daily value of \(\|Rg(z)\|_2\).

\begin{table}[h]
\centering
\caption{\footnotesize
\textbf{Out-of-time operational decision quality.}
Cost and regret are evaluated using the realized engineering objective.}
\label{tab:app_exp5_op_quality}
\footnotesize
\setlength{\tabcolsep}{3.6pt}
\begin{tabular}{lcccc}
\toprule
Pipeline & Cost & Regret & Peak exchange & Ramp norm \\
\midrule
No Storage & 865.4857 & 1.1567 & 0.7959 & 0.1342 \\
Point Forecast & 864.3345 & 0.0054 & 0.7904 & 0.1272 \\
Historical SAA & 864.3394 & 0.0104 & 0.7918 & 0.1303 \\
ReLU-ICNN & 866.9870 & 2.6580 & 0.8258 & 0.2158 \\
Quad-ICNN & 866.3496 & 2.0206 & 0.8082 & 0.1746 \\
SOC-ICNN & 866.5023 & 2.1733 & 0.8135 & 0.1889 \\
PI oracle & 864.3290 & 0.0000 & 0.7901 & 0.1266 \\
\bottomrule
\end{tabular}
\end{table}

\begin{table}[h]
\centering
\caption{\footnotesize
\textbf{Operational certification and computation.}
SG denotes safeguard invocation.}
\label{tab:app_exp5_op_compute}
\footnotesize
\setlength{\tabcolsep}{3.2pt}
\begin{tabular}{lccccc}
\toprule
Pipeline & Certified & Median s & $Q_{.95}$ s & Mean s & SG \\
\midrule
No Storage & -- & 0.00000 & 0.0000 & 0.0000 & 0.0\% \\
Point Forecast & 100.0\% & 0.96179 & 8.1375 & 1.9420 & 0.0\% \\
Historical SAA & 100.0\% & 6.77422 & 33.5591 & 10.7374 & 0.0\% \\
ReLU-ICNN & 97.8\% & 0.00224 & 0.1225 & 0.0217 & 0.0\% \\
Quad-ICNN & 99.6\% & 0.00189 & 0.0624 & 0.0121 & 0.0\% \\
SOC-ICNN & 99.1\% & 0.00305 & 0.0738 & 0.1762 & 1.6\% \\
PI oracle & 100.0\% & 0.99607 & 5.8912 & 1.8159 & 0.0\% \\
\bottomrule
\end{tabular}
\end{table}

\paragraph{Paired uncertainty and error attribution.}
We use a seven-day circular moving-block bootstrap with 2000 repetitions and
seed 2027. In \cref{tab:app_exp5_boot}, a positive difference means that the
first pipeline has higher realized cost.

\begin{table}[h]
\centering
\caption{\footnotesize
\textbf{Paired moving-block bootstrap for realized operating cost.}}
\label{tab:app_exp5_boot}
\footnotesize
\setlength{\tabcolsep}{4pt}
\begin{tabular}{lccc}
\toprule
Difference & Mean & 95\% CI & $P(A<B)$ \\
\midrule
SOC-ICNN $-$ ReLU-ICNN & -0.4847 & $[-0.6659,\,-0.3272]$ & 1.000 \\
SOC-ICNN $-$ Quad-ICNN & 0.1526 & $[0.0276,\,0.2623]$ & 0.009 \\
SOC-ICNN $-$ Historical SAA & 2.1629 & $[1.9397,\,2.3885]$ & 0.000 \\
SOC-ICNN $-$ Point Forecast & 2.1678 & $[1.9502,\,2.4059]$ & 0.000 \\
\bottomrule
\end{tabular}
\end{table}

Among the 1089 SOC-ICNN operational solves, 1079 (99.1\%) satisfy the
complete certificate. Their mean operational regret is 2.1654, nearly
unchanged from 2.1733 over all solves, and the noncertified runs contribute
only 0.0278 to the overall mean. The operational gap is therefore primarily
attributable to approximation of the realized engineering objective, rather
than failure of DCI to minimize the learned SOC-ICNN objective.

\end{document}